\documentclass[article,10pt]{article}
\usepackage{stmaryrd}
\usepackage{amsfonts}
\usepackage{bbm}
\usepackage{amscd}
\usepackage{mathrsfs}
\usepackage{latexsym,amssymb,amsmath,amscd,amscd,amsthm,amsxtra}
\usepackage[dvips]{graphicx}
\usepackage[utf8]{inputenc}
\usepackage[T1]{fontenc}
\usepackage{lmodern}
\usepackage{amssymb}
\usepackage[all]{xy}
\usepackage{nicefrac,mathtools}
\usepackage{microtype}
\usepackage{CJK}
\usepackage{amssymb}
\usepackage{epstopdf}
\usepackage{graphicx}
\usepackage{graphics}
\usepackage[T1]{fontenc}
\usepackage[utf8]{inputenc}
\usepackage{authblk}
\usepackage{subfigure}
\usepackage{caption}
\usepackage{framed}
\usepackage{mathrsfs}
\usepackage[noend]{algpseudocode}
\usepackage{algorithmicx,algorithm}
\usepackage{framed}
\usepackage{booktabs}
\usepackage{threeparttable}
\usepackage{color}
\usepackage{subfigure}
\usepackage{diagbox}
\usepackage{verbatim}
\usepackage{amsmath, bm}

\usepackage{amsfonts}
\usepackage{amssymb,amsmath,amsthm, enumerate}
\usepackage{latexsym}
\usepackage{fullpage}

\usepackage{cite}

\usepackage{hyperref}

\usepackage{mathtools}

\newtheorem{thm}{Theorem}
\newtheorem{lem}{Lemma}
\newtheorem{cor}{Corollary}
\newtheorem{pro}{Proposition}

\let\today\relax
\makeatletter
\def\ps@pprintTitle{%
	\let\@oddhead\@empty
	\let\@evenhead\@empty
	\def\@oddfoot{\footnotesize\itshape
		{Submitted preprint} \hfill\today}%
	\let\@evenfoot\@oddfoot
}
\makeatother

\newcommand {\emptycomment}[1]{}

\newcommand{\be }{\begin{equation}}
	\newcommand{\ee }{\end{equation}}

\newcommand{\huaR}{\mathcal{R}}
\newcommand{\huaE}{\mathcal{E}}

\newcommand{\huaH}{\mathcal{H}}

\newcommand{\huaO}{\mathcal{O}}
\newcommand{\huaT}{\mathcal{T}}

\newcommand{\huaN}{\mathcal{N}}

\newcommand{\nono}{\nonumber}

\allowdisplaybreaks

\newcommand{\f}{\frac}
\newcommand{\ttt}{\theta}
\newcommand{\aaa}{\alpha}

\newcommand{\nn}{\langle}
\newcommand{\mm}{\rangle}
\newcommand{\ww}{\widetilde}
\newcommand{\la}{\lambda}
\newcommand{\mbb}{\mathbb}
\newcommand{\eee}{\epsilon}
\newcommand{\ooo}{\omega}
\newcommand{\kkk}{\kappa}
\newcommand{\rrrr}{\right}
\newcommand{\llll}{\left}
\newcommand{\wh}{\widehat}

\def\bea{\begin{eqnarray}}
	\def\eea{\end{eqnarray}}
\def\be{\begin{equation}}
	\def\ee{\end{equation}}
\def\blm{\begin{lem}}
	\def\elm{\end{lem}}
\def\bes{\begin{eqnarray*}}
	\def\ees{\end{eqnarray*}}

\def\o{\mathcal{O}}

\def\dd{\Lambda}

\def\Ab{\mathbf{A}}
\def\Bb{\mathbf{B}}
\def\Cb{\mathbf{C}}
\def\Db{\mathbf{D}}
\def\Eb{\mathbf{E}}

\def\ad{\Ab_{D,\la}}
\def\bd{\Bb_{D,\la}}
\def\cd{\Cb_{D,\la}}
\def\dd{\Db_{D,\la}}

\def\adj{\Ab_{D_j,\la}}
\def\bdj{\Bb_{D_j,\la}}
\def\cdj{\Cb_{D_j,\la}}
\def\ddj{\Db_{D_j,\la}}

\def\ads{\Ab_{D^*,\la}}

\def\cds{\Cb_{D^*,\la}}
\def\dds{\Db_{D^*,\la}}

\def\adjs{\Ab_{D_j^*,\la}}
\def\bdjs{\Bb_{D_j^*,\la}}
\def\cdjs{\Cb_{D_j^*,\la}}
\def\ddjs{\Db_{D_j^*,\la}}

\def\eds{\Eb_{D,D^*,\la}}

\usepackage{booktabs}%加载 booktabs 宏包之后可以使用 \toprule 和 \bottomrule 命令分别画出表格头和表格底的粗横线，而用 \midrule 画出表格中的横线
\usepackage{diagbox} %加载 diagbox 宏包之后可以绘制斜线表头
\usepackage{multirow} %加载 multirow 宏包后可以让表格单元实现合并与分割
\usepackage{makecell} %加载 makecell 宏包可以在表项中使用"\\"命令自由的换行。在不打算固定表列宽度时，它比p{宽度}选项更为灵活。
\usepackage{longtable} %加载 longtable 宏包可以处理行数非常多的长表格
\usepackage{array}
\usepackage{tabularx}
\usepackage{caption}	
\usepackage{graphicx}
\graphicspath{ {./images/} }
\usepackage{subfigure} 
\begin{document}
	\title{Distributed Gradient Descent for Functional  Learning 
		}\vspace{2mm}
	%\author[]{ }
	%\author[]{}
	%\author[]{}
	%%\author[]{\\}
	\date{}
	%\affil[]{Department of Mathematics, City University of Hong Kong}
	%\affil[]{Kowloon, Hong Kong}
	
	%\affil[]{}
	%\affil[]{}

	\author{Zhan Yu\\ \small Department of Mathematics,   Hong Kong Baptist University\\ \small Waterloo Road, Kowloon, Hong Kong\\ \small   Email: mathyuzhan@gmail.com  % "\small" is optional
	\and
	Jun Fan\\ \small Department of Mathematics,   Hong Kong Baptist University\\ \small Waterloo Road, Kowloon, Hong Kong\\ \small   Email: junfan@hkbu.edu.hk
	\and
	Zhongjie Shi\\ \small Department of Statistics and Actuarial Science,   University of Hong Kong\\ \small Pok Fu Lam Road, Hong Kong\\ \small   Email: zshi2@hku.hk
	\and
	Ding-Xuan Zhou\\ \small School of Mathematics and Statistics, University of Sydney \\ \small Sydney NSW 2006, Australia\\
	\small   Email: dingxuan.zhou@sydney.edu.au
}

	\maketitle
	
	\begin{abstract}
In recent years, different types of {distributed and parallel} learning schemes have received increasing attention for their strong advantages in handling large-scale data information. In the information era, to face the big data challenges {that} stem from functional data analysis very recently, we  propose a novel distributed gradient descent functional  learning (DGDFL)  algorithm to tackle functional data across numerous local machines (processors) in the framework of reproducing kernel Hilbert space. Based on integral operator approaches, we provide the first  theoretical understanding  of the DGDFL algorithm in many different aspects {of} the literature. On the way of understanding DGDFL, firstly, a  data-based gradient descent functional  learning (GDFL) algorithm associated with a single-machine model is proposed and comprehensively studied.   Under mild conditions,  confidence-based optimal learning rates of DGDFL are obtained without the saturation boundary on the regularity index suffered in previous works in functional regression. We further provide a semi-supervised DGDFL approach to weaken the restriction on the maximal number of local machines to ensure optimal rates. To our best knowledge, the DGDFL provides the first {divide-and-conquer} iterative training approach to  functional learning  based on data samples of intrinsically  infinite-dimensional random functions (functional covariates) and enriches the {methodologies for} functional data analysis. 
	\end{abstract}
	
	\textbf{Keywords: functional data, gradient descent, learning theory, functional linear model, {divide and conquer}, semi-supervised learning, reproducing kernel, integral operator} 
	
	%$\mathfrak L\mathrm{z}z\mathbf{z}\mathbbm{L}\mathscr{D}{\L}_D$
	\section{Introduction}\vspace{0.000000000000000000000001mm}

Introduced by Ramsay in 1980s \cite{ramsay1982}, \cite{rd1991}, functional data analysis (FDA) has been intensively studied in recent years. Over the past three decades, {the} great success of FDA has been witnessed in a variety of fields including machine learning, image science, economics, medicine and electronic commerce \cite{wcm2016}. Different from conventional data analysis, FDA focuses on data that are intrinsically infinite-dimensional and often appear as random functions or time series. The high- or infinite-dimensional structure of functional data is a rich source of information and {brings} many opportunities for future studies in the information era. To reveal the functional nature, one of {the} most popularly studied frameworks is the functional linear model.   In this paper, we consider the functional linear model
\bea
Y=\aaa_0+\int_{\Omega}X(x)\beta^*(x)dx+\eee,  \label{flmodel}
\eea	
where $Y$ is a scalar response variable, $X:\Omega\rightarrow\mbb R$ is a square integrable functional predictor defined on a compact domain $\Omega\subset\mbb R^d$ for some positive integer $d$, $\beta^*\in L^2(\Omega)$ is the slope function, $\aaa_0$ is the intercept, $\eee$ is the random noise independent of $X$ with $\mbb E[\eee]=0$.
For the sake of simplicity, we assume $\mbb E[X]=0$ and $\aaa_0=0$.
Our goal is to recover the target functional $\eta^*$ given by
\bes
\eta^*(X)=\int_{\Omega}\beta^*(x)X(x)dx, \ X\in L^2(\Omega)
\ees
by constructing an estimator $\wh{\eta}_D$ based on a training sample set $D=\{(X_i,Y_i)\}_{i=1}^{|D|}$ consisting of $|D|$ independent copies of $(X,Y)$.
For a prediction $\eta$, the risk  is defined as  
\bes
\huaE (\eta)=\mbb E_{(X,Y)}\Big[\eta(X)-Y\Big]^2
\ees
where $(X,Y)$ is   independent of $\eta$ and the training data, $\mbb E_{(X,Y)}$ denotes the expectation taken over $X$ and $Y$ only. For any prediction rule $\wh{\eta}_D$ constructed from the training data set $D$,  its prediction accuracy can be measured by the excess risk
\bes
\huaE(\wh{\eta}_D)-\huaE(\eta^*)=\mbb E_X\Big[\wh{\eta}_D(X)-\eta^*(X)\Big]^2
\ees
where $\mbb E_X$ denotes the expectation with respect to $X$.

Recently, there is a growing literature circling the functional linear model \eqref{flmodel} \cite{rs2005}, \cite{ch2006}, \cite{hh2007}, \cite{yc2010}, \cite{cy2012}, \cite{tn2018}, \cite{fls2019}, \cite{tong2021}, \cite{ggs2022}, \cite{ctfg2022}, \cite{ls2022}, \cite{mg2022}. An earlier popular technique for handling such models is the functional principal component analysis (FPCA)  which performs the estimation of $\beta^*$ by  a linear combination of the eigenfunctions of the covariance function of the random function $X$ \cite{ch2006}, \cite{hh2007}. In the past decade,  introduced by Cai and Yuan \cite{yc2010}, \cite{cy2012}, an approach called the reproducing kernel approach to functional linear model {has grown} up quickly. The method introduces the RKHS framework in the functional linear model and  focuses on establishing estimation of the slope function $\beta^*$ which lies in a reproducing kernel Hilbert space (RKHS), {for details of RKHS, we refer to references} e.g.{\cite{yz2006}},{\cite{sz2007}}, \cite{yp2008}, \cite{yc2010},{\cite{db2016}}, \cite{cy2012},{\cite{gs2019}},{\cite{bach2023}}. A well-known strategy to implement the RKHS approach is to consider the Tikhonov regularization scheme (see e.g. \cite{zhou2003}, \cite{bpr2007}) over an RKHS $(\huaH_K,\nn\cdot,\cdot\mm_K,\|\cdot\|_K)$ induced by a Mercer kernel $K$ (continuous, symmetric, positive semi-definite function on $\Omega\times\Omega$). To be more precise, given a training sample $D=\{(X_i,Y_i)\}_{i=1}^{|D|}$ of $|D|$ independent copies of $(X,Y)$, one can utilize the estimator $\hat{\beta}_{D,\la}$ 
generated by the regularized least squares (RLS) scheme given by
\bea
\hat{\beta}_{D,\la}=\arg\min_{\beta\in\huaH_K}\left\{\f{1}{|D|}\sum_{i=1}^{|D|}\left(\int_{\Omega}\beta(x)X_i(x)dx-Y_i\right)^2+\la\|\beta\|_{K}^2\right\}  \label{rlseq}
\eea
to realize the approximation of $\beta^*$. There have been wide
studies on the convergence analysis of $\hat{\beta}_{D,\la}$ generated from the RLS scheme \eqref{rlseq} \cite{yc2010},  \cite{cy2012}, \cite{tn2018}, \cite{tong2021}.  

Our present work aims to establish a new distributed gradient descent functional learning algorithm (DGDFL) to solve the  functional linear model \eqref{flmodel} and systematically carry out convergence analysis of the algorithm.  The motivation of proposing  DGDFL is to face massive data challenges which appear everywhere in modern society. In a single-machine model, when the data scale of random functions that the machine needs to handle is extremely  large, it would be quite difficult to reduce the computational  time, burden and single-machine memory requirements. Moreover, single-machine models are not convenient for preserving privacy. To address the above issues, in this paper, 
inspired by a divide and conquer approach  \cite{zdw2015}, we propose DGDFL for handling functional data. Distributed learning is a very hot topic and {a} preferable approach to conquer massive data information challenges. The theoretical foundation of divide-and-conquer learning has been established in the framework of learning theory in recent work \cite{zdw2015}, \cite{lgz2017}, \cite{lz2018}, \cite{zhou2018}, \cite{gls2019}, \cite{hwz2020},  \cite{sw2021}. {There is also another route for designing {distributed} learning algorithms, often referred to as the decentralized distributed learning algorithm (e.g. \cite{kprr2018}, \cite{rrr2020}, \cite{ypg2020}, \cite{xwct2021}, \cite{jssh2023}).} However, in the literature of functional data analysis {for handling datasets consisting of random functions},  theoretical understanding of divide-and-conquer learning has not started until the  very recent papers \cite{tong2021}, \cite{ls2022} where the authors mainly focus on convergence analysis of the estimator from Tikhonov RLS schemes. It can be witnessed that a divide-and-conquer iterative training approach for {the} computational realization of recovering $\beta^*$ is still lacking in the functional linear model. Moreover, theoretical results on {the} convergence of such algorithms have not been established yet. To address the issue, we would introduce our divide-and-conquer iterative algorithm DGDFL  and investigate its convergence ability in different aspects.

    To realize the goal of recovering the functional $\eta^*$, we first propose a functional-data based gradient descent functional learning (GDFL) algorithm that starts with $\beta_{0,D}=0$ and is iteratively given by 
\bea
\beta_{t+1,D}=\beta_{t,D}-\gamma_t\f{1}{|D|}\sum_{i=1}^{|D|}\left(\int_\Omega\beta_{t,D}(x)X_i(x)dx-Y_i\right)\int_\Omega K(\cdot,x)X_i(x)dx, \quad t=0,1,2,...\label{alg1}
\eea
where $\gamma_k>0$ is the stepsize, $K:\Omega\times\Omega\rightarrow\mbb R$ is a Mercer kernel. The corresponding functional estimator for $\eta^*$ is defined by
$\eta_{t,D}(X)=\int_\Omega \beta_{t,D}(x)X(x)dx$.
 Based on a divide-and-conquer approach, our distributed gradient descent functional learning  (DGDFL) algorithm  starts with partitioning the data set $D=\{(X_i,Y_i)\}_{i=1}^{|D|}$  into $m$ disjoint sub-datasets $\{D_j\}_{j=1}^{m}$ with corresponding disjoint union $D=\bigcup_{j=1}^mD_j$. Then we assign the information of corresponding data set $D_j$ to  one local machine (processor) to produce a local estimator $\beta_{t,D_j}$ via the algorithm \eqref{alg1}. These local estimators
are communicated to a central processor, the central processor synthesizes a global estimator $\overline{\beta_{t,D}}$
by taking the following weighted average
\bea
\overline{\beta_{t,D}}=\sum_{j=1}^m\f{|D_j|}{|D|}\beta_{t,D_j}.  \label{alg2}
\eea
 Then the corresponding {divide-and-conquer} based prediction $\overline{\eta_{t,D}}$ is obtained by 
$\overline{\eta_{t,D}}(X)=\int_\Omega \overline{\beta_{t,D}}(x)X(x)dx$. {We remark that, in the above model, the disjoint union $\bigcup_{j=1}^mD_j$ also includes the case when the data are stored naturally across multiple local machines in a distributive way, and they are not combined at the very beginning for the reasons of protecting privacy and reducing potential costs. Then the data partitioning step is not required, and in this case, the GDFL naturally cannot be done by a single machine or processor and a {divide-and-conquer} approach (DGDFL) has to be considered. There are many social examples belonging to this scenario. For example, in financial markets, the consumers' behavior data are stored in different institutions, these local institutions are not allowed to directly share data information to the others and their consumers' behavior data are not accessible to the public due to privacy {considerations}.
	 In medical systems, the clinical records from different medical institutions are often sensitive and cannot be shared. Thus it is difficult to analyze these sensitive data by directly combining them together. However, these institutions {desires} to collaboratively conduct training based on these clinical data to optimize medical decision-making under the premise of protecting their own clinical records.
		Our divide-and-conquer based learning algorithm DGDFL enables these local data holders to collaborate without directly sharing their data and improve the efficiency of analyzing functional data. For a given big data set with {the} same type of elements that are not naturally divided previously, there is no coercive restriction on the allocating manner of the data set $D$. In our model, the data can be allocated with great freedom. For example, for a labeled data set $\{X_i,Y_i\}_{i=1}^{|D|}$ which forms random copies of $(X,Y)$, we only need to allocate these data by randomly selecting $D_i$ elements from $D\backslash \bigcup_{j=0}^{i-1}D_j$ with $D_0=\emptyset$ at $i$-th step with $i=1,2,...,m$ according to any distribution. Then there are $D_i$ data being naturally allocated to $i$-th local machine and $D_i\bigcap D_j=\emptyset$ for any $i\neq j$.} As far as we know, the existing studies on stochastic gradient descent functional learning methods only appear very recently in references \cite{ctfg2022}, \cite{ggs2022}, \cite{mg2022} which focus on online  learning with sound convergence analysis  performed. However, these works are essentially related to single-machines and are relatively restricted when the functional data scale  is extremely large.  As a divide-and-conquer based training scheme, our divide-and-conquer functional learning algorithm can overcome the limitation and substantially reduce the computational burden in time
and memory. 

To investigate the approximation ability of our algorithms, we  establish estimates for the estimator $\eta_{t,D}$ associated with one single-processor first and comprehensively study the learning rates of the excess risk 
\bea
\huaE(\eta_{t,D})-\huaE(\eta^*)  \label{er1}
\eea
of the estimator $\eta_{t,D}$. The estimates related to $\eta_{t,D}$ would play an important role to further study our main estimator $\overline{\eta_{t,D}}$ and its associated excess risk 
 \bea
 \huaE(\overline{\eta_{t,D}})-\huaE(\eta^*).  \label{er2}
 \eea

We briefly summarize the contributions of the current work. To our best knowledge, the work is the first to propose a {divide-and-conquer} iterative approach  to study the functional linear model via DGDFL and provide solid convergence analysis. Under some mild conditions, we first establish basic error analysis and optimal rates for the GDFL algorithm \eqref{alg1}. This part of {the} main results for GDFL are also foundational and meaningful for the field of functional data analysis. Based on analysis of \eqref{alg1}, we comprehensively establish convergence analysis for DGDFL algorithm \eqref{alg2}. Optimal learning rates are obtained under mild conditions.   Our proofs also reveal the influence of two types of noise conditions on convergence results. It is shown that the noise condition on $\eee$ also influences the required maximal number $m$ of local processors to guarantee the optimal learning rates of the excess risk \eqref{er2}. Our main results also indicate that GDFL and DGDFL can overcome the saturation phenomenon on {the} regularity index of the target function suffered by  previous works in functional learning. Furthermore, based on our DGDFL algorithm, we also establish a semi-supervised DGDFL algorithm by introducing additional unlabeled functional data. We show that with unlabeled data, the restriction on $m$ can be relaxed, even when $\beta^*$ satisfies a weaker regularity restriction (detailed discussions are provided in section \ref{mainresults_discussion}).

\section{Main results and discussions}\label{mainresults_discussion}
  We denote the standard inner product $\nn f,g\mm=\int_{\Omega}f(x)g(x)dx$ and $L^2(\Omega)$-norm $\|f\|_{L^2(\Omega)}=(\int_{\Omega}|f(x)|^2dx)^{1/2}$ for any measurable functions $f$, $g$ defined on $\Omega$.   For a real, symmetric, square integrable and nonnegative definite function $R:\Omega\times\Omega\rightarrow \mbb R$, use $L_R:L^2(\Omega)\rightarrow L^2(\Omega)$ to denote the integral operator
\bea
L_R(f)(\cdot)=\left\nn R(x,\cdot), f\right\mm=\int_\Omega R(x,\cdot)f(x)dx. \label{lrdef}
\eea
For this operator, the spectral theorem implies that there exists a set of normalized eigenfunctions
$\{\psi_k^R:k\geq1\}$ and a sequence of eigenvalues $\ooo_1^R\geq\ooo_2^R\geq\cdots>0$ such that
$R(x,y)=\sum_{k\geq1}\ooo_k^R\psi_k^R(x)\psi_k^R(y), \ \ x,y\in\Omega$,
and
$L_R(\psi_k^R)=\ooo_k^R\psi_k^R, \ k=1,2,...$
Then the square root operator of $L_R$ can be defined by
$L_R^{1/2}(\psi_k^R)=L_{R^{1/2}}(\psi_k^R)=\sqrt{\ooo_k^R}\psi_k^R$,
where
$R^{1/2}(x,y)=\sum_{k\geq1}\sqrt{\ooo_k^R}\psi_k^R(x)\psi_k^R(y)$.
We also define 
$(R_1R_2)(x,y)=\int_{\Omega}R_1(x,u)R_2(u,y)du$.
Then it is easy to see $L_{R_1R_2}=L_{R_1}L_{R_2}$. {For any two self-adjoint operators $L_1$ and $L_2$, we write $L_1\succeq L_2$ if $L_1-L_2$ is positive semi-definite.}
In functional learning, the covariance function $C$ of $X$ is an important object which is defined as 
\bes
C(x,y)=\mbb E\Big[X(x)X(y)\Big], \forall x,y\in \Omega.
\ees
It is easy to see that the covariance function $C$ is symmetric and positive semi-definite.  In this paper, we assume that $C$ is continuous and therefore $C$ is a Mercer kernel. Then the corresponding operator $L_C$ can be defined accordingly with $R$ replaced by $C$ in \eqref{lrdef} and $L_C$ is compact, positive semi-definite and of trace class.  Due to the reason that $K$ and $C$ are Mercer kernels on the compact set $\Omega$, there exist positive finite constants $B_C$ and $B_K$ such that 
\bes
B_C=\sup_{x\in \Omega}\sqrt{C(x,x)}, \ \ B_K=\sup_{x\in \Omega}\sqrt{K(x,x)}.
\ees
Hence the spectral norms of $L_C$ and $L_K$ can be bounded by $\|L_C\|\leq B_C^2$ and $\|L_K\|\leq B_K^2$. Given a Mercer kernel $K$ and covariance function $C$, we define a composite operator 
\bea
T_{C,K}=L_K^{1/2}L_CL_K^{1/2}.   \label{tck}
\eea
If we use $\overline{\huaH_K}$ to denote the closure of $\huaH_K$ in $L^2(\Omega)$, then it is well-known that $L_K^{1/2}:\overline{\huaH_K}\rightarrow\huaH_K$ is an isomorphism, namely, $\|L_K^{1/2}f\|_{K}=\|f\|_{L^2(\Omega)}$ for $f\in\overline{\huaH_K}$. In this paper, for brevity, we assume that $\overline{\huaH_K}=L^2(\Omega)$.
We use the effective dimension $\huaN(\la)$ to measure the regularity of the operator $T_{C,K}$ that is defined to be the trace of the operator $(\la I+T_{C,K})^{-1}T_{C,K}$:
\bea
\huaN(\la)=\text{Tr}\left((\la I+T_{C,K})^{-1}T_{C,K}\right).  \label{effetivedef}
\eea
 We assume the following capacity condition that there exists a constant $c_0>0$ and some $\aaa\in(0,1]$ such that for any $\la>0$,
\bea
\huaN(\la)\leq c_0\la^{-\aaa}.  \label{capacity}
\eea
The effective dimension \eqref{effetivedef} and the decaying condition \eqref{capacity} have been widely considered in  learning theory of kernel ridge regression problems (e.g. \cite{clz2017}, \cite{lgz2017}, \cite{gls2019}, \cite{yhsz2021}, \cite{sl2022},{\cite{tong2021}}). The condition is slightly more general than the corresponding {entropy} assumption in the seminal work \cite{cy2012} where a polynomial decaying condition $\ww\ooo_k\leq ck^{-2r}$ on eigenvalues $\{\ww{\ooo}_k\}$ of the operator $T_{C,K}$ {associated with the composite kernel $K^{1/2}CK^{1/2}$,}
is used for some constant $c$ and $r>1/2$ where $\{(\ww\ooo_k,\ww\psi_k)\}$ are eigenpairs of  $T_{C,K}$. In fact, an easy calculation shows that $\ww\ooo_k\leq ck^{-2r}$ implies  $\huaN(\la)\leq c'\la^{-\aaa}$ with $\aaa=\f{1}{2r}$. {Thus our assumption is more general than the entropy assumption.} {Additionally, the above decaying condition is satisfied for some well-known kernel function classes such as Sobolev classes and Besov classes that are commonly considered, thereby ensuring the meaningfulness of the capacity assumption in a large number of practical  occasions.}
To establish theoretical results for our GDFL and DGDFL algorithms, we also assume the following boundedness condition for predictor $X$, that is, there is an absolute constant $\kkk$ such that
\bea
\left\|L_K^{1/2}X\right\|_{L^2(\Omega)}\leq\kkk.  \label{lk1/2}
\eea
{Introduced by \cite{tn2018}, this technical assumption} has been adopted in recent studies {on functional linear models} \cite{tn2018}, \cite{tong2021}, \cite{tong2023} and in parts of main results in \cite{ls2022}. {Similar to the idea of assuming the input space of data samples to be a compact space in prior art of statistical learning theory e.g. \cite{clz2017}, \cite{lgz2017}, \cite{hwz2020} this assumption can be understood as a natural extension of boundedness condition on the predictor $X$ in the scenario of functional linear model \eqref{flmodel}. For example, an easy analysis shows that, if $X$ lies in the bounded subset $U=\{X\in L^2(\Omega): \|X\|_{L^2(\Omega)}\leq B\}$, then it is easy to discern that $\|L_K^{1/2}X\|\leq B\|K\|_{L^2(\Omega\times\Omega)}^{1/2}$. Additionally, just as pointed out by references e.g. \cite{tong2021}, \cite{tong2023}, the real-world data-sampling processes are usually bounded, so the assumption is reasonable and accessible in many practical settings.} 
In addition, we need a mild regularity condition on $\beta^*$, that is, there exists some $\theta\in[0,\infty)$ and function $g^*\in L^2(\Omega)$ such that
\bea
L_K^{-1/2}\beta^*= T_{C,K}^\theta g^*. \label{regularity}
\eea
{The technical assumption \eqref{regularity} can be treated as a regularity condition of the target function (functional) in the  functional-linear-model scenario. This assumption has been considered in the learning theory of functional linear models of prior art such as \cite{ggs2022}, \cite{mg2022}, \cite{tn2018} and  \cite{tong2021}  for establishing main results.   If $L_C\succeq \tau L_K^t$ for some $\tau>0$ and $t>0$, then it is easy to discern that the assumption \eqref{regularity} is guaranteed when $\beta^*\in L_K^{1/2+\ttt(t+1)}\left(L^2(\Omega)\right)$ with $\ttt\in[0,1/(2+2t)]$. This form coincides with the widely-adopted regularity assumption $\beta^*\in L_K^{\ttt}\left(L^2(\Omega)\right)$ with $\ttt\in[0,1]$ in a large literature of the kernel-based learning theory of prior art e.g. \cite{clz2017}, \cite{lgz2017}, \cite{lz2018}, \cite{gls2019}, \cite{hwz2020},  \cite{sl2022} thereby showing an obvious relationship between the current assumption and these regularity assumptions.  It is also well understood that, in learning theory, for algorithms to learn a target function $\beta^*$ based on a data set, a non-trivial learning rate (convergence rate) often depends on the regularity of the target function e.g. \cite{clz2017}, \cite{lgz2017}, \cite{tn2018}, \cite{gls2019}, \cite{hwz2020}, \cite{tong2021}, \cite{yhsz2021}, \cite{sl2022}, \cite{ggs2022}.  
	%An obvious feature of 
The	assumption \eqref{regularity} in fact implies the target slope function $\beta^*$ lies in the underlying space $\huaH_K$. We also note that there also exist rather mature practical techniques to simulate the learning performance of algorithms in RKHS. Thus all the above discussions indicate the regularity assumption \eqref{regularity} is a reasonable assumption in the current setting.}
In this paper, we consider two types of  noise conditions. The first is the second moment condition
\be
\mbb E\left[|\eee|^2\right]\leq\sigma^2. \label{2momc}
\ee
{Assumption \eqref{2momc} is a very common and standard technical assumption on random noise in  functional regression e.g. \cite{cy2012}, \cite{tn2018}, \cite{fls2019}, \cite{ggs2022}, \cite{ls2022}, \cite{mg2022}}. We also consider  the following well-known moment condition which is slightly stricter than \eqref{2momc}.  That is,
 there exist $\sigma>0$ and $M>0$ such that for any integer $p\geq2$,
\be
\mbb E[|\eee|^p]\leq\f{1}{2}\sigma^2p!M^{p-2}. \label{pmomc}
\ee
  Condition \eqref{pmomc} {is usually referred to as Bernstein condition and} often appears in the setting of kernel-based learning theory {by} impos{ing} restrictions on the performance of random variables e.g. \cite{gs2013}, \cite{Wellner2013}, \cite{wh2019a}, \cite{tong2021}. Noises that satisfy condition \eqref{pmomc} include the well-known {noises encountered in practice such as} Gaussian noises, sub-Gaussian noises, {noises with compactly supported distributions and noises associated with some types of exponential distributions. Hence, in practical settings, the noise assumptions in this paper are reasonable and easily verifiable}. {In the subsequent sections of this paper, we aim to establish a comprehensive set of convergence results for GDFL, DGDFL, and semi-supervised DGDFL. To achieve this, we will establish our main theorems by considering these two widely recognized random noise conditions within a unified analytical framework.}

\subsection{Gradient descent functional learning algorithm} \label{section21}
Our first main result reveals the convergence ability of the basic GDFL algorithm \eqref{alg1}. We establish explicit optimal learning rates of the excess risk \eqref{er1} of $\eta_{t,D}$.
% by considering two types of noise conditions \eqref{2momc}, \eqref{pmomc}.
\begin{thm}\label{gdfl_thm1}
Assume conditions \eqref{capacity}-\eqref{regularity} hold. Let the	stepsize be selected as $\gamma_k=\f{\gamma}{(k+1)^{\mu}}, 0\leq\mu<1, 0<\gamma\leq \f{1}{(1+\kkk)^2(1+B_C^2B_K^2)}$, total iteration step $t=\left\lfloor|D|^{\f{1}{(2\theta+\aaa+1)(1-\mu)}}\right\rfloor$.
If noise condition \eqref{2momc} holds, we have, with probability at least $1-\delta$,
 \bes
 \huaE(\eta_{t,D})-\huaE(\eta^*)	\leq
 \begin{dcases}
 	C_1^*\f{1}{\delta^2}\left(\log \f{16}{\delta}\right)^6|D|^{-\f{1}{\aaa+1}}\left(\log|D|\right)^2,\quad \ttt=0,\\
 	C_2^*\f{1}{\delta^2}\left(\log \f{16}{\delta}\right)^6|D|^{-\f{2\ttt+1}{2\ttt+\aaa+1}} ,\quad \ttt>0,\\
 \end{dcases}
 \ees
 and if  noise condition \eqref{pmomc} holds, we have, with probability at least $1-\delta$,
 \bes
 \huaE(\eta_{t,D})-\huaE(\eta^*)\leq
 \begin{dcases}
 	C_1^*\left(\log \f{16}{\delta}\right)^6|D|^{-\f{1}{\aaa+1}}\left(\log|D|\right)^2,\quad \ttt=0,\\
 	C_2^*\left(\log \f{16}{\delta}\right)^6|D|^{-\f{2\ttt+1}{2\ttt+\aaa+1}} ,\quad \ttt>0,\\
 \end{dcases}
 \ees
 $C_1^*$ and $C_2^*$ are absolute constants given in the proof.
\end{thm}
Theorem \ref{gdfl_thm1} establishes  confidence-based convergence rates of GDFL \eqref{alg1}. We can see that when $\ttt>0$, optimal learning rates can be always derived. Even when $\ttt=0$, a confidence-based optimal rate up to logarithmic factor which is minimal effect can also be obtained. The results also enrich the understands of functional learning in the existing literature.

The next main result reveals confidence-based learning rates of the estimator $\beta_{t,D}$ generated from GDFL \eqref{alg1} in terms of the RKHS norm $\|\cdot\|_K$.

\begin{thm}\label{gdfl_thm2}
	Assume conditions \eqref{capacity}, \eqref{lk1/2} hold and \eqref{regularity} holds for some $\ttt>0$. Let the	stepsize be selected as $\gamma_k=\f{\gamma}{(k+1)^{\mu}}, 0\leq\mu<1, 0<\gamma\leq \f{1}{(1+\kkk)^2(1+B_C^2B_K^2)}$, total iteration step $t=\left\lfloor|D|^{\f{1}{(2\theta+\aaa+1)(1-\mu)}}\right\rfloor$. Then we have with probability at least $1-\delta$,
	\bes
	\left\|\beta_{t,D}-\beta^*\right\|_K	\leq
	\begin{dcases}
		C_3^*\f{1}{\delta}\left(\log \f{16}{\delta}\right)^3|D|^{-\f{\ttt}{2\ttt+\aaa+1}},\quad if \ {\textbf{the noise condition}}\ \eqref{2momc}\ holds,\\
		C_3^*\left(\log \f{16}{\delta}\right)^3|D|^{-\f{\ttt}{2\ttt+\aaa+1}},\quad if \ {\textbf{the noise condition}} \ \eqref{pmomc}\ holds,\\
	\end{dcases}
	\ees
	with absolute constant $C_3^*$ given in the proof.
\end{thm}

\subsection{Distributed gradient descent functional  learning algorithm} \label{section22}
Our next result establishes  explicit confidence-based learning rates of DGDFL.
\begin{thm}\label{dgdthm1}
	Assume conditions \eqref{capacity}-\eqref{regularity} hold. Let the	stepsize be selected as $\gamma_k=\f{\gamma}{(k+1)^{\mu}}, 0\leq\mu<1, 0<\gamma\leq \f{1}{(1+\kkk)^2(1+B_C^2B_K^2)}$, total iteration step $t=\left\lfloor|D|^{\f{1}{(2\theta+\aaa+1)(1-\mu)}}\right\rfloor$ and $|D_1|=|D_2|=\cdots=|D_m|$. If noise condition \eqref{2momc} holds and total number $m$ of local machines satisfies 
\bea
	m	
	\begin{dcases}
		=1,\quad \ttt=0,\\
		\leq\f{\left\lfloor|D|^{\f{1}{2\ttt+\aaa+1}}\right\rfloor^{\f{\ttt}{2}}}{(\log|D|)^{\f{5}{2}}},\quad \ttt>0,\\
	\end{dcases}
	 \label{disfl_m1}
	\eea
	 there holds
	\bes
		\huaE(\overline{\eta_{t,D}})-\huaE(\eta^*)\leq 
		\begin{dcases}
			C_1^*\f{1}{\delta^2}\left(\log \f{16}{\delta}\right)^6|D|^{-\f{1}{\aaa+1}}\left(\log|D|\right)^2,\quad \ttt=0,\\
			C_4^*\f{1}{\delta^2}\left(\log\f{64}{\delta}\right)^8|D|^{-\f{2\ttt+1}{2\ttt+\aaa+1}},\quad \ttt>0,\\
		\end{dcases}
		%\right. 
	\ees
	with $C_4^*$ being an absolute constant.
\end{thm}

\begin{thm}\label{dgdthm2}
	Assume conditions \eqref{capacity}-\eqref{regularity} hold. Let the	stepsize be selected as $\gamma_k=\f{\gamma}{(k+1)^{\mu}}, 0\leq\mu<1, 0<\gamma\leq \f{1}{(1+\kkk)^2(1+B_C^2B_K^2)}$, total iteration step $t=\left\lfloor|D|^{\f{1}{(2\theta+\aaa+1)(1-\mu)}}\right\rfloor$ and $|D_1|=|D_2|=\cdots=|D_m|$. If noise condition \eqref{pmomc} holds and total number $m$ of local machines satisfies 
	\bea
	m	
	\begin{dcases}
		=1,\quad \ttt=0,\\
		\leq\f{\left\lfloor|D|^{\f{1}{2\ttt+\aaa+1}}\right\rfloor^\ttt}{(\log|D|)^{5}},\quad \ttt>0,\\
	\end{dcases}
	 \label{disfl_m2}
	\eea
	 then we have, with probability at least $1-\delta$,
	\bes
	\huaE(\overline{\eta_{t,D}})-\huaE(\eta^*)\leq %\left\{
	\begin{dcases}
	C_1^*\left(\log \f{16}{\delta}\right)^6|D|^{-\f{1}{\aaa+1}}\left(\log|D|\right)^2,\quad \ttt=0,\\
		C_4^*\left(\log\f{64}{\delta}\right)^8|D|^{-\f{2\ttt+1}{2\ttt+\aaa+1}},\quad \ttt>0.\\
	\end{dcases}
	%\right. 
	\ees
\end{thm}
After establishing the results of Theorem \ref{dgdthm1} and Theorem \ref{dgdthm2}, the effectiveness of the DGDFL has been clearly understood. We observe from the results in Theorem \ref{dgdthm1} and Theorem \ref{dgdthm2}, there is an obvious difference {in} the requirements of the maximal number $m$ of local processors to guarantee the optimal learning rates of the excess risk $\huaE(\overline{\eta_{t,D}})-\huaE(\eta^*)$ in \eqref{disfl_m1} and \eqref{disfl_m2}. This difference reflects the influence of the two types of noise conditions \eqref{2momc} and \eqref{pmomc}. The detailed reason for  raising such a difference can be found in the estimates from the proof  in Subsection \eqref{thm3thm4proof}. In the literature of  regression analysis for massive data, divide-and-conquer based kernel ridge regression has been intensively studied in the past decade \cite{zdw2015}, \cite{clz2017}, \cite{gls2019}, \cite{hwz2020}, \cite{sw2021}. In the setting of functional linear regression,  no result on divide-and-conquer based Tikhonov RLS  functional linear regression \eqref{rlseq} has  been established until the very recent works \cite{tong2021}, \cite{ls2022}. However, the computational complexity of Tikhonov RLS  functional linear regression scheme \eqref{rlseq} is $\huaO(|D|^3)$ which is much larger than $\huaO(|D|^2/m^2)$ of our DGDFL algorithm (see e.g. \cite{yra2007}). Hence the proposed DGDFL algorithm largely reduces the computational cost in contrast to previous Tikhonov RLS functional linear regression methods.

It can be witnessed that, under current conditions, our convergence rates can nicely overcome the saturation phenomenon of the regularity index $\ttt$ suffered  in some previous works on functional linear regression e.g.  \cite{tong2021}, \cite{ls2022} and online functional learning algorithm \cite{ctfg2022}, \cite{ggs2022} in some aspects. The saturation means that, beyond a critical index $\ttt_0$, improvement of $\ttt$ would not help to improve the convergence rates. Theorem \ref{gdfl_thm1} shows that our regularity range satisfies $\ttt\in[0,\infty)$, the convergence rates can always be improved when $\ttt$ increases and {remains} optimal. In contrast, for example, in \cite{tong2021}, to obtain  optimal rates of the RLS based functional regression method, a strict restriction on a narrow range $[0,1/2]$ of $\ttt$ is required. {The absence of the saturation phenomenon is mainly due to two ingredients. The first ingredient is the inherent advantage of the gradient descent type algorithm in overcoming the saturation phenomenon compared to the ridge regression/regularization schemes widely adopted in statistics, statistical learning theory and inverse problems.  This advantage of gradient descent has been observed in various studies outside the realm of functional learning, such as those mentioned in \cite{yra2007} and \cite{lz2018}. It is also widely recognized in theoretical and practical studies on learning theory that regularization schemes tend to saturate when the regularity index exceeds a certain threshold.
	The second ingredient is the novel utilization of integral operator techniques within the context of functional learning. The new incorporation of functional-learning-related error decomposition, along with the utilization of the integral operator techniques based on kernel $K$ and the covariance kernel $C$  also plays a crucial role in achieving a series of optimal learning rates without saturation.} The regularity condition \eqref{regularity} in fact implies that $\beta^*\in\huaH_K$ which is considered by previous works \cite{yc2010}, \cite{cy2012}, \cite{tong2021}. {The current existing works on divide-and-conquer functional learning in the empirical risk minimization (ERM) scheme \eqref{rlseq} are mainly \cite{ls2022} and \cite{tong2023}.} %\cite{ls2022} is a nice work which focuses on establishing some optimal min-max learning rates of the distributed estimator of ERM scheme in a hard learning scenario that the target function might not belong to RKHS.
	The main goal of the current work is to provide {an} understanding {of} the learning rates of the first batch of divide-and-conquer gradient descent type algorithms for functional learning.  In comparison to the two works, it is worth mentioning that, Theorem \ref{dgdthm1} and Theorem \ref{dgdthm2} demonstrate the significant impact of the noise moment condition of $\eee$ on the required maximum of the number $m$ of local processors to guarantee the optimal learning rates of the DGDFL estimator. In a nutshell, with the stricter moment performance satisfied by $\eee$, there will be more {a} relaxed requirement on the maximum of $m$.  Such a key phenomenon has been observed through the study of DGDFL and can also be witnessed in the following results on semi-supervised DGDFL in Theorem \ref{semithm1}. It would also be interesting and challenging to develop results of DGDFL for the case $\beta^*\notin\huaH_K$.

%\cite{ctfg2022}, \cite{ggs2022}, \cite{mg2022} for online functional learning setting and are essentially a single-processor model.

The following two direct corollaries on optimal learning rates in expectation and almost sure convergence can be easily obtained based on the result of  confidence-based learning rates in Theorem \ref{dgdthm2}.  
\begin{cor}\label{disflcor1}
 Under the assumptions of Theorem \ref{dgdthm2}, if $t=\left\lfloor|D|^{\f{1}{(2\theta+\aaa+1)(1-\mu)}}\right\rfloor$  and \eqref{disfl_m2} holds,
 	then we have, with probability at least $1-\delta$,
	\bes
\mbb E\Big[\huaE(\overline{\eta_{t,D}})-\huaE(\eta^*)\Big]\leq %\left\{
\begin{dcases}
	C_5^*|D|^{-\f{1}{\aaa+1}}\left(\log|D|\right)^{2},\quad \ttt=0,\\
	C_5^*|D|^{-\f{2\ttt+1}{2\ttt+\aaa+1}},\quad \ttt>0.\\
\end{dcases}
%\right. 
\ees
\end{cor}

\begin{cor}\label{disflcor2}
 Under the assumptions of Theorem \ref{dgdthm2}, if $t=\left\lfloor|D|^{\f{1}{(2\theta+\aaa+1)(1-\mu)}}\right\rfloor$  and \eqref{disfl_m2} holds,
 then for arbitrary $\eee>0$, there holds
 \bes
 \lim_{|D|\rightarrow+\infty}|D|^{\f{2\ttt+1}{2\ttt+\aaa+1}(1-\eee)}\Big[\huaE(\overline{\eta_{t,D}})-\huaE(\eta^*)\Big]=0, \ \ttt\geq0.
 \ees
\end{cor}

\subsection{Semi-supervised DGDFL algorithm}\label{semi-supervised_learning_section}
To enhance the performance of our DGDFL algorithm, we propose the semi-supervised DGDFL in this subsection by introducing unlabeled data in our DGDFL algorithm. One of the goals of doing so is to relax the restriction on the maximal number of local machines. The idea of introducing unlabeled data is mainly inspired by our earlier work on semi-supervised learning  with kernel ridge regularized least squares regression \cite{clz2017}.  We use the notation $D_j(X)=\{X:(X,Y)\in D_j \ \text{for} \  \text{some} \ Y\}$, $j=1,2,...,m$.
We assume that, in each local machine, in addition to the labeled data, we have a sequence of unlabeled data denoted by 
\bea
\ww{D}_j(X)=\{X_1^j,X_2^j,...,X_{|\ww{D}_j|}^j\}, \ j=1,2,...,m.
\eea
Then we can introduce the training data set associated with labeled and unlabeled data in $j$-th local machine (processor) as 
\bea
D_j^*=D_j\cup\ww{D}_j=\left\{(X_i^*,Y_i^*)\right\}_{i=1}^{|D_j^*|}
\eea
with 
\bes
X_i^*	=\left\{
\begin{aligned}
	X_i, \text{if} \ X_i\in D_j(X),\\
	\ww{X}_i, \text{if} \ 	\ww{X}_i\in \ww{D}_j(X),\\  
\end{aligned}
\right. 
\ \ \text{and} \ \ 
Y_i^*	=%\left\{
\begin{dcases}
	\f{|D_j^*|}{|D_j|}Y_i, \ \ \ \text{if} \ (X_i,Y_i)\in D_j,\\
	0, \text{otherwise}.\\  
\end{dcases}
%\right.
\ees
Let $D^*=\cup_{j=1}^mD_j^*$, then we can use the following semi-supervised divide-and-conquer gradient descent functional learning algorithm
\bea
\overline{\beta_{t,D^*}}=\sum_{j=1}^m\f{|D_j^*|}{|D^*|}\beta_{t,D_j^*},   \label{semiDGDFL}
\eea
and the semi-supervised divide-and-conquer gradient descent functional learning estimator is given by
\bea
\overline{\eta_{t,D^*}}(X)=\int_{\Omega}\overline{\beta_{t,D^*}}(x)X(x)dx, \ \  X\in L^2(\Omega).
\eea
Throughout the paper, we use the function $H(\ttt)$ to denote 
\bea
H(\ttt)	=\left\{
\begin{aligned}
	1,\quad \ttt=0,\\
	0,\quad \ttt>0.\\
\end{aligned}
\right.   \label{hdef}
\eea
\begin{thm}\label{semithm1}
Assume conditions \eqref{capacity}-\eqref{regularity} hold. Let the	stepsize be selected as $\gamma_k=\f{\gamma}{(k+1)^{\mu}}, 0\leq\mu<1, 0<\gamma\leq \f{1}{(1+\kkk)^2(1+B_C^2B_K^2)}$, total iteration step $t=\left\lfloor|D|^{\f{1}{(2\theta+\aaa+1)(1-\mu)}}\right\rfloor$, $|D_1|=|D_2|=\cdots=|D_m|$ and $|D_1^*|=|D_2^*|=\cdots=|D_m^*|$. If the noise condition \eqref{2momc} holds and $m$ satisfies
	\bea
	m\leq\f{\min\left\{|D^*|^{\f{1}{4}}|D|^{-\f{\aaa+1}{4(2\ttt+\aaa+1)}},|D^*|^{\f{1}{5}}|D|^{\f{2\ttt+\aaa-1}{5(2\ttt+\aaa+1)}}\right\}}{\left(\log|D|\right)^{\f{H(\ttt)+5}{2}}}, \ \ttt\geq0,    \label{semi1m}
	\eea
	we have, with probability at least $1-\delta$,
	\bes
	\huaE(\overline{\eta_{t,D^*}})-\huaE(\eta^*)\leq C_6^*\f{1}{\delta^2}\left(\log\f{64}{\delta}\right)^8|D|^{-\f{2\ttt+1}{2\ttt+\aaa+1}}(\log|D|)^{2H(\ttt)}, \ \ttt\geq0.
	\ees
	 If the noise condition \eqref{pmomc} holds and $m$ satisfies
	\bea
	m\leq\f{\min\left\{|D^*|^{\f{1}{2}}|D|^{-\f{\aaa+1}{2(2\ttt+\aaa+1)}},|D^*|^{\f{1}{3}}|D|^{\f{2\ttt+\aaa-1}{3(2\ttt+\aaa+1)}}\right\}}{\left(\log|D|\right)^{H(\ttt)+5}}, \ \ttt\geq0,    \label{semi2m}
	\eea
	we have, with probability at least $1-\delta$,
	\bes
	\huaE(\overline{\eta_{t,D^*}})-\huaE(\eta^*)\leq C_6^*\left(\log\f{64}{\delta}\right)^8|D|^{-\f{2\ttt+1}{2\ttt+\aaa+1}}(\log|D|)^{2H(\ttt)}, \ \ttt\geq0.
	\ees
\end{thm}

In Theorem \ref{semithm1}, we establish confidence-based optimal learning rates for our semi-supervised DGDFL algorithm. We can see that, by introducing unlabeled data via this  semi-supervised DGDFL algorithm, this result can relax the restriction on $m$ in contrast to Theorems \ref{dgdthm1} and \ref{dgdthm2}. For example, under the noise condition \eqref{pmomc}, when $\ttt>0$, if $|D^*|=|D|$, it is easy to see $|D^*|^{\f{1}{2}}|D|^{-\f{\aaa+1}{2(2\ttt+\aaa+1)}}=|D|^{\f{\ttt}{2\ttt+\aaa+1}}$ and $|D^*|^{\f{1}{3}}|D|^{\f{2\ttt+\aaa-1}{3(2\ttt+\aaa+1)}}>|D|^{\f{\ttt}{2\ttt+\aaa+1}}$. Then the condition \eqref{semi2m} reduces to 
\bes
m\leq\f{|D|^{\f{\ttt}{2\ttt+\aaa+1}}}{\left(\log|D|\right)^{5}}
\ees
which coincides with \eqref{disfl_m2}. However, when we assign some larger $|D^*|>|D|$, the merit of utilizing unlabeled data can be obviously witnessed even for  the case $\ttt=0$. We demonstrate this idea by selecting the sample size of the training data set $D^*$ as
\bea
|D^*|=|D|^{\f{s_1\ttt+s_2\aaa+1}{2\ttt+\aaa+1}}, s_1\geq5, s_2\geq3.
\eea
Then we know from Theorem \ref{semithm1} that the corresponding range of the number $m$ of local machines is
\bea
m\leq\f{\min\left\{|D|^{\f{s_1\ttt+(s_2-1)\aaa}{2(2\ttt+\aaa+1)}},|D|^{\f{(s_1+2)\ttt+(s_2+1)\aaa}{3(2\ttt+\aaa+1)}}\right\}}{\left(\log|D|\right)^{H(\ttt)+5}}, \ \ttt\geq0.      \label{ms1s2}
\eea
It is easy to see 
\bes
|D|^{\f{s_1\ttt+(s_2-1)\aaa}{2(2\ttt+\aaa+1)}},|D|^{\f{(s_1+2)\ttt+(s_2+1)\aaa}{3(2\ttt+\aaa+1)}}>|D|^{\f{\ttt}{2\ttt+\aaa+1}},
\ees
therefore  \eqref{ms1s2} is weaker than \eqref{disfl_m2}. Moreover, even when $\ttt=0$, \eqref{ms1s2} reduces to 
\bea
m\leq\f{\min\left\{|D|^{\f{(s_2-1)\aaa}{2(\aaa+1)}},|D|^{\f{(s_2+1)\aaa}{3(\aaa+1)}}\right\}}{\left(\log|D|\right)^{6}}.     \label{ms1s2ttt=0}
\eea
It is obvious to see that \eqref{ms1s2ttt=0} allows freer selection of $m$ since $m$ can be selected to be larger when $|D|$ increases, while in \eqref{disfl_m2}, the range of $m$ is only limited to $1$ when $\ttt=0$. These facts indicate some advantages of establishing a  semi-supervised DGDFL algorithm by introducing unlabeled data.

{
\subsection{Some further remarks}
	\subsubsection{Remarks on the notion ``distributed learning''}
	We remark that, the adoption of the term ``distributed'' in this paper is to emphasize that our divide-and-conquer-based algorithm DGDFL is mainly designed for scenarios where data is stored in a distributed manner and cannot be shared among local machines, {while it is parallel since there is a global synchronization at a parameter server}.
	%This concept differs from classical parallel computations in the field of optimization, which is primarily focused on accelerating computations.
	The literature has commonly referred to the classical distributed learning scheme where local machines/agents communicate with their neighbors to realize local information updating as ``decentralized distributed learning'' e.g. \cite{no2009}, \cite{kprr2018}, \cite{rrr2020}, \cite{yhy2022} while the parallel computation approach with a central server for distributively stored data is referred to as ``centralized distributed learning'' or ``divide-and-conquer distributed learning'' e.g. \cite{zdw2015}, \cite{lgz2017}, \cite{lc2018}, \cite{sw2021}, \cite{tong2021}, \cite{sl2022}. {In this paper, when we mention distributed gradient descent, it means the divide-and-conquer distributed gradient descent approach.}
}
{
\subsubsection{Remarks on decentralized kernel-based distributed learning}
 In the previous subsections, we have discussed the relationship between our current work and related studies on the learning theory of functional linear models. We remark that, in addition to the divide-and-conquer distributed learning scheme, there is another approach called decentralization that is used to develop distributed algorithms in RKHSs. This approach has been explored in works such as \cite{kprr2018}, \cite{rrr2020} and \cite{xwct2021} that allow direct information communications among local agents in a decentralized manner. The earlier work on decentralized distributed algorithms mainly lies in the field of multi-agent consensus distributed optimization e.g. \cite{no2009}. Recent studies have just turned to designing decentralized algorithms by constructing consensus optimization models in the framework of RKHSs.   In fact, our work, which considers functional learning based on intrinsically infinite-dimensional random function data, differs significantly from the works \cite{kprr2018}, \cite{rrr2020}, and \cite{xwct2021}, which concentrate on Euclidean data. In the following,  we describe some obvious differences {in} problem formulations and theoretical approaches to clearly distinguish our current work from these references.  
}
  %We compare the, 
  %and theoretical results presented in our main %theorems with those of the decentralized distributed learning works in RKHS that allow direct information communication among local agents in a decentralized manner  e.g. \cite{kprr2018}, \cite{rrr2020}, \cite{xwct2021}. 
%Before coming into comparing the current work with these works, 
%}

%{
{Firstly, we remark that, there is a significant and fundamental distinction in the problem formulation between the current work and the main works on decentralized distributed learning in RKHSs that have been mentioned. In this work and related work on functional linear models, the main objective is to recover the  functional $\eta^*:X\mapsto\int_{\Omega}\beta^*(x)X(x)dx$ in the model  \eqref{flmodel}
%the estimation of the slope function $\beta^*$ 
 from an input data space consisting of random functions (predictors) that   {lie} in $L^2(\Omega)$ to an output (response) space. The   sample of random functions $\{X(x):x\in\Omega\}$ forms a stochastic process/field with sample paths in $L^2(\Omega)$ (see also e.g. e.g. \cite{cy2012}, \cite{tn2018}, \cite{fls2019}, \cite{tong2021}, \cite{ctfg2022},  \cite{ggs2022}, \cite{ls2022}, \cite{mg2022},  \cite{tong2023}).
 %$\{X_1, X_2, ..., X_{|D|}\}\subset L^2(\Omega)$.
  One notable characteristic of this model is that the functional covariates (random sample) $\{X_i\}_{i=1}^{|D|}$ in input space are intrinsically infinite-dimensional $L^2(\Omega)$ data which include random functions or curves frequently encountered in modern neuroscience and econometrics. This is in contrast to the works on decentralized distributed learning that primarily focus on conventional regression models involving Euclidean data and aim to recover the target function defined on a Euclidean space e.g. \cite{kprr2018}, \cite{rrr2020}, \cite{xwct2021}.  Consequently, there exists a significant distinction in the problem formulation, namely the sampling process, the data form and the ultimate goal of approximation.}
%}

%{
 {To further distinguish the current work from the work in decentralized distributed kernel-based learning, we describe the main %main theorems and 
 theoretical frameworks of these references. In \cite{kprr2018}, by imposing a consensus constraint among neighbors of two agents, the work successfully transforms the problem of learning the target function into a multi-agent consensus distributed optimization (MCDO) problem (e.g. \cite{no2009}) 
 %Specifically, after  enforcing an approximate consensus for applying representer theorem, 
 with the local objective functions in \cite{kprr2018} being the penalty functional consisting of the summation of Tikhonov-regularized expected loss functions based on local data and an approximation consensus term. The theoretical approach in \cite{kprr2018} mainly focuses on MCDO through consensus SGD within the framework of RKHSs. For conducting the convergence analysis of  an online distributed gradient descent,  \cite{kprr2018} imposes a conventional gradient boundedness condition for the objective function which is widely adopted in the literature on multi-agent optimization.  Notably, the disagreement analysis in \cite{kprr2018} stands out as a main feature of consensus-based techniques.    %In the current work, we do not require any functional gradient boundedness condition associated with regularized least squares scheme \eqref{rlseq} to establish our convergence theory. 
 The work \cite{xwct2021} also formulates the problem as an MCDO problem in the random feature space, then utilizes a distributed ADMM-based communication-censored to solve it. Rademacher complexity is the main tool utilized by \cite{xwct2021}  for the corresponding learning rate analysis.
 %We present a concise comparison between our current work and \cite{xwct2021} regarding the main theoretical results concerning generalization performance in the key theorems.  We note the in \cite{xwct2021}, even if the number of random feature is taken to be sufficiently large,  the excess risk fails to achieve an optimal rate in relation to the data sample size. One of the main reasons is due to the limitation of Rademacher complexity utilized by \cite{xwct2021}  for the corresponding learning rate analysis. 
  In contrast, our current work takes a completely different and innovative approach by employing integral operators in the context of functional learning models based on random function samples. As a result, we are able to obtain optimal learning rates for the excess risk of all the functional estimators $\eta_{t,D}$, $\overline{\eta_{t,D}}$ and $\overline{\eta_{t,D^*}}$ studied in this work. Hence, the advantages of integral operator-based theoretical approaches for deriving optimal learning rates in the framework of functional learning/data analysis are clearly reflected. 
 On the other hand, 
 %with the idea from multi-agent consensus decentralized optimization \cite{no2009}, 
  the work \cite{rrr2020} introduces a  doubly-stochastic communication matrix to construct a decentralized gradient descent. The basic procedure is that, each local agent performs a local gradient descent with respect to their own data, subsequently, each agent performs averaging operations with its neighbors, facilitating information communications through the utilization of a communication weight matrix. Based on these descriptions, it is easy to clearly distinguish the current work from the references on decentralized kernel learning.}

  {In the context of functional linear model \eqref{flmodel}, the aforementioned methods in references \cite{kprr2018}, \cite{rrr2020}, \cite{xwct2021} and the conventional techniques in \cite{gs2013}, \cite{clz2017}, \cite{lz2018},  \cite{gls2019}, \cite{hwz2020}, \cite{sl2022}  cannot be directly applied. In contrast to previous kernel-based learning, the difficulty of the prediction problem in functional linear model depends on both the kernels $K$ and $C$. The analysis and derived rates depend on the kernel complexity of $K^{1/2}CK^{1/2}$ (as observed in e.g. \cite{cy2012}). Thus it can be witnessed the covariance kernel $C$ of the  random predictor $X$ (random function), its empirical version $C_{\widehat D}$ and their associated integral operators integral operators $T_{C,K}$,  $T_{\widehat C_D,K}$ are introduced in this work and they play a significant role for deriving main theoretical results (such as optimal learning rates of  $\eta_{t,D}$, $\overline{\eta_{t,D}}$ and $\overline{\eta_{t,D^*}}$ in previous sections) throughout this work. The approaches of utilizing these operators are also  essentially different from conventional kernel learning problems in \cite{gs2013}, \cite{clz2017}, \cite{lz2018},  \cite{kprr2018}, \cite{gls2019}, \cite{hwz2020}, \cite{rrr2020}, \cite{xwct2021}, \cite{sl2022} which do not require these further approaches. The corresponding novelties are also clearly reflected throughout this paper.}
   
   {We also remark that,  the semi-supervised learning approach in subsection \ref{semi-supervised_learning_section} is also another important novelty, compared with the aforementioned work in this remark on decentralized kernel-based learning which has not developed theoretical results. In this work, our main results demonstrate the significance of incorporating unlabeled data of random functions in order to increase the number of data subsets 
   $m$ and potentially enhance scalability. Theorem \ref{semithm1} shows that by further considering our semi-supervised learning scheme, one can still obtain optimal learning rates by utilizing our analytical approaches while allowing for much greater flexibility in the total number $m$ of local machines/processors. It is interesting to note that the decentralized approach has not been established for the learning theory of functional learning problems based on samples of random functions. It would be valuable to develop appropriate decentralized distributed learning schemes for the functional learning problem addressed in our work. Additionally, establishing a decentralized semi-supervised functional data analysis scheme would be a challenging and worthwhile endeavor. The basic algorithm GDFL \eqref{alg1} and its associated main results in Theorem \ref{gdfl_thm1} and Theorem \ref{gdfl_thm2} provide a potential foundation for developing these decentralized functional learning algorithms in future work.}

  {In summary, there are significant differences in problem formulation/background, theoretical approaches, and main results between our work and the previous work on decentralized kernel-based learning. Through the discussion above, we have effectively distinguished our work from theirs and highlighted {the} contributions of this work.}
{
\subsubsection{Some advantages of DGDFL in privacy protection and discussion}
It is important to note that information communications among different local machines often come with the risk of privacy disclosure. However, the divide-and-conquer scheme considered in our work offers a high level of privacy protection because it does not allow direct data information communications among agents.   This is particularly advantageous in scenarios such as the financial market, where consumer behavior data stored in different commercial institutions are not accessible to the public due to privacy considerations. Similarly, in the medical system, clinical records of a medical organization cannot be shared with or owned by different medical institutions to protect privacy. However, these medical organizations may need to collaboratively conduct classification based on the medical data to optimize medical decision-making, without compromising the privacy of their own clinical records.  The methods proposed in our work provide effective solutions for these scenarios.  Our divide-and-conquer based distributed learning algorithm DGDFL enables these local data holders (modeled as nodes) to collaborate without directly sharing their data information with their neighbors to realize a local updating process that  many decentralized distributed learning schemes considered. This scheme has also contributed to the recent rapid development of federated learning (e.g. \cite{mm2017}, \cite{wyz2023}) which often utilizes an outer fusion center/master to aggregate the estimates of local processors/agents for protecting privacy. 
On the other hand,  by allowing information communications among local agents/processors in some decentralized schemes, the efficiency of the corresponding algorithms can be enhanced in certain settings. It is worth mentioning that the choice between the divide-and-conquer and decentralized approaches in applications depends on specific situations and requirements. }
{
\subsubsection{Remarks on scalability and possible future kernel approximation approaches}
The time complexity of our proposed DGDFL is significantly lower, with a time complexity of $\o(|D|^2/m^2)$, compared to the regularized ridge functional regression scheme \eqref{rlseq} with a time complexity of $\o(|D|^3)$. This clearly demonstrates the scalability advantage of DGDFL over the regularized ridge functional regression scheme. For future treatment of our proposed algorithms for extremely large-scale applications in functional statistical models, it would be intriguing to incorporate kernel approximation tools into our proposed algorithms, namely GDFL, divide-and-conquer DGDFL, and semi-supervised DGDFL. To make our algorithms more scalable to extremely large-scale sample size, note that random features are important tools for parameterization in kernel spaces for the merit of  reducing the memory footprint and hence reducing the computational burden and complexity. While random features have not been well developed in the functional learning problem addressed in this paper,  we discuss potential future treatments and difficulties that involve applying random feature techniques  
%may be a possibly convenient tools that can be applied
 to  the algorithms GDFL, DGDFL and semi-supervised DGDFL. The fundamental concept behind utilizing random features is to parameterize functions in a kernel space using a set of finite-dimensional feature maps that map elements from the data space to a Euclidean space. One popular example is the random Fourier features which are commonly employed to approximate positive definite kernels like Gaussians.  For Euclidean data points $x,x'\in \mbb R^M$ and a properly scaled kernel $k(x)$ with its Fourier transform $p(w)=\hat{k}(w)$, one can take a feature map  $\phi_N: \mbb R^M\rightarrow\mbb R^N: x\mapsto N^{-1/2}(\cos(w_1^Tx+b_1), \cos(w_2^Tx+b_2),...,\cos(w_N^Tx+b_N))$ as an approximation of kernel $k$ (approximately $k(x,x')\approx\phi_N(x)^T\phi_N(x')$) where $w_1,w_2,...,w_N$ and  $b_1,b_2,...,b_N$ are sampled independently from $p(w)$ and uniformly from $[0,2\pi]$. Then one can parameterize an RKHS  function $f\in\huaH_k$ by $w\in\mbb R^N$ in terms of $f(x)=w^T\phi_N(x)$ \cite{rr2007}, \cite{rr2017}, \cite{rrr2020}, \cite{msz2023}. In the context of functional learning based on random function samples, the situation becomes more complex. Direct utilization of random features would be challenging since the data sample we encountered in this work consists of random functions $X$ from $L^2(\Omega)$ instead of Euclidean data. Additionally, the kernel we need to rigorously handle in this paper is the composite kernel $K^{1/2}CK^{1/2}$, rather than the simpler kernel $K$. This fundamental difference significantly increases the difficulty of incorporating random feature techniques into the functional linear model. It is also worth noting that, an obvious feature of the kernel $K^{1/2}CK^{1/2}$ is that it is generally not a shift-invariant kernel, which further complicates the theoretical realization of our algorithm using random features.  Thus, for {the} theoretical and practical realization of our algorithm via random features, one must address the crucial influence of the covariance kernel $C$ in addition to $K$. 
 %a further appropriate and crucial parameterization for the functional covariates should be rigorously conducted. We think the parameterization for the functional covariates would rely on the random points from the space $\Omega$ which is the underlying space of kernel $K$. Additionally, the  and the conduction of the parameterized algorithm might follow a two-stage process of which the idea has been widely adopted in the literature on distribution regression that focuses on predicting from distribution samples to response variables e.g. \cite{yhsz2021}. 
  As far as we know, the theoretical understanding of the random feature approaches to the functional learning scenario discussed in this paper is still an open question and falls outside the scope of the current work.
  % and theoretical results via random feature subspaces or kernel approximations in the scenario of functional linear model remains open,
   Even for GDFL, the implementation of random features has not been carried out. Similarly, for the establishment of other kernel approximation approaches such as Nystr\"om approximation (e.g. \cite{jymlz2013}) and kernel-based sparse projections (e.g. \cite{kwsr2019}) for the function learning problem in this work, some issues mentioned above also need to be rigorously addressed, and we leave them for future work.   }
{
\subsubsection{Remarks on essential differences from conventional (regularized) linear regression}
 The problem in this work is to recover the functional $\eta^*$, based on intrinsically infinite-dimensional samples consisting of $L^2(\Omega)$ functions, in contrast to the conventional (regularized) linear regression which aims at regressing from Euclidean points (finite-dimensional space) to the output space. In the existing literature, handling random function samples and handling Euclidean samples follow totally different routes. That is also the reason why we introduce concepts such as covariance kernels $C$ associated with random function $X$ and the integral operators $T_{C,K}$ associated with the composite kernel $K^{1/2}CK^{1/2}$. These elements are essential for constructing the theoretical framework and analyzing methods for the problem in current work, which are not required in the conventional (regularized) linear regression. Moreover,  the {recovery} of a functional is also deeply related to the estimation of  $\beta^*$ which is intrinsically an infinite-dimensional slope function, rather than a scalar in conventional linear regression. Hence, based on these facts, the  approaches employed in this work differ significantly from conventional finite-dimensional (regularized) linear regression methods. We also refer to the reference \cite{ch2006} for further details on essential distinctions.  }
\section{Preliminary results}
\subsection{Approximation error of a data-free iterative GDFL  algorithm}
%algorithm:
%\bea
%\beta_{k+1,D}=\beta_{k,D}-\gamma_k\f{1}{|D|}\sum_{i=1}^{|D|}\left(\int_\Omega\beta_{k,D}(x)X_i(x)dx-Y_i\right)\int_\Omega K(\cdot,x)X_i(x)dx 
%\eea
%$\beta_{0,D}=0$. Namely,
To perform further convergence analysis of GDFL and DGDFL, we need to first investigate a data-free iterative GDFL algorithm associated with the original algorithm \eqref{alg1} defined by
\bea
\beta_{t+1}=\beta_t-\gamma_tL_KL_C(\beta_t-\beta^*), \ \beta_0=0.   \label{datafreeseq}
\eea
For  simplifying notations, denote the  regularization polynomial $g_k$ by 
\bea
g_k(x)=\sum_{t=0}^{k-1}\gamma_t\pi_{t+1}^{k-1}(x) \label{gdef}
\eea
where 
\be
\pi_t^k(x)=\left\{
\begin{aligned}
	\prod_{i=t}^k(1-\gamma_ix) ,\quad t\leq k,\\
	1 ,\quad t>k.\\
\end{aligned}
\right. \label{deftau}
\ee
Also denote the residue polynomial (see e.g. \cite{yra2007})
\bes
r_t(x)=1-xg_t(x)=\pi_0^{t-1}(x).
\ees
The following lemma is from \cite{ctfg2022}.
\begin{lem}\label{pilem}
	Let $L$ be a compact positive semi-definite operator on some real separate Hilbert space, such that $\|L\|\leq C_*$ for some $C_*>0$. Let $l\leq k$ and $\gamma_l$, $\gamma_{l+1}$, ..., $\gamma_k\in [0,1/C_*]$. Then when $\theta>0$, there holds,
	\bes
	\left\|L^\theta\pi_l^k(L)\right\|\leq\sqrt{\f{(\theta/e)^{2\theta}+C_*^{2\theta}}{1+\left(\sum_{j=l}^k\gamma_j\right)^{2\theta}}}, \ \text{and} \ \left\|\pi_l^k(L)\right\|\leq1.
	\ees
\end{lem}

The following result is about the approximation error related to the data-free function sequence $\{\beta_k\}$ generated by the data-free functional  learning algorithm \eqref{datafreeseq}. It is the foundation to further establish convergence analysis of GDFL and DGDFL in this paper.
\begin{thm} \label{datafreees}
Let $\beta^*$ satisfy the regularity condition \eqref{regularity}. If the stepsizes are selected as $\gamma_k=\gamma\f{1}{(k+1)^\mu}$, $0\leq\mu<1$, with  $\gamma$ satisfying $\gamma\leq\f{1}{(1+\kkk)^2B_C^2B_K^2}$, then we have, for $k\geq1$,
	\bes
	\left\|\beta_k-\beta^*\right\|_K,	\left\|L_K^{-1/2}\left(\beta_k-\beta^*\right)\right\|_{L^2(\Omega)}\leq\left\{
	\begin{aligned}
		\|g^*\|_{L^2(\Omega)} ,\quad\ttt=0,\\
		C_{\theta,\gamma}\left\|g^*\right\|_{L^2(\Omega)} k^{-\theta(1-\mu)} ,\quad \ttt>0,\\
	\end{aligned}
	\right. 
	\ees
	and
\bes
\left\|T_{C,K}^{1/2}L_K^{-1/2}\left(\beta_k-\beta^*\right)\right\|_{L^2(\Omega)}\leq C_{\theta+\f{1}{2},\gamma}\left\|g^*\right\|_{L^2(\Omega)} k^{-(\theta+\f{1}{2})(1-\mu)},
\ees
where the constant 
\bea
C_{\ttt,\gamma}=\sqrt{(\theta/e)^{2\theta}+[(1+\kkk)^2(1+B_C^2B_K^2)]^{2\theta}}\left(\f{2}{\gamma}\right)^{\ttt}.\label{ctg}
\eea
\end{thm}

\begin{proof}
	From the iteration \eqref{datafreeseq}, we know
	\bes
	\begin{aligned}
		L_K^{-1/2}\beta_{k+1}=&	L_K^{-1/2}\beta_k-\gamma_kL_K^{1/2}L_C(\beta_k-\beta^*)\\
		=&L_K^{-1/2}\beta_k-\gamma_kL_K^{1/2}L_CL_K^{1/2}L_K^{-1/2}(\beta_k-\beta^*).
	\end{aligned}
	\ees
	Due to the fact that $T_{C,K}=L_K^{1/2}L_CL_K^{1/2}$, we know
	\bea
	L_K^{-1/2}\beta_{k+1}=\left(I-\gamma_kT_{C,K}\right)L_K^{-1/2}\beta_k+\gamma_kT_{C,K}L_K^{-1/2}\beta^*. \label{datafree1}
	\eea
	Then an iteration implies 
	\bea
	L_K^{-1/2}\beta_k=g_k(T_{C,K})T_{C,K}L_K^{-1/2}\beta^*=\sum_{s=0}^{k-1}\gamma_s\pi_{s+1}^{k-1}(T_{C,K})T_{C,K}L_K^{-1/2}\beta^*.  \label{betakrep1}
	\eea
Hence we have 	
	\bea
	L_K^{-1/2}\beta_k-L_K^{-1/2}\beta^*=-r_k(T_{C,K})L_K^{-1/2}\beta^*=-r_k(T_{C,K})T_{C,K}^\theta g^*. \label{noa}
	\eea
	After taking $L^2(\Omega)$-norms on both sides, we have
	\bes
	\left\|L_K^{-1/2}\beta_k-L_K^{-1/2}\beta^*\right\|_{L^2(\Omega)}\leq\left\|r_k(T_{C,K})T_{C,K}^\theta g^*\right\|_{L^2(\Omega)}\leq\llll\|T_{C,K}^\theta\pi_0^{k-1}(T_{C,K})\rrrr\|\cdot\llll\|g^*\rrrr\|_{L^2(\Omega)}.
	\ees
	For $\theta=0$, Lemma \ref{pilem} implies that $\left\|\pi_0^{k-1}(T_{C,K})\right\|\leq1$ and
	\bes
	\left\|L_K^{-1/2}\beta_k-L_K^{-1/2}\beta^*\right\|_{L^2(\Omega)}\leq\left\|g^*\right\|_{L^2(\Omega)}.
	\ees
	For $\theta>0$, Lemma \ref{pilem} implies that
	\bes
\left\|L_K^{-1/2}\beta_k-L_K^{-1/2}\beta^*\right\|_{L^2(\Omega)}	\leq \sqrt{\f{(\theta/e)^{2\theta}+(B_C^2B_K^2)^{2\theta}}{1+\left(\sum_{j=0}^{k-1}\gamma_j\right)^{2\theta}}}\llll\|g^*\rrrr\|_{L^2(\Omega)}\leq C_{\theta}\left\|g^*\right\|_{L^2(\Omega)}\f{1}{\llll(\sum_{j=0}^{k-1}\gamma_j\rrrr)^{\theta}},
	\ees
where $$C_{\ttt}=\sqrt{(\theta/e)^{2\theta}+[(1+\kkk)^2(1+B_C^2B_K^2)]^{2\theta}}.$$	After using the trivial fact that $\sum_{t=0}^{k-1}\gamma_t=\gamma\sum_{t=0}^{k-1}\f{1}{(t+1)^{\mu}}\geq\f{\gamma}{2}k^{1-\mu}$ and noting that $\|L_K^{-1/2}\beta_k-L_K^{-1/2}\beta^*\|_{L^2(\Omega)}=\left\|\beta_k-\beta^*\right\|_{K}$, we obtain the first inequality.

For the second inequality, note from \eqref{noa} that 
\bes
\left\|T_{C,K}^{1/2}L_K^{-1/2}\left(\beta_k-\beta^*\right)\right\|_{L^2(\Omega)}\leq\llll\|T_{C,K}^{\theta+\f{1}{2}}\pi_0^{k-1}(T_{C,K})\rrrr\|\cdot\llll\|g^*\rrrr\|_{L^2(\Omega)}.
\ees
Then using Lemma \ref{pilem} again and similar procedures with the above inequalities yield
\bes
\left\|T_{C,K}^{1/2}L_K^{-1/2}\left(\beta_k-\beta^*\right)\right\|_{L^2(\Omega)}\leq C_{\theta+\f{1}{2}}\left\|g^*\right\|_{L^2(\Omega)}\left(\f{2}{\gamma}\right)^{\theta+\f{1}{2}} k^{-(\theta+\f{1}{2})(1-\mu)}.
\ees
We conclude the proof by setting $C_{\ttt,\gamma}=C_{\ttt}\left(\f{2}{\gamma}\right)^{\ttt}$.
\end{proof}
\subsection{Empirical operator and basic lemmas}
Denote the empirical covariance function associated with the data set $D$ by
\bea
\wh{C}_D(x,y)=\f{1}{|D|}\sum_{(X,Y)\in D}X(x)X(y),
\eea
then we denote the corresponding empirical operator of $T_{C,K}$ by
\bea
T_{\wh{C}_D,K}=L_K^{1/2}L_{\wh{C}_D}L_K^{1/2}.   \label{tcdk}
\eea

The next result is a basic estimate on operators $T_{C,K}$ and  $T_{\wh{C}_D,K}$ that will be used later.
\begin{lem}\label{sumlalem}
	For the operators $T_{C,K}$ and $T_{\wh{C}_D,K}$, if the stepsizes satisfy $\gamma_k\leq\f{1}{(1+\kkk)^2(1+B_C^2B_K^2)}$, then for any $\la>0$, the following basic estimates hold,
	\bes
	\left\|\sum_{s=0}^{t-1}\gamma_s (\la I+T_{C,K})\pi_{s+1}^{t-1}(T_{C,K})\right\|\leq D_{\mu,\gamma}(1+\la t^{1-\mu}),
	\ees
	\bes
	\left\|\sum_{s=0}^{t-1}\gamma_s (\la I+T_{\wh{C}_D,K})\pi_{s+1}^{t-1}(T_{\wh{C}_D,K}))\right\|\leq D_{\mu,\gamma}(1+\la t^{1-\mu}),
	\ees
	with the constant $D_{\mu,\gamma}=\max\left\{1,\f{2\gamma}{1-\mu}\right\}$.
\end{lem}
\begin{proof}
	According to the representations of $T_{C,K}$ and $T_{\wh{C}_D,K}$ in \eqref{tck} and \eqref{tcdk}, we know the norms of $T_{C,K}$ and $T_{\wh{C}_D,K}$ satisfy 
	\bes
	\left\|T_{C,K}\right\|\leq B_C^2B_K^2<(1+\kkk)^2(1+B_C^2B_K^2)
	\ees
	and
	\bes
	\begin{aligned}
		\left\|T_{\wh{C}_D,K}\right\|=\sup_{\|f\|_{L^2(\Omega)}\leq1}\left\|\f{1}{|D|}\sum_{i=1}^{|D|}\left\nn L_K^{1/2}X_i,f\right\mm L_K^{1/2}X_i\right\|_{L^2(\Omega)}\leq&\sup_{\|f\|_{L^2(\Omega)}\leq1}\sup_{i=1,...,|D|}\left\|L_K^{1/2}X_i\right\|^2\|f\|_{L^2(\Omega)}^2\\
		\leq&\kkk^2\leq(1+\kkk)^2(1+B_C^2B_K^2).
	\end{aligned}
	\ees
	Then we know from Lemma \ref{pilem} with $L=T_{C,K}$, $T_{\wh{C}_D,K}$ that 
	\bes
	\left\|\pi_0^{t-1}(T_{C,K})\right\|, \left\|\pi_0^{t-1}\left(T_{\wh{C}_D,K}\right)\right\|\leq1
	\ees
	and hence
	\bes
	\left\|\sum_{s=0}^{t-1}\gamma_s T_{C,K}\pi_{s+1}^{t-1}(T_{C,K})\right\|=\left\|I-\pi_{0}^{t-1}(T_{C,K})\right\|\leq1.
	\ees
	Finally, we have
	\bes
	&&\left\|\sum_{s=0}^{t-1}\gamma_s (\la I+T_{C,K})\pi_{s+1}^{t-1}(T_{C,K})\right\|\leq\left\|\sum_{s=0}^{t-1}\gamma_s T_{C,K}\pi_{s+1}^{t-1}(T_{C,K})\right\|+\la\left\|\sum_{s=0}^{t-1}\gamma_s \pi_{s+1}^{t-1}(T_{C,K})\right\|\\
	&&	\leq1+\la\sum_{s=0}^{t-1}\gamma_s\leq  1+\la\gamma\left(1+\int_0^{t-1}\f{ds}{(1+s)^\mu}\right)\leq D_{\mu,\gamma}(1+\la t^{1-\mu})
	\ees
	where $ D_{\mu,\gamma}=\max\left\{1,\f{2\gamma}{1-\mu}\right\}$. The estimates for operator $T_{\wh{C}_D,K}$ follows in a similar way to that for $T_{C,K}$.
\end{proof}
We end this section with the following basic lemma from \cite{hwz2020} that will be used later. 
\begin{lem}\label{loglem}
If $0\leq\mu<1$, $\tau\geq0$, then 
$\sum_{s=1}^ts^{-(\mu+\tau)}\left(\sum_{j=s+1}^tj^{-\mu}\right)^{-1}\leq C_{\mu,\tau}'t^{-\min\{\tau,1-\mu\}}\log t$, $t\geq3$.
In particular, if $\tau=0$, there holds,
$\sum_{s=1}^ts^{-\mu}\left(\sum_{j=s+1}^tj^{-\mu}\right)^{-1}\leq 15\log t$, where $C_{\mu,\tau}'$ is an absolute constant defined by
\bes
C_{\mu,\tau}'	=3^\mu2^{\tau+2}(1-\mu)^2+1+\left\{
\begin{aligned}
	\f{2^{-\tau}[1-2^{-(1-\mu)}]^{-1}(1-\mu)}{1-\tau-\mu},\quad \tau+\mu<1,\\
	2[1-2^{-(1-\mu)}]^{-1}(1-\mu),\quad \tau+\mu=1,\\
	\f{[1-2^{-(1-\mu)}]^{-1}(1-\mu)(\tau+\mu)}{\mu+\tau-1},\quad \tau+\mu>1.\\
\end{aligned}
\right.   
\ees
\end{lem}
The original lemma is expressed with $0\leq\tau<1$. In fact, the original proof does not necessarily require $0\leq\tau<1$ and it is obvious that when $\tau\geq1$, the lemma automatically holds. Hence, we state it in the above form.
\section{Analysis of GDFL algorithm}

\subsection{Error analysis and error decomposition}

Let us start with the error analysis of the GDFL algorithm \eqref{alg1}. For our GDFL algorithm
\bes
\beta_{k+1,D}=\beta_{k,D}-\gamma_k\f{1}{|D|}\sum_{i=1}^{|D|}\left(\left\nn \beta_{k,D},X_i\right\mm-Y_i\right)L_KX_i, \ \beta_{0,D}=0,
\ees
using our operator representations of $T_{C,K}$ and $T_{\wh{C}_D,K}$ in \eqref{tck} and \eqref{tcdk}, we can  rewrite it as 
\bes
\beta_{k+1,D}=\beta_{k,D}-\gamma_k\llll(L_KL_{\wh{C}_D}\beta_{k,D}-\f{1}{|D|}\sum_{i=1}^{|D|}Y_iL_KX_i\rrrr).
\ees
Acting operation $L_K^{-1/2}$ on both sides of the above equality and noting that 
$$L_K^{1/2}L_{\wh{C}_D}\beta_{k,D}=\left(L_K^{1/2}L_{\wh{C}_D}L_K^{1/2}\right)L_K^{-1/2}\beta_{k,D}=T_{\wh{C}_D,K}L_K^{-1/2}\beta_{k,D},$$
we have
\bes
L_K^{-1/2}\beta_{k+1,D}=\left(I-\gamma_kT_{\wh{C}_D,K}\right)L_K^{-1/2}\beta_{k,D}+\f{\gamma_k}{|D|}\sum_{i=1}^{|D|}Y_iL_K^{1/2}X_i.
\ees
An iteration implies 
\bea
L_K^{-1/2}\beta_{t,D}=\sum_{s=0}^{t-1}\gamma_s\pi_{s+1}^{t-1}\llll(T_{\wh{C}_D,K}\rrrr)\f{1}{|D|}\sum_{i=1}^{|D|}Y_iL_K^{1/2}X_i. \label{datadep_rep1}
\eea
Return to data-free iteration \eqref{datafreeseq}, rewrite \eqref{datafree1} to 
\bea
L_K^{-1/2}\beta_{k+1}=\left(I-\gamma_kT_{\wh{C}_D,K}\right)L_K^{-1/2}\beta_k+\gamma_k\llll(T_{\wh{C}_D,K}-T_{C,K}\rrrr)L_K^{-1/2}\beta_k+\gamma_kT_{C,K}L_K^{-1/2}\beta^*. 
\eea
Then an iteration implies 
\bea
L_K^{-1/2}\beta_{t}=\sum_{s=0}^{t-1}\gamma_s\pi_{s+1}^{t-1}\llll(T_{\wh{C}_D,K}\rrrr)\llll[\llll(T_{\wh{C}_D,K}-T_{C,K}\rrrr)L_K^{-1/2}\beta_s+T_{C,K}L_K^{-1/2}\beta^*\rrrr]. \label{datafree_rep1}
\eea
Then we know \eqref{datadep_rep1} and \eqref{datafree_rep1} together imply
\bes
\begin{aligned}
L_K^{-1/2}\llll(\beta_{t,D}-\beta_t\rrrr)=&\sum_{s=0}^{t-1}\gamma_s\pi_{s+1}^{t-1}\llll(T_{\wh{C}_D,K}\rrrr)\Bigg[\llll(T_{C,K}-T_{\wh{C}_D,K}\rrrr)L_K^{-1/2}\beta_s\\
&+\f{1}{|D|}\sum_{i=1}^{|D|}Y_iL_K^{1/2}X_i-T_{C,K}L_K^{-1/2}\beta^*\Bigg],
\end{aligned}
\ees
which further gives the following error decomposition for $\beta_{t,D}$ as
\bes
	&&\llll\|T_{C,K}^{1/2}L_K^{-1/2}\llll(\beta_{t,D}-\beta_t\rrrr)\rrrr\|_{L^2(\Omega)}\\
	&&\leq\llll\|T_{C,K}^{1/2}\sum_{s=0}^{t-1}\gamma_s\pi_{s+1}^{t-1}\llll(T_{\wh{C}_D,K}\rrrr)\llll(T_{C,K}-T_{\wh{C}_D,K}\rrrr)L_K^{-1/2}(\beta_s-\beta^*)\rrrr\|_{L^2(\Omega)}\\
	&&+\llll\|T_{C,K}^{1/2}\sum_{s=0}^{t-1}\gamma_s\pi_{s+1}^{t-1}\llll(T_{\wh{C}_D,K}\rrrr)\llll(T_{C,K}-T_{\wh{C}_D,K}\rrrr)L_K^{-1/2}\beta^*\rrrr\|_{L^2(\Omega)}\\
	&&+\llll\|T_{C,K}^{1/2}\sum_{s=0}^{t-1}\gamma_s\pi_{s+1}^{t-1}\llll(T_{\wh{C}_D,K}\rrrr)\llll(\f{1}{|D|}\sum_{i=1}^{|D|}Y_iL_K^{1/2}X_i-T_{C,K}L_K^{-1/2}\beta^*\rrrr)\rrrr\|_{L^2(\Omega)}\\
	&&:=S_1+S_2+S_3.
\ees
Denote the following norms
\bea
\bd=\llll\|\llll(\la I+T_{C,K}\rrrr)\llll(\la I+T_{\wh{C}_D,K}\rrrr)^{-1}\rrrr\|,\\
\cd=\llll\|\llll(\la I+T_{C,K}\rrrr)^{-1/2}\left(T_{C,K}-T_{\wh{C}_D,K}\right)\rrrr\|,
\eea
\bea
\dd=\left\|\left(\la I+T_{C,K}\right)^{-1/2}\left(\f{1}{|D|}\sum_{(X,Y)\in D}YL_K^{1/2}X-T_{C,K}L_K^{-1/2}\beta^*\right)\right\|_{L^2(\Omega)}.    \label{dddef}
\eea
The next result gives a general error bound for $\left\|T_{C,K}^{1/2}L_K^{-1/2}\llll(\beta_{t,D}-\beta_t\rrrr)\right\|_{L^2(\Omega)}$.
\begin{pro}\label{cknorm}
	Let $\beta_{t,D}$ and $\beta_t$ be defined in GDFL \eqref{alg1} and data-free GDFL \eqref{datafreeseq} respectively. Assume  conditions \eqref{lk1/2} and \eqref{regularity} hold. Let the	stepsize be selected as $\gamma_k=\f{\gamma}{(k+1)^{\mu}}, 0\leq\mu<1, 0<\gamma\leq \f{1}{(1+\kkk)^2(1+B_C^2B_K^2)}$. Then for any $\la>0$, we have
	\bes
	\llll\|T_{C,K}^{1/2}L_K^{-1/2}\llll(\beta_{t,D}-\beta_t\rrrr)\rrrr\|_{L^2(\Omega)}\leq
		\left(\ww{C}_{5}+\ww{C}_{6}\right)\left((\log t)^{H(\ttt)}+\la t^{1-\mu}\right)\bd(\cd+\dd) 
	\ees
	where $H(\ttt)$ is defined as in \eqref{hdef} and $\ww{C}_{5}$ and $\ww{C}_{6}$ are some absolute constants given in the proof.
\end{pro}

\begin{proof}
	We make a decomposition for $S_1$ as 
\bes
\begin{aligned}
	S_1\leq&\llll\|\left(\la I+T_{C,K}\right)^{1/2}\sum_{s=0}^{t-1}\gamma_s\pi_{s+1}^{t-1}\llll(T_{\wh{C}_D,K}\rrrr)\llll(T_{C,K}-T_{\wh{C}_D,K}\rrrr)L_K^{-1/2}(\beta_s-\beta^*)\rrrr\|_{L^2(\Omega)}\\
	\leq&\sum_{s=1}^{t-1}\gamma_s \llll\|\llll(\la I+T_{C,K}\rrrr)^{1/2}\llll(\la I+T_{\wh{C}_D,K}\rrrr)^{-1/2}\rrrr\|\times \left\|\llll(\la I+T_{\wh{C}_D,K}\rrrr)\pi_{s+1}^{t-1}\left(T_{\wh{C}_D,K}\right)\right\|\\
	&\times\left\|\llll(\la I+T_{\wh{C}_D,K}\rrrr)^{-1/2}(\la I+T_{C,K})^{1/2}\right\|\times \llll\|\llll(\la I+T_{C,K}\rrrr)^{-1/2}\left(T_{C,K}-T_{\wh{C}_D,K}\right)\rrrr\|\\
	&\times \llll\| L_K^{-1/2}(\beta_s-\beta^*)\rrrr\|_{L^2(\Omega)}.
\end{aligned}
\ees
Using the fact that $\|A^sB^s\|\leq\|AB\|^s$, $0<s<1$, and $\|AB\|=\|BA\|$ for any two positive self-adjoint operators $A$, $B$ on a separable Hilbert space, the above inequality can be bounded by
\bea
\bd\cd\left(\sum_{s=0}^{t-1}\gamma_s\left\|T_{\wh{C}_D,K}\pi_{s+1}^{t-1}(T_{\wh{C}_D,K})\right\| \llll\|L_K^{-1/2}(\beta_s-\beta^*)\rrrr\|_{L^2(\Omega)}+\la\sum_{s=0}^{t-1}\gamma_s\llll\|L_K^{-1/2}(\beta_s-\beta^*)\rrrr\|_{L^2(\Omega)}\right).  \label{s1subes}
\eea
When $\ttt>0$,  $L_K^{-1/2}\beta^*=T_{C,K}^{\ttt}g^*$, Lemma \ref{pilem}, Lemma \ref{loglem} and the basic fact $s^{-\ttt(1-\mu)}\leq(s+1)^{-\ttt(1-\mu)}2^{\ttt(1-\mu)}$, $s\geq1$ imply that
\bea
\nono&&\sum_{s=0}^{t-1}\gamma_s\left\|T_{\wh{C}_D,K}\pi_{s+1}^{t-1}(T_{\wh{C}_D,K})\right\| \llll\|L_K^{-1/2}(\beta_s-\beta^*)\rrrr\|_{L^2(\Omega)}\\
\nono&&\leq \gamma(B_C^2B_K^2)^{\ttt}\kkk^2\|g^*\|_{L^2(\Omega)}+\sum_{s=1}^{t-1}\f{C_{1,\gamma}\|g^*\|_{L^2(\Omega)}}{(s+1)^{\mu}}\left(\sum_{j=s+1}^{t-1}\f{1}{(j+1)^\mu}\right)^{-1}C_{\ttt,\gamma}(s+1)^{-\ttt(1-\mu)}2^{\ttt(1-\mu)}\\
\nono&&= \gamma(B_C^2B_K^2)^{\ttt}\kkk^2\|g^*\|_{L^2(\Omega)}+\sum_{s=2}^{t}\f{C_{1,\gamma}\|g^*\|_{L^2(\Omega)}}{s^{\mu}}\left(\sum_{j=s}^{t-1}\f{1}{(j+1)^\mu}\right)^{-1}C_{\ttt,\gamma}s^{-\ttt(1-\mu)}2^{\ttt(1-\mu)}\\
&&= \gamma(B_C^2B_K^2)^{\ttt}\kkk^2\|g^*\|_{L^2(\Omega)}+C_{1,\gamma}C_{\ttt,\gamma}\|g^*\|_{L^2(\Omega)}2^{\ttt(1-\mu)}\sum_{s=2}^{t}s^{-\mu-\ttt(1-\mu)}\left(\sum_{j=s+1}^{t}\f{1}{j^\mu}\right)^{-1} \label{firstterm}
\eea
where $C_{1,\gamma}$ and $C_{\ttt,\gamma}$ are defined as in \eqref{ctg}. By using Lemma \ref{loglem},
\bea
\nono&&\sum_{s=0}^{t-1}\gamma_s\left\|T_{\wh{C}_D,K}\pi_{s+1}^{t-1}(T_{\wh{C}_D,K})\right\| \llll\|L_K^{-1/2}(\beta_s-\beta^*)\rrrr\|_{L^2(\Omega)}\\
\nono&&\leq\gamma(B_C^2B_K^2)^{\ttt}\kkk^2\|g^*\|_{L^2(\Omega)}+2^{\ttt(1-\mu)}C_{1,\gamma}C_{\mu,\ttt(1-\mu)}'\|g^*\|_{L^2(\Omega)}t^{-\min\{1-\mu,\ttt(1-\mu)\}}\log t\\
&&\leq\gamma(B_C^2B_K^2)^{\ttt}\kkk^2\|g^*\|_{L^2(\Omega)}+2^{\ttt(1-\mu)}C_{1,\gamma}C_{\mu,\ttt(1-\mu)}'\|g^*\|_{L^2(\Omega)}(e\min\{1-\mu,\ttt(1-\mu)\})^{-1}    \label{gammatk}
\eea
If we denote $$\ww{C}_{1}=\Bigg[\gamma(B_C^2B_K^2)^{\ttt}(\kkk^2+B_C^2B_K^2)+2^{\ttt(1-\mu)}C_{1,\gamma}C_{\mu,\ttt(1-\mu)}'(e\min\{1-\mu,\ttt(1-\mu)\})^{-1}\Bigg]\|g^*\|_{L^2(\Omega)},$$ then 
\bes
\sum_{s=0}^{t-1}\gamma_s\left\|T_{\wh{C}_D,K}\pi_{s+1}^{t-1}(T_{\wh{C}_D,K})\right\| \llll\|L_K^{-1/2}(\beta_s-\beta^*)\rrrr\|_{L^2(\Omega)}\leq \ww{C}_{1}.
\ees
Also, an easy calculation shows
\bes
\begin{aligned}
\sum_{s=0}^{t-1}\gamma_s\llll\|L_K^{-1/2}(\beta_s-\beta^*)\rrrr\|_{L^2(\Omega)}\leq&\gamma(B_C^2B_K^2)^{\ttt}\|g^*\|_{L^2(\Omega)}+\sum_{s=1}^{t-1}C_{\theta,\gamma}\left\|g^*\right\|_{L^2(\Omega)} (s+1)^{-\mu-\ttt(1-\mu)}2^{\ttt(1-\mu)}\\
\leq&\gamma(B_C^2B_K^2)^{\ttt}\|g^*\|_{L^2(\Omega)}+ C_{\theta,\gamma}\left\|g^*\right\|_{L^2(\Omega)}2^{\ttt(1-\mu)}\int_{1}^t\f{1}{(s+1)^{\mu}}.
\end{aligned}
\ees
Then if we denote $\ww{C}_{2}=\left(\gamma(B_C^2B_K^2)^{\ttt}+ C_{\ttt,\gamma}\f{2^{\ttt(1-\mu)}}{1-\mu}\right)\|g^*\|_{L^2(\Omega)}$,
we have
\bes
\sum_{s=0}^{t-1}\gamma_s\llll\|L_K^{-1/2}(\beta_s-\beta^*)\rrrr\|_{L^2(\Omega)}\leq \ww{C}_{2}t^{1-\mu}
\ees
and
\bes
S_1\leq \bd\cd \left(\ww{C}_{1}+ \ww{C}_{2}\la t^{1-\mu}\right).
\ees
When $\ttt=0$, following similar procedures as in \eqref{firstterm} and using Lemma \ref{loglem}, we have
\bea
\nono&&\sum_{s=0}^{t-1}\gamma_s\left\|T_{\wh{C}_D,K}\pi_{s+1}^{t-1}(T_{\wh{C}_D,K})\right\| \llll\|L_K^{-1/2}(\beta_s-\beta^*)\rrrr\|_{L^2(\Omega)}\\
\nono&&\leq\gamma\kkk^2\|g^*\|_{L^2(\Omega)}+ C_{1,\gamma}\|g^*\|_{L^2(\Omega)}2^{\ttt(1-\mu)}\sum_{s=2}^{t}s^{-\mu}\left(\sum_{j=s+1}^{t}\f{1}{j^\mu}\right)^{-1}\\
&&\leq \ww{C}_{3}\log t
\label{logbdd}\eea
where $\ww{C}_{3}=\left(\gamma\kkk^2+15C_{1,\gamma}2^{\ttt(1-\mu)}\right)\|g^*\|_{L^2(\Omega)}$. Also, it is easy to see when $\ttt=0$, 
\bea
\sum_{s=0}^{t-1}\gamma_s\llll\|L_K^{-1/2}(\beta_s-\beta^*)\rrrr\|_{L^2(\Omega)}\leq\gamma\|g^*\|_{L^2(\Omega)}\left(1+\f{t^{1-\mu}}{1-\mu}\right)\leq \ww{C}_{4}t^{1-\mu}  \label{gammmasum}
\eea
where $\ww{C}_{4}=\gamma(1+\f{1}{1-\mu})\|g^*\|_{L^2(\Omega)}$. Finally, Combining the above results for the case $\ttt>0$ and $\ttt=0$, we have
\bea
S_1	\leq\left\{
\begin{aligned}
 \ww{C}_{5}\left(1+\la t^{1-\mu}\right)\bd\cd,\quad \ttt>0,\\
	 \ww{C}_{5}\left(\log t+\la t^{1-\mu}\right)\bd\cd,\quad \ttt=0,\\
\end{aligned}
\right. 
\eea
where $\ww{C}_{5}=\max\left\{\ww{C}_{1},\ww{C}_{2},\ww{C}_{3},\ww{C}_{4}\right\}$.
 
 Now we estimate $S_2$ by making the following  decomposition,
\bes
\begin{aligned}
	S_2\leq& \llll\|\sum_{s=0}^{t-1}\gamma_s\left(\la I+T_{C,K}\right)^{1/2}\pi_{s+1}^{t-1}\llll(T_{\wh{C}_D,K}\rrrr)\llll(T_{C,K}-T_{\wh{C}_D,K}\rrrr)L_K^{-1/2}\beta^*\rrrr\|_{L^2(\Omega)}\\
	\leq& \llll\|\llll(\la I+T_{C,K}\rrrr)^{1/2}\llll(\la I+T_{\wh{C}_D,K}\rrrr)^{-1/2}\rrrr\|  \left\|\sum_{s=0}^{t-1}\gamma_s\left(\la I+T_{\wh{C}_D,K}\right)\pi_{s+1}^{t-1} \left(T_{\wh{C}_D,K}\right)\right\|\\
	&\times \left\|\llll(\la I+T_{\wh{C}_D,K}\rrrr)^{-1/2}(\la I+T_{C,K})^{1/2}\right\|\llll\|\llll(\la I+T_{C,K}\rrrr)^{-1/2}\left(T_{C,K}-T_{\wh{C}_D,K}\right)\rrrr\|\left\|T_{C,K}^\theta\right\|\left\|g^*\right\|_{L^2(\Omega)}.
\end{aligned}
\ees
By using Lemma \ref{sumlalem}, we obtain 
\bea
S_2\leq  \ww{C}_{6}\left(1+\la t^{1-\mu}\right)\bd\cd.
\eea
where $\ww{C}_{6}= (B_C^2B_K^2)^\theta D_{\mu,\gamma}(1+\|g^*\|_{L^2(\Omega)})$.
A similar procedure implies that 
\bea
S_3\leq\ww{C}_{6}\left(1+\la t^{1-\mu}\right)\bd\dd.
\eea
Combining the above estimates for $S_1$, $S_2$ and $S_3$ yields 
\bes
&&\llll\|T_{C,K}^{1/2}L_K^{-1/2}\llll(\beta_{t,D}-\beta_t\rrrr)\rrrr\|_{L^2(\Omega)}\\
&&\leq\left\{
\begin{aligned}
	\left(\ww{C}_{5}+\ww{C}_{6}\right)\left(1+\la t^{1-\mu}\right)\bd(\cd+\dd),\quad \ttt>0,\\
	\left(\ww{C}_{5}+\ww{C}_{6}\right)\left(\log t+\la t^{1-\mu}\right)\bd(\cd+\dd), \quad \ttt=0,\\
\end{aligned}
\right. 
\ees
which concludes the proof.
\end{proof}

\subsection{Deriving learning rates: proof of Theorem \ref{gdfl_thm1}}
In this subsection, we derive learning rates of the GDFL algorithm.
Denote $\Ab_{D,\la}$ as
\bes
\ad=\f{2\kkk^2\nu}{|D|\sqrt{\la}}+2\kkk\nu\sqrt{\f{\huaN(\la)}{|D|}}
\ees
with $\nu=\max\llll\{1,(B_C^2B_K^2)^\theta,1+\sqrt{c_0}\rrrr\}$. Firstly, we need a confidence-based upper bound for $\dd$ in terms of $\ad$. The following lemma is needed and it follows directly from \cite{tn2018} and \cite{tong2021}.
\begin{lem}\label{elem}
 Assume  condition \eqref{lk1/2} holds.	With probability at least $1-\delta$,
	\bes
	\left\|(\la I+T_{C,K})^{-1/2}\f{1}{|D|}\sum_{i=1}^{|D|}\eee_i L_K^{1/2}X_i\right\|_{L^2(\Omega)}	\leq 
	\begin{dcases}
		\f{\sigma}{2\kkk\delta}\ad ,\quad if \ condition \ \eqref{2momc} \ holds,\\
		\f{(M+\sigma)}{\kkk}\left(\log\f{2}{\delta}\right)\ad ,\quad if \ condition \ \eqref{pmomc} \ holds,\\
	\end{dcases} 
	\ees
\end{lem}
The next  lemma  \cite{pinelis1994} on Hilbert-valued random variables is needed.
\begin{lem}\label{concentration}
	For a random variable $\eta$ on $(Z,\rho)$ with values in a separable Hilbert space $(H,\|\cdot\|)$ satisfying $\|\eta\|\leq \ww{M}<\infty$ almost surely, and a random  sample $\{z_i\}_{i=1}^m\subseteq Z$ independent drawn according to $\rho$, there holds with probability $1-\delta$, 
	\bes
	\left\|\f{1}{m}\sum_{i=1}^m\eta(z_i)-\mbb E(\eta)\right\|\leq\f{2\ww{M}\log(2/\delta)}{m}+\sqrt{\f{2\mbb E(\|\eta\|^2)\log(2/\delta)}{m}}.
	\ees
\end{lem}
The next proposition provides our required confidence-based estimate for $\dd$.
\begin{pro}\label{ddest}
 Assume  conditions \eqref{lk1/2} and \eqref{regularity} hold.	With probability at least $1-\delta$, there holds
	\bea
	\dd	\leq%\left\{
	\begin{dcases}
		\left(\f{\sigma}{\kkk}+2\|g^*\|_{L^2(\Omega)}\right)\left(\f{1}{\delta}\log\f{4}{\delta}\right)\ad ,\quad if \ condition \ \eqref{2momc} \ holds,\\
		\left(\f{(M+\sigma)}{\kkk}+2\|g^*\|_{L^2(\Omega)}\right)\left(\log\f{4}{\delta}\right)\ad ,\quad if \ condition \ \eqref{pmomc} \ holds.\\
	\end{dcases}
	%\right. 
	\eea
\end{pro}

\begin{proof}
	Recall that the functional linear model gives
	\bes
	Y_i=\left\nn X_i,\beta^*\right\mm+\eee_i,
	\ees
	 we know the following decomposition for $\dd$ holds,
	\bes
	\begin{aligned}
		\dd\leq&\left\|(\la I+T_{C,K})^{-1/2}\left(\f{1}{|D|}\sum_{i=1}^{|D|}\left\nn X_i,\beta^*\right\mm L_K^{1/2}X_i-T_{C,K}L_K^{-1/2}\beta^*\right)\right\|_{L^2(\Omega)}\\
		&+ \left\|(\la I+T_{C,K})^{-1/2}\f{1}{|D|}\sum_{i=1}^{|D|}\eee_i L_K^{1/2}X_i\right\|_{L^2(\Omega)}.
	\end{aligned}
	\ees
	By using Lemma \ref{elem}, after scaling on $\delta$, we know with probability at least $1-\f{\delta}{2}$,
	\bea
	\left\|(\la I+T_{C,K})^{-1/2}\f{1}{|D|}\sum_{i=1}^{|D|}\eee_i L_K^{1/2}X_i\right\|_{L^2(\Omega)}	\leq%\left\{
	\begin{dcases}
		\f{\sigma}{\kkk\delta}\ad ,\quad if \ \eqref{2momc} \ holds,\\
		\f{(M+\sigma)}{\kkk}\left(\log\f{4}{\delta}\right)\ad,\quad if  \ \eqref{pmomc} \ holds.\label{ct1}
	\end{dcases}
	%\right. 
	\eea
	We turn to estimate the first term. Denote the random variable
	\bea
	\xi=\left(\la I+T_{C,K}\right)^{-1/2}\left\nn X,\beta^*\right\mm L_K^{1/2}X
	\eea
	which takes values in the space $L^2(\Omega)$. Note that 
	\bes
	\left\nn X,\beta^*\right\mm=	\left\nn X,L_K^{1/2}(L_K^{-1/2}\beta^*)\right\mm=	\left\nn L_K^{1/2}X,L_K^{-1/2}\beta^*\right\mm,
	\ees
	 we know
	\bes
	\begin{aligned}
		\|\xi\|_{L^2(\Omega)}=&\left\|\left(\la I+L_{C,K}\right)^{-1/2}\left\nn L_K^{1/2}X,L_K^{-1/2}\beta^*\right\mm L_K^{1/2}X\right\|_{L^2(\Omega)}\\
		\leq&\|L_K^{1/2}X\|_{L^2(\Omega)}^2\|L_K^{-1/2}\beta^*\|_{L^2(\Omega)}\times\f{1}{\sqrt{\la}}\leq\kkk^2(B_C^2B_K^2)^{\theta}\|g^*\|_{L^2(\Omega)}\f{1}{\sqrt{\la}}.
	\end{aligned}
	\ees
	Let  $\{(\ww\omega_j,\ww\psi_j)\}_j$ be a set of  normalized eigenpairs of $T_{C,K}$ on $L^2(\Omega)$ with $\{\ww\psi_k,k\geq1\}$ being an orthonormal basis of $L^2(\Omega)$.	Expand $L_K^{1/2}X=\sum_{j}\left\nn L_K^{1/2}X,\ww\psi_j\right\mm\ww\psi_j$, we have
	\bes
	\begin{aligned}
		\|\xi\|_{L^2(\Omega)}^2=&\left|\left\nn L_K^{1/2}X,L_K^{-1/2}\beta^*\right\mm\right|^2\left\|\sum_{j}\left\nn L_K^{1/2}X,\ww\psi_j\right\mm(\la I+T_{C,K})^{-1/2}\ww\psi_j\right\|_{L^2(\Omega)}^2\\
		\leq&\left\|L_K^{1/2}X\right\|_{L^2(\Omega)}^2\left\|L_K^{-1/2}\beta^*\right\|_{L^2(\Omega)}^2\left\|\sum_j\left\nn L_K^{1/2}X,\ww\psi_j\right\mm\f{1}{\sqrt{\la+\ww \omega_j}}\ww\psi_j\right\|_{L^2(\Omega)}^2\\
		\leq&\kkk^2(B_K^2B_C^2)^{2\theta}\|g^*\|_{L^2(\Omega)}^2\sum_{j}\f{\left|\left\nn L_K^{1/2}X,\ww\psi_j\right\mm\right|^2}{\la+\ww\omega_j}.
	\end{aligned}
	\ees
	After taking expectations, we have
	\bes
	\begin{aligned}
		\mbb E\Big[\|\xi\|_{L^2(\Omega)}^2\Big]\leq&\kkk^2(B_C^2B_K^2)^{2\theta}\|g^*\|_{L^2(\Omega)}^2\sum_j\f{\left\nn\mbb E\left\nn L_K^{1/2}X,\ww\psi_j\right\mm L_K^{1/2}X,\ww\psi_j\right\mm}{\la+\ww\omega_j}\\
		=&\kkk^2(B_C^2B_K^2)^{2\theta}\|g^*\|_{L^2(\Omega)}^2\sum_{j}\f{\left\nn T_{C,K}\ww\psi_j,\ww\psi_j\right\mm}{\la+\ww\omega_j}\\
		=&\kkk^2(B_C^2B_K^2)^{2\theta}\|g^*\|_{L^2(\Omega)}^2\sum_{j}\f{\ww\omega_j}{\la+\ww\omega_j}=\kkk^2(B_C^2B_K^2)^{2\theta}\|g^*\|_{L^2(\Omega)}^2\huaN(\la).
	\end{aligned}
	\ees
	On the other hand, it is easy to see
	\bes
	\mbb E \Big[\left\nn X,\beta^*\right\mm L_K^{1/2}X\Big]=L_K^{1/2}L_C\beta^*=T_{C,K}L_K^{-1/2}\beta^*.
	\ees
	Then using Lemma \ref{concentration}, we obtain 
	with with probability at least $1-\f{\delta}{2}$,
	\bea
	\nono&&\left\|(\la I+T_{C,K})^{-1/2}\left(\f{1}{|D|}\sum_{i=1}^{|D|}\left\nn X_i,\beta^*\right\mm L_K^{1/2}X_i-T_{C,K}L_K^{-1/2}\beta^*\right)\right\|_{L^2(\Omega)}=\left\|\f{1}{|D|}\sum_{i=1}^{|D|}\xi_i-\mbb E\xi\right\|_{L^2(\Omega)}\\
	\nono&&\leq\f{2\kkk^2(B_C^2B_K^2)^\theta\|g^*\|_{L^2(\Omega)}}{|D|\sqrt{\la}}\log\f{4}{\delta}+\sqrt{\f{2\kkk^2(B_K^2B_C^2)^{2\theta}\huaN(\la)\|g^*\|_{L^2(\Omega)}^2}{|D|}\log\f{4}{\delta}}\\
	&&\leq2\|g^*\|_{L^2(\Omega)}\left(\log\f{4}{\delta}\right)\ad.  \label{ct2}
	\eea
	Combining \eqref{ct1} and \eqref{ct2}, using the fact $1\leq2\log\f{4}{\delta}$, $\forall \delta\in(0,1]$, we complete the proof of the proposition.
\end{proof}

On the other hand, from \cite{tong2021}, we know with probability at least $1-\delta$, each of the following inequalities holds,
\bea
&&\bd\leq\left(\f{\ad\log\f{2}{\delta}}{\sqrt{\la}}+1\right)^2, \label{bdest}\\
&&\cd\leq\ad\log\f{2}{\delta}.  \label{cdest}
\eea
Therefore, combining the above two estimates with Lemma \ref{ddest}, $\bd$, $\cd$, $\dd$ can be bounded together by $\ad$ in a high confidence level.

It is easy to see that for any prediction estimator $\wh{\eta}_D$ based on data set $D$ associated with corresponding slope function $\wh{\beta}_D$ via $\wh{\eta}_D(X)=\nn\wh{\beta}_D,X\mm$, the following fact holds,
\bes
\begin{aligned}
	\huaE(\wh{\eta}_D)-\huaE(\eta^*)=&\mbb E_X\left[\wh{\eta}_D(X)-\eta^*(X)\right]^2\\
	=&\mbb E_X[\nn\wh{\beta}_D,X\mm-\nn\beta^*,X\mm]^2\\
	=&\mbb E_X[\nn L_K^{-1/2}(\wh{\beta}_D-\beta^*),L_K^{1/2}X\mm]^2\\
	=&\left\nn L_K^{-1/2}(\wh{\beta}_D-\beta^*), \mbb E_X\nn L_K^{-1/2}(\wh{\beta}_D-\beta^*),L_K^{1/2}X\mm L_K^{1/2}X\right\mm\\
	=&\nn L_K^{-1/2}(\wh{\beta}_D-\beta^*), T_{C,K}L_K^{-1/2}(\wh{\beta}_D-\beta^*)\mm\\
	=&\|T_{C,K}^{1/2}L_K^{-1/2}(\wh{\beta}_D-\beta^*)\|_{L^2(\Omega)}^2.  
\end{aligned}
\ees
Then we know for our proposed estimator $\eta_{t,D}(X)=\int_{\Omega}\beta_{t,D}(x)X(x)dx$, for any $\la>0$, there holds
\bea
\begin{aligned}
\huaE(\eta_{t,D})-\huaE(\eta^*)=&\left\|T_{C,K}^{1/2}L_K^{-1/2}(\beta_{t,D}-\beta^*)\right\|_{L^2(\Omega)}^2\\
\leq&2\left\|T_{C,K}^{1/2}L_K^{-1/2}(\beta_{t,D}-\beta_t)\right\|_{L^2(\Omega)}^2+2\left\|T_{C,K}^{1/2}L_K^{-1/2}(\beta_{t}-\beta^*)\right\|_{L^2(\Omega)}^2\\
\leq&\ww{C}_{7}\bd^2\left(\cd+\dd\right)^2\left((\log t)^{H(\ttt)}+\la t^{1-\mu}\right)^2\\
&+2C_{\theta+\f{1}{2},\gamma}^2\left\|g^*\right\|_{L^2(\Omega)}^2 t^{-(2\theta+1)(1-\mu)}, 
\end{aligned}   \label{eta_excessbdd}
\eea
where $\ww{C}_{7}=2\left(\ww{C}_{5}+\ww{C}_{6}\right)^2$. %and the function $H(\ttt)$ is defined as
%\bea
%H(\ttt)	=\left\{
%\begin{aligned}
%	1,\quad \ttt=0,\\
%	0,\quad \ttt>0.\\
%\end{aligned}
%\right. 
%\eea
\begin{proof}[Proof of Theorem \ref{gdfl_thm1}]

Combine \eqref{bdest}, \eqref{cdest} with Proposition \ref{ddest}, after scaling on $\delta$, we know with probability at least $1-\delta$, the following inequalities hold simultaneously,
\bea
&&\bd\leq\left(\f{\ad\log\f{8}{\delta}}{\sqrt{\la}}+1\right)^2\leq 4\left(\f{\ad}{\sqrt{\la}}+1\right)^2\left(\log\f{8}{\delta}\right)^2,\label{bd_ad}\\
&&\cd\leq\ad\log\f{8}{\delta},\label{cd_ad}\\
&&\dd	\leq%\left\{
\begin{dcases}
	4\left(\f{\sigma}{\kkk}+2\|g^*\|_{L^2(\Omega)}\right)\left(\f{1}{\delta}\log\f{16}{\delta}\right)\ad ,\quad if \  \eqref{2momc} \ holds,\\
	\left(\f{(M+\sigma)}{\kkk}+2\|g^*\|_{L^2(\Omega)}\right)\left(\log\f{16}{\delta}\right)\ad ,\quad if  \ \eqref{pmomc} \ holds.\\
\end{dcases}\label{dd_ad}
%\right. 
\eea
Then after combining these estimates with \eqref{eta_excessbdd}, we know that, if noise condition \eqref{2momc} holds, with probability at least $1-\delta$,
\bes
\begin{aligned}
	\huaE(\eta_{t,D})-\huaE(\eta^*)
	\leq&16\ww{C}_{7}\left(\f{\ad}{\sqrt{\la}}+1\right)^4\left(\log\f{8}{\delta}\right)^4\left(8\|g^*\|_{L^2(\Omega)}+1+\f{4\sigma}{\kkk}\right)^2\\
	&\times\f{1}{\delta^2}\left(\log\f{16}{\delta}\right)^2\ad^2\left((\log t)^{H(\ttt)}+\la t^{1-\mu}\right)^2+2C_{\theta+\f{1}{2},\gamma}^2\left\|g^*\right\|_{L^2(\Omega)}^2 t^{-(2\theta+1)(1-\mu)}.\\
	\leq&16\ww{C}_{7}\left(8\|g^*\|_{L^2(\Omega)}+1+\f{4\sigma}{\kkk}\right)^2\f{1}{\delta^2}\left(\log\f{16}{\delta}\right)^6\left(\f{\ad}{\sqrt{\la}}+1\right)^4\ad^2\\
	&\times\left((\log t)^{H(\ttt)}+\la t^{1-\mu}\right)^2+2C_{\theta+\f{1}{2},\gamma}^2\left\|g^*\right\|_{L^2(\Omega)}^2 t^{-(2\theta+1)(1-\mu)}.
\end{aligned}
\ees
If noise condition \eqref{pmomc} holds, we also have, with probability at least $1-\delta$,
\bes
\begin{aligned}
	\huaE(\eta_{t,D})-\huaE(\eta^*)
	\leq&16\ww{C}_{7}\left(\f{(M+\sigma)}{\kkk}+1+2\|g^*\|_{L^2(\Omega)}\right)^2\left(\log\f{16}{\delta}\right)^6\left(\f{\ad}{\sqrt{\la}}+1\right)^4\\
	&\times\ad^2\left((\log t)^{H(\ttt)}+\la t^{1-\mu}\right)^2+2C_{\theta+\f{1}{2},\gamma}^2\left\|g^*\right\|_{L^2(\Omega)}^2 t^{-(2\theta+1)(1-\mu)}.
\end{aligned}
\ees
When $\la=1/t^{1-\mu}$, $t=\left\lfloor|D|^{\f{1}{(2\theta+\aaa+1)(1-\mu)}}\right\rfloor$, using the condition $\huaN(\la)\leq c_0\la^{-\aaa}$ we can directly derive 
\bea
\ad\leq2\kkk^2\nu|D|^{\f{-1-2\aaa-4\ttt}{2(2\ttt+\aaa+1)}}+2\kkk\nu|D|^{\f{-2\ttt-1}{2(2\ttt+\aaa+1)}}\leq(2\kkk^2\nu+2\kkk\nu)|D|^{\f{-2\ttt-1}{2(2\ttt+\aaa+1)}}  \label{adrate}
\eea
and
\bea
\f{\ad}{\sqrt{\la}}+1\leq 2\kkk^2\nu|D|^{\f{-2\aaa-4\ttt}{2(2\ttt+\aaa+1)}}+2\kkk\nu|D|^{\f{-2\ttt}{2(2\ttt+\aaa+1)}}+1\leq2\kkk^2\nu+2\kkk\nu+1.        \label{adla}
\eea
Then we know if noise condition \eqref{2momc} holds, we have, with probability at least $1-\delta$,
\bes
\huaE(\eta_{t,D})-\huaE(\eta^*)	\leq%\left\{
\begin{dcases}
C_1^*\f{1}{\delta^2}\left(\log \f{16}{\delta}\right)^6|D|^{\f{-1}{\aaa+1}}\left(\log|D|\right)^2,\quad \ttt=0,\\
	C_2^*\f{1}{\delta^2}\left(\log \f{16}{\delta}\right)^6|D|^{\f{-2\ttt-1}{2\ttt+\aaa+1}} ,\quad \ttt>0,\\
\end{dcases}
%\right.
\ees
and if  noise condition \eqref{pmomc} holds, we have, with probability at least $1-\delta$,
\bes
\huaE(\eta_{t,D})-\huaE(\eta^*)	\leq%\left\{
\begin{dcases}
	C_1^*\left(\log \f{16}{\delta}\right)^6|D|^{\f{-1}{\aaa+1}}\left(\log|D|\right)^2,\quad \ttt=0,\\
	C_2^*\left(\log \f{16}{\delta}\right)^6|D|^{\f{-2\ttt-1}{2\ttt+\aaa+1}} ,\quad \ttt>0,\\
\end{dcases}
%\right.
\ees
where 
\bes
C_1^*&=&16\ww{C}_{7}(2\kkk^2\nu+2\kkk\nu+1)^6\left(\f{(M+5\sigma)}{\kkk}+1+8\|g^*\|_{L^2(\Omega)}\right)^2\\
&&\times\f{4}{(2\ttt+\aaa+1)^2(1-\mu)^2}+4C_{\f{1}{2},\gamma}^2\left\|g^*\right\|_{L^2(\Omega)}^2,
\ees
and
\bes
C_2^*=64\ww{C}_{7}(2\kkk^2\nu+2\kkk\nu+1)^6\left(\f{(M+5\sigma)}{\kkk}+1+8\|g^*\|_{L^2(\Omega)}\right)^2+2^{2\ttt+2}C_{\ttt+\f{1}{2},\gamma}^2\left\|g^*\right\|_{L^2(\Omega)}^2.
\ees
This completes the proof of Theorem \ref{gdfl_thm1}.
\end{proof}

\subsection{Convergence in the RKHS norm: proof of Theorem \ref{gdfl_thm2}}
To establish the learning rates of $\beta_{t,D}$ in terms of the RKHS norm, we  first consider an error decomposition for $\llll\|L_K^{-1/2}\llll(\beta_{t,D}-\beta_t\rrrr)\rrrr\|_{L^2(\Omega)}$ that will be used later. The proof of the results on the DGDFL algorithm in the next section  also relies on this proposition. 
\begin{pro}\label{knorm}
 Assume  conditions \eqref{lk1/2} and \eqref{regularity} hold. Let the	stepsize be selected as $\gamma_k=\f{\gamma}{(k+1)^{\mu}}, 0\leq\mu<1, 0<\gamma\leq \f{1}{(1+\kkk)^2(1+B_C^2B_K^2)}$. Then	for any data set $D=\{(X_i,Y_i)\}_{i=1}^{|D|}$ and any $\la>0$, we have
	\bes
	\llll\|L_K^{-1/2}\llll(\beta_{t,D}-\beta_t\rrrr)\rrrr\|_{L^2(\Omega)}\leq\ww{C}_{8}\|g^*\|_{L^2(\Omega)}\f{1}{\sqrt{\la}}\left((\log t)^{H(\ttt)}+\la t^{1-\mu}\right)\bd(\cd+\dd).
	\ees
\end{pro} 
\begin{proof}
	We start from the following decomposition,
	\bes
	\begin{aligned}
		\llll\|L_K^{-1/2}\llll(\beta_{t,D}-\beta_t\rrrr)\rrrr\|_{L^2(\Omega)}\leq&\llll\|\sum_{s=0}^{t-1}\gamma_s\pi_{s+1}^{t-1}\llll(T_{\wh{C}_D,K}\rrrr)\llll(T_{C,K}-T_{\wh{C}_D,K}\rrrr)L_K^{-1/2}(\beta_s-\beta^*)\rrrr\|_{L^2(\Omega)}\\
		&+\llll\|\sum_{s=0}^{t-1}\gamma_s\pi_{s+1}^{t-1}\llll(T_{\wh{C}_D,K}\rrrr)\llll(T_{C,K}-T_{\wh{C}_D,K}\rrrr)L_K^{-1/2}\beta^*\rrrr\|_{L^2(\Omega)}\\
		&+\llll\|\sum_{s=0}^{t-1}\gamma_s\pi_{s+1}^{t-1}\llll(T_{\wh{C}_D,K}\rrrr)\llll(\f{1}{|D|}\sum_{i=1}^{|D|}Y_iL_K^{1/2}X_i-T_{C,K}L_K^{-1/2}\beta^*\rrrr)\rrrr\|_{L^2(\Omega)}\\
		=:&S_1'+S_2'+S_3'.
	\end{aligned}
	\ees
	For $S_1'$, we make the decomposition as
	\bes
	\begin{aligned}
		S_1'\leq&\sum_{s=1}^{t-1}\gamma_s  \left\|\llll(\la I+T_{\wh{C}_D,K}\rrrr)\pi_{s+1}^{t-1}\left(T_{\wh{C}_D,K}\right)\right\|\times\left\|\llll(\la I+T_{\wh{C}_D,K}\rrrr)^{-1}(\la I+T_{C,K})\right\|\\
		&\times\left\|(\la I+T_{C,K})^{-1/2}\right\|\times \llll\|\llll(\la I+T_{C,K}\rrrr)^{-1/2}\left(T_{C,K}-T_{\wh{C}_D,K}\right)\rrrr\|\times \llll\| L_K^{-1/2}(\beta_s-\beta^*)\rrrr\|_{L^2(\Omega)}\\
		\leq&\bd\cd\f{1}{\sqrt{\la}}\sum_{s=1}^{t-1}\gamma_s  \left\|\llll(\la I+T_{\wh{C}_D,K}\rrrr)\pi_{s+1}^{t-1}\left(T_{\wh{C}_D,K}\right)\right\| \llll\| L_K^{-1/2}(\beta_s-\beta^*)\rrrr\|_{L^2(\Omega)}.
	\end{aligned}
	\ees
	Then following the same estimate as \eqref{s1subes}, we can derive
	\bes
	S_1'\leq \ww{C}_{5}\f{1}{\sqrt{\la}}\left((\log t)^{H(\ttt)}+\la t^{1-\mu}\right)\bd\cd.
	\ees
	For $S_2'$, we have following  decomposition,
	\bes
	\begin{aligned}
		S_2'\leq&  \left\|\sum_{s=0}^{t-1}\gamma_s\left(\la I+T_{\wh{C}_D,K}\right)\pi_{s+1}^{t-1} \left(T_{\wh{C}_D,K}\right)\right\|\times\left\|\llll(\la I+T_{\wh{C}_D,K}\rrrr)^{-1}(\la I+T_{C,K})\right\|\\
		&\times \left\|(\la I+T_{C,K})^{-1/2}\right\|\times \llll\|\llll(\la I+T_{C,K}\rrrr)^{-1/2}\left(T_{C,K}-T_{\wh{C}_D,K}\right)\rrrr\|\times\left\|T_{C,K}^\theta\right\|\left\|g^*\right\|_{L^2(\Omega)}.
	\end{aligned}
	\ees
	Then Lemma \ref{sumlalem} implies that
	\bes
	S_2'\leq
	\ww{C}_{6}\f{1}{\sqrt{\la}}\left(1+\la t^{1-\mu}\right)\bd\cd.
	\ees
	Similarly,
	\bes
	S_3'\leq
	\ww{C}_{6}\f{1}{\sqrt{\la}}\left(1+\la t^{1-\mu}\right)\bd\dd.
	\ees
	Combining estimates for $S_1'$, $S_2'$, $S_3'$, we have
	\bes
	\llll\|L_K^{-1/2}\llll(\beta_{t,D}-\beta_t\rrrr)\rrrr\|_{L^2(\Omega)}\leq\ww{C}_{8}\f{1}{\sqrt{\la}}\left((\log t)^{H(\ttt)}+\la t^{1-\mu}\right)\bd(\cd+\dd),
	\ees
	where $\ww{C}_{8}=\ww{C}_{5}+\ww{C}_{6}$.
\end{proof}
We are ready to give the proof of Theorem \ref{gdfl_thm2}.
\begin{proof}[Proof of Theorem \ref{gdfl_thm2}]
	We observe that the main difference between the error decompositions of $\llll\|L_K^{-1/2}\llll(\beta_{t,D}-\beta_t\rrrr)\rrrr\|_{L^2(\Omega)}$ and $\llll\|T_{C,K}^{1/2}L_K^{-1/2}\llll(\beta_{t,D}-\beta_t\rrrr)\rrrr\|_{L^2(\Omega)}$ in Proposition \ref{cknorm} comes from the additional terms $\f{1}{\sqrt{\la}}$. Other terms share the same estimates. Hence when taking $t=\left\lfloor|D|^{\f{1}{(2\theta+\aaa+1)(1-\mu)}}\right\rfloor$ and $\la=\f{1}{t^{1-\mu}}$, we can directly use the established error bounds for $\bd$, $\cd$, $\dd$ in \eqref{bd_ad}, \eqref{cd_ad}, \eqref{dd_ad} and the corresponding estimates for $\ad$ in \eqref{adrate} and \eqref{adla} to obtain with probability at least $1-\delta$,
\bes
	\left\|\beta_{t,D}-\beta_t\right\|_K	\leq%\left\{
	\begin{dcases}
			\sqrt{c}\f{1}{\delta}\left(\log \f{16}{\delta}\right)^3|D|^{\f{-\ttt-\f{1}{2}}{(2\ttt+\aaa+1)}}\times|D|^{\f{\f{1}{2}}{2\ttt+\aaa+1}},\quad if \ \eqref{2momc}\ holds,\\
		\sqrt{c}\left(\log \f{16}{\delta}\right)^3|D|^{\f{-\ttt-\f{1}{2}}{(2\ttt+\aaa+1)}}\times|D|^{\f{\f{1}{2}}{2\ttt+\aaa+1}},\quad if \ \eqref{pmomc}\ holds,\\
	\end{dcases}
	%\right.
	\ees
	where $c=32\ww{C}_{7}(2\kkk^2\nu+2\kkk\nu+1)^6\left(\f{(M+5\sigma)}{\kkk}+1+8\|g^*\|_{L^2(\Omega)}\right)^2$ which is half of the first term of $C_2^*$. Recall that Theorem \ref{datafreees} and the basic fact $\f{s}{\lfloor s\rfloor}\leq2$ for any $s>1$ give that
	\bes
	\|\beta_t-\beta^*\|_K\leq C_{\theta,\gamma}\left\|g^*\right\|_{L^2(\Omega)} t^{-\theta(1-\mu)}\leq2^{\ttt(1-\mu)}C_{\theta,\gamma}\left\|g^*\right\|_{L^2(\Omega)}|D|^{-\f{\ttt}{2\ttt+\aaa+1}}, \ \ttt>0.
	\ees
	The triangle inequality finally implies
	\bes
	\left\|\beta_{t,D}-\beta^*\right\|_K	\leq%\left\{
	\begin{dcases}
		C_3^*\f{1}{\delta}\left(\log \f{16}{\delta}\right)^3|D|^{\f{-\ttt}{2\ttt+\aaa+1}},\quad if \ \eqref{2momc}\ holds,\\
		C_3^*\left(\log \f{16}{\delta}\right)^3|D|^{\f{-\ttt}{2\ttt+\aaa+1}},\quad if \ \eqref{pmomc}\ holds,\\
	\end{dcases}
	%\right.
	\ees
	where $C_3^*=\sqrt{c}+2^{\ttt(1-\mu)}C_{\theta,\gamma}\left\|g^*\right\|_{L^2(\Omega)}$. This proves Theorem \ref{gdfl_thm2}.
\end{proof}

\begin{comment}
Take stepsize 
\bes
&&\gamma_k=\f{\gamma}{(k+1)^{\mu}}, \gamma\leq \f{1}{B_C^2B_K^2},\\
&&\la=\f{1}{t^{1-\mu}},\\
&&t=|D|^{\f{1}{(2\theta+\aaa+1)(1-\mu)}}
\ees

Then we can derive with probability at least $1-\delta$,
\bea
\huaE(\eta_{t,D})-\huaE(\eta^*)\leq C\f{1}{\delta^2}\left(\log\f{16}{\delta}\right)^8|D|^{-\f{2\theta+1}{2\theta+1+\aaa}}
\eea
\end{comment}

\section{Analysis of DGDFL algorithm}

\subsection{Iterative procedures and error decompostion for DGDFL}
This section aims  to provide corresponding representations and error decompositions related to the distributed estimator $\overline{\beta_{t,D}}$. We start with  another representation of  $L_K^{-1/2}\beta_{k+1,D}$, that is,
\bes
\begin{aligned}
	L_K^{-1/2}\beta_{k+1,D}=&\left(I-\gamma_kT_{\wh{C}_D,K}\right)L_K^{-1/2}\beta_{k,D}+\f{\gamma_k}{|D|}\sum_{i=1}^{|D|}Y_iL_K^{1/2}X_i\\
	=&\left(I-\gamma_kT_{C,K}\right)L_K^{-1/2}\beta_{k,D}+\gamma_k\left(T_{C,K}-T_{\wh{C}_D,K}\right)L_K^{-1/2}\beta_{k,D}+\f{\gamma_k}{|D|}\sum_{i=1}^{|D|}Y_iL_K^{1/2}X_i.
\end{aligned}
\ees
Then an iteration implies that
\bes
L_K^{-1/2}\beta_{t,D}=\sum_{s=0}^{t-1}\gamma_s\pi_{s+1}^{t-1}\left(T_{C,K}\right)\left[\left(T_{C,K}-T_{\wh{C}_D,K}\right)L_K^{-1/2}\beta_{s,D}+\f{1}{|D|}\sum_{i=1}^{|D|}Y_iL_K^{1/2}X_i\right].
\ees
Recalling the representation \eqref{betakrep1} of data-free GDFL algorithm, we know
\bes
\begin{aligned}
	L_K^{-1/2}\left(\beta_{t,D}-\beta_t\right)=& \sum_{s=0}^{t-1}\gamma_s\pi_{s+1}^{t-1}\left(T_{C,K}\right)\left(T_{C,K}-T_{\wh{C}_D,K}\right)L_K^{-1/2}\beta_{s,D}\\
	&+\sum_{s=0}^{t-1}\gamma_s\pi_{s+1}^{t-1}\left(T_{C,K}\right)\left(\f{1}{|D|}\sum_{i=1}^{|D|}Y_iL_K^{1/2}X_i-T_{C,K}L_K^{-1/2}\beta^*\right).
\end{aligned}
\ees
Applying the above equality to the data set $D_j$, $j=1,2,...,m$ with $D$ replaced by $D_j$, we have
\bea
\begin{aligned}
	T_{C,K}^{1/2}L_K^{-1/2}\left(\overline{\beta_{t,D}}-\beta_t\right)=& 	T_{C,K}^{1/2}\sum_{s=0}^{t-1}\gamma_s\pi_{s+1}^{t-1}\left(T_{C,K}\right)\Bigg[\sum_{j=1}^m\f{|D_j|}{|D|}\Big(\f{1}{|D_j|}\sum_{(X,Y)\in D_j}YL_K^{1/2}X-T_{C,K}L_K^{-1/2}\beta^*\Big)\\
	&+\sum_{j=1}^m\f{|D_j|}{|D|}\left(T_{C,K}-T_{\wh{C}_{D_j},K}\right)L_K^{-1/2}\beta_{s,D_j}\Bigg]. 
\end{aligned} \label{dis_rep}
\eea
Since $\sum_{j=1}^m\f{|D_j|}{|D|}=1$, 
\bes
\sum_{j=1}^m\f{|D_j|}{|D|}\left(\f{1}{|D_j|}\sum_{(X,Y\in D_j)}YL_K^{1/2}X\right)=\f{1}{|D|}\sum_{(X,Y)\in D}YL_K^{1/2}X.
\ees
Now we are ready to give the following general error bound estimate of the distributed estimator $\overline{\beta_{t,D}}$ in the DGDFL algorithm.
\begin{pro}\label{distbeta_est}
 Assume  conditions \eqref{lk1/2} and \eqref{regularity} hold.	If $\la=1/t^{1-\mu}$, let the estimator $\overline{\beta_{t,D}}$ be generated from DGDFL algorithm, there holds
	\bes
	\left\|T_{C,K}^{1/2}L_K^{-1/2}\left(\overline{\beta_{t,D}}-\beta_t\right)\right\|_{L^2(\Omega)}&\leq& \ww{C}_{10}\Bigg[(\log t)^{H(\ttt)+1}\f{1}{\sqrt{\la}}\sup_{1\leq j \leq m}\bdj\left(\cdj^2+\cdj\ddj\right)\\
	&&+(\log t)^{H(\ttt)}(\cd+\dd)\Bigg], \ \ \ttt\geq0.
	\ees
\end{pro}

\begin{proof}

After taking norms on both sides of \eqref{dis_rep}, we have
\bes
\begin{aligned}
		\left\|T_{C,K}^{1/2}L_K^{-1/2}\left(\overline{\beta_{t,D}}-\beta_t\right)\right\|_{L^2(\Omega)}\leq& \left\|T_{C,K}^{1/2}\sum_{s=0}^{t-1}\gamma_s\pi_{s+1}^{t-1}\left(T_{C,K}\right)\left(\f{1}{|D|}\sum_{i=1}^{|D|}Y_iL_K^{1/2}X_i-T_{C,K}L_K^{-1/2}\beta^*\right)\right\|_{L^2(\Omega)}\\
		&+\left\|T_{C,K}^{1/2}\sum_{s=0}^{t-1}\gamma_s\pi_{s+1}^{t-1}\left(T_{C,K}\right)\sum_{j=1}^m\f{|D_j|}{|D|}\left(T_{C,K}-T_{\wh{C}_{D_j},K}\right)L_K^{-1/2}\beta_{s,D_j}\right\|_{L^2(\Omega)}\\
		:=&\bar{S}_1+\bar{S}_2.
\end{aligned}
\ees
Using Lemma \ref{sumlalem}, we can estimate $\bar{S}_1$ as 
\bes
\begin{aligned}
\bar{S}_1\leq&\left\|\left(\la I+T_{C,K}\right)^{1/2}\sum_{s=0}^{t-1}\gamma_s\pi_{s+1}^{t-1}\left(T_{C,K}\right)\left(\f{1}{|D|}\sum_{i=1}^{|D|}Y_iL_K^{1/2}X_i-T_{C,K}L_K^{-1/2}\beta^*\right)\right\|_{L^2(\Omega)}\\
\leq&\left\|\sum_{s=0}^{t-1}\gamma_s(\la I+T_{C,K})\pi_{s+1}^{t-1}(T_{C,K})\right\|\left\|(\la I+T_{C,K})^{-1/2}\left(\f{1}{|D|}\sum_{i=1}^{|D|}Y_iL_K^{1/2}X_i-T_{C,K}L_K^{-1/2}\beta^*\right)\right\|_{L^2(\Omega)}\\
\leq&D_{\mu,\gamma}\left(1+\la t^{1-\mu}\right)\dd.
\end{aligned}
\ees
For $\bar{S}_2$, we split it into three terms as follows,
\bes
\begin{aligned}
	\bar{S}_2\leq&\left\|T_{C,K}^{1/2}\sum_{s=0}^{t-1}\gamma_s\pi_{s+1}^{t-1}\left(T_{C,K}\right)\sum_{j=1}^m\f{|D_j|}{|D|}\left(T_{C,K}-T_{\wh{C}_{D_j},K}\right)L_K^{-1/2}(\beta_{s,D_j}-\beta_s)\right\|_{L^2(\Omega)}\\
	&+\left\|T_{C,K}^{1/2}\sum_{s=0}^{t-1}\gamma_s\pi_{s+1}^{t-1}\left(T_{C,K}\right)\sum_{j=1}^m\f{|D_j|}{|D|}\left(T_{C,K}-T_{\wh{C}_{D_j},K}\right)L_K^{-1/2}(\beta_s-\beta^*)\right\|_{L^2(\Omega)}\\
	&+\left\|T_{C,K}^{1/2}\sum_{s=0}^{t-1}\gamma_s\pi_{s+1}^{t-1}\left(T_{C,K}\right)\sum_{j=1}^m\f{|D_j|}{|D|}\left(T_{C,K}-T_{\wh{C}_{D_j},K}\right)L_K^{-1/2}\beta^*\right\|_{L^2(\Omega)}\\
	:=&\bar{S}_{2,1}+\bar{S}_{2,2}+\bar{S}_{2,3}.
\end{aligned}
\ees
Since $T_{\wh{C}_{D_j},K}=L_K^{1/2}L_{\wh{C}_{D_j}}L_K^{1/2}=\f{1}{|D_j|}\sum_{X\in D_j(X)}\left\nn L_K^{1/2}X,\cdot\right\mm L_K^{1/2}X$,  it is easy to see
\bes
\sum_{j=1}^m\f{|D_j|}{|D|}T_{\wh{C}_{D_j},K}=\sum_{j-1}^m\f{|D_j|}{|D|}\f{1}{|D_j|}\sum_{X\in D_j(X)}\left\nn L_K^{1/2}X,\cdot\right\mm L_K^{1/2}X=\f{1}{|D|}\sum_{i=1}^{|D|}\left\nn L_K^{1/2}X_i,\cdot\right\mm L_K^{1/2}X_i=T_{\wh{C}_{D},K},
\ees
where $D_j(X)=\{X:(X,Y)\in D_j \ \text{for} \  \text{some} \ Y\}$. Then we know that
\bes
\sum_{j=1}^m\f{|D_j|}{|D|}(\la I+T_{C,K})^{-1/2}\left(T_{C,K}-T_{\wh{C}_{D_j},K}\right)=(\la I+T_{C,K})^{-1/2}\left(T_{C,K}-T_{\wh{C}_{D},K}\right).
\ees
Then following the same procedure as that in getting \eqref{gammatk}, we have
\bes
\begin{aligned}
\bar{S}_{2,2}\leq&\sum_{s=0}^{t-1}\gamma_s\left\|\left(\la I+T_{C,K}\right)\pi_{s+1}^{t-1}(T_{C,K})\right\|\left\|\sum_{j=1}^m\f{|D_j|}{|D|}(\la I+T_{C,K})^{-1/2}\left(T_{C,K}-T_{\wh{C}_{D_j},K}\right)\right\|\\
&\times\left\|L_K^{-1/2}(\beta_s-\beta^*)\right\|_{L^2(\Omega)}\\
\leq&\cd\sum_{s=0}^{t-1}\gamma_s\left\|\left(\la I+T_{C,K}\right)\pi_{s+1}^{t-1}(T_{C,K})\right\|\left\|L_K^{-1/2}(\beta_s-\beta^*)\right\|_{L^2(\Omega)}\leq\ww{C}_{1}\cd\left((\log t)^{H(\ttt)}+\la t^{1-\mu}\right)
\end{aligned}
\ees
with $\ww{C}_1$ given as before. For $\bar{S}_{2,3}$, using Lemma \ref{sumlalem}, we have
\bes
\bar{S}_{2,3}&\leq& \left\|\sum_{s=0}^{t-1}\gamma_s(\la I+T_{C,K})\pi_{s+1}^{t-1}(T_{C,K})\right\|\left\|\sum_{j=1}^m\f{|D_j|}{|D|}(\la I+T_{C,K})^{-1/2}\left(T_{C,K}-T_{\wh{C}_{D_j},K}\right)\right\|\|T_{C,K}^\ttt g^*\|_{L^2(\Omega)}\\
&\leq&\ww{C}_{6}\left(1+\la t^{1-\mu}\right)\cd
\ees
with $\ww{C}_6$ defined as before. Now we estimate $\bar S_{2,1}$ as
\bes
\begin{aligned}
\bar{S}_{2,1}\leq&\sum_{s=0}^{t-1}\gamma_s\left\|\left(\la I+T_{C,K}\right)\pi_{s+1}^{t-1}\left(T_{C,K}\right)\right\|\left\|\sum_{j=1}^m\f{|D_j|}{|D|}\left(\la I+T_{C,K}\right)^{-1/2}\left(T_{C,K}-T_{\wh{C}_{D_j},K}\right)L_K^{-1/2}(\beta_{s,D_j}-\beta_s)\right\|_{L^2(\Omega)}\\
\leq&\sum_{s=0}^{t-1}\gamma_s\left\|\left(\la I+T_{C,K}\right)\pi_{s+1}^{t-1}\left(T_{C,K}\right)\right\|\sup_{1\leq j\leq m}\left\|\left(\la I+T_{C,K}\right)^{-1/2}\left(T_{C,K}-T_{\wh{C}_{D_j},K}\right)\right\|\left\|L_K^{-1/2}(\beta_{s,D_j}-\beta_s)\right\|_{L^2(\Omega)}\\
\leq&\sum_{s=0}^{t-1}\gamma_s\left\|\left(\la I+T_{C,K}\right)\pi_{s+1}^{t-1}\left(T_{C,K}\right)\right\|\sup_{1\leq j\leq m}\cdj\left\|L_K^{-1/2}(\beta_{s,D_j}-\beta_s)\right\|_{L^2(\Omega)}.
\end{aligned}
\ees
Applying Proposition \ref{knorm} to data set $D_j$, $j=1,2,...,m$, we have, for $s=1,2,...,t-1$ and  $\la=t^{1-\mu}$,
\bes
\llll\|L_K^{-1/2}\llll(\beta_{s,D_j}-\beta_s\rrrr)\rrrr\|_{L^2(\Omega)}&\leq&\ww{C}_{8}\f{1}{\sqrt{\la}}\left((\log t)^{H(\ttt)}+\la t^{1-\mu}\right)\bdj(\cdj+\ddj)\\
&\leq&2\ww{C}_{8}\f{1}{\sqrt{\la}}(\log t)^{H(\ttt)}\bdj(\cdj+\ddj).
\ees
Following the same procedure in getting \eqref{logbdd} and \eqref{gammmasum} with $T_{\wh{C}_D,K}$ replaced by $T_{C,K}$, we know, when $\la=1/t^{1-\mu}$,
\bes
\sum_{s=0}^{t-1}\gamma_s\left\|\left(\la I+T_{C,K}\right)\pi_{s+1}^{t-1}\left(T_{C,K}\right)\right\|\leq \ww{C}_{9}\log t,
\ees
where $\ww{C}_{9}=2[\gamma(B_C^2B_K^2)+15C_{1,\gamma}2^{\ttt(1-\mu)}+\gamma(1+\f{1}{1-\mu})]$. Then we arrive at 
\bes
\bar{S}_{2,1}\leq2\ww{C}_{8}\ww{C}_{9}(\log t)^{H(\ttt)+1}\f{1}{\sqrt{\la}}\sup_{1\leq j \leq m}\bdj\left(\cdj^2+\cdj\ddj\right).
\ees
Finally, combining the above estimates for $\bar{S}_{1}$, $\bar{S}_{2,1}$, $\bar{S}_{2,2}$, $\bar{S}_{2,3}$, we obtain 
\bes
\begin{aligned}
	\left\|T_{C,K}^{1/2}L_K^{-1/2}\left(\overline{\beta_{t,D}}-\beta_t\right)\right\|_{L^2(\Omega)}\leq& 2\ww{C}_{8}\ww{C}_{9}(\log t)^{H(\ttt)+1}\f{1}{\sqrt{\la}}\sup_{1\leq j \leq m}\bdj\left(\cdj^2+\cdj\ddj\right)\\
	&+\left[D_{\mu,\gamma}\dd+(\ww{C}_8)\cd\right]((\log t)^{H(\ttt)}+\la t^{1-\mu})\\
	\leq& \ww{C}_{10}\Bigg[(\log t)^{H(\ttt)+1}\f{1}{\sqrt{\la}}\sup_{1\leq j \leq m}\bdj\left(\cdj^2+\cdj\ddj\right)\\
	& +(\log t)^{H(\ttt)}(\cd+\dd)\Bigg]
\end{aligned}
\ees
where $\ww{C}_{10}=\max\{2\ww{C}_8\ww{C}_9,2D_{\mu,\gamma},2\ww{C}_8\}$.
\end{proof}

\subsection{Convergence analysis: proofs of Theorem \ref{dgdthm1}, Theorem \ref{dgdthm2}} \label{thm3thm4proof}
This subsection aims to provide proofs of Theorems \ref{dgdthm1} and \ref{dgdthm2}. When $\ttt=0$ and   $m=1$, Theorem \ref{gdfl_thm1} directly implies the desired results in Theorem \ref{dgdthm1} and Theorem \ref{dgdthm2}. In this subsection, we focus on the case $\ttt>0$.
According to \eqref{bdest}, \eqref{cdest}, Proposition \ref{ddest}, with probability at least $1-4\delta$, the following bounds hold simultaneously,
\bes
&&\bdj\cdj^2 \leq 4\left(\log\f{8}{\delta}\right)^4\left(\f{\adj}{\sqrt{\la}}+1\right)^2\adj^2\\
&&\bdj\cdj\ddj	\leq%\left\{
\begin{dcases}
\ww{C}_{11}\f{1}{\delta}\left(\log\f{8}{\delta}\right)^4\left(\f{\adj}{\sqrt{\la}}+1\right)^2\adj^2 ,\quad if \  \ \eqref{2momc} \ holds,\\
	\ww{C}_{11}\left(\log\f{8}{\delta}\right)^4\left(\f{\adj}{\sqrt{\la}}+1\right)^2\adj^2 ,\quad if \  \ \eqref{pmomc} \ holds,\\
\end{dcases}
%\right. 
\ees
where $\ww{C}_{11}=4\left(\f{(M+\sigma)}{\kkk}+2\|g^*\|_{L^2(\Omega)}\right)$. Then we know with probability at least $1-4m\delta$,
\bea
\nono&&\sup_{1\leq j \leq m}\bdj\left(\cdj^2+\cdj\ddj\right)\\
&&\leq%\left\{
\begin{dcases}
	\ww{C}_{12}\f{1}{\delta}\left(\log\f{8}{\delta}\right)^4\sup_{1\leq j\leq m}\left(\f{\adj}{\sqrt{\la}}+1\right)^2\adj^2 ,\quad if \  \ \eqref{2momc} \ holds,\\
	\ww{C}_{12}\left(\log\f{8}{\delta}\right)^4\sup_{1\leq j\leq m}\left(\f{\adj}{\sqrt{\la}}+1\right)^2\adj^2 ,\quad if \  \ \eqref{pmomc} \ holds, \\
\end{dcases}
 \label{supbdd1.7}
\eea
where $$\ww{C}_{12}=\ww{C}_{11}+4.$$ Note that 
\bea
\adj\leq\max\{2\kkk^2\nu,2\kkk\nu\}\sqrt{\f{m\la^{-\aaa}}{|D|}}\left(\sqrt{\f{m\la^{\aaa-1}}{|D|}}+1\right).
\eea
If the noise condition \eqref{2momc} is satisfied, then, when $\la=1/t^{1-\mu}$, $t=\left\lfloor|D|^{\f{1}{(2\theta+\aaa+1)(1-\mu)}}\right\rfloor$, and the total number m of the local processors satisfy 
\bes
m\leq\f{\left\lfloor|D|^{\f{1}{2\ttt+\aaa+1}}\right\rfloor^{\f{\ttt}{2}}}{(\log|D|)^{\f{5}{2}}},
\ees
\begin{comment}
\bes
m	\left\{
\begin{aligned}
	=1,\quad \ttt=0,\\
	\leq\f{\left\lfloor|D|^{\f{1}{2\ttt+\aaa+1}}\right\rfloor^{\f{\ttt}{2}}}{(\log|D|)^{\f{5}{2}}},\quad \ttt>0,\\
\end{aligned}
\right. 
\ees
\end{comment}
we have
\bea
\sqrt{\f{m\la^{\aaa-1}}{|D|}}\leq\la^{-\f{\ttt}{4}}\times\la^{\f{\aaa-1}{2}}\times\la^{\ttt+\f{\aaa+1}{2}}=\la^{\aaa+\f{3\ttt}{4}}\leq1  \label{adeq1}
\eea
\begin{comment}
\bea
\sqrt{\f{m\la^{\aaa-1}}{|D|}}	\leq\left\{
\begin{aligned}
\la^{\f{\aaa-1}{2}}\times\la^{\f{\aaa+1}{2}}=\la^{\aaa}\leq1,\quad \ttt=0,\\
	\la^{-\f{\ttt}{4}}\times\la^{\f{\aaa-1}{2}}\times\la^{\ttt+\f{\aaa+1}{2}}=\la^{\aaa+\f{3\ttt}{4}}\leq1,\quad \ttt>0,\\
\end{aligned}
\right.
\eea
\end{comment}
and  
\bea
\adj\leq 2\max\{2\kkk^2\nu,2\kkk\nu\}\sqrt{\f{m\la^{-\aaa}}{|D|}}.   \label{adeq2}
\eea
Meanwhile, we know that 
\bes
&&\f{\adj}{\sqrt{\la}}\leq2\max\{2\kkk^2\mu,2\kkk\nu\}\sqrt{\f{m\la^{-\aaa-1}}{|D|}}\leq2\max\{2\kkk^2\nu,2\kkk\nu\}.
\ees
After scaling on $\delta$, we have with probability at least $1-\f{\delta}{2}$, there holds,
\bes
&&(\log t)^{H(\ttt)+1}\f{1}{\sqrt{\la}}\sup_{1\leq j \leq m}\bdj\left(\cdj^2+\cdj\ddj\right)\\
&&\leq8\ww{C}_{12}\f{m}{\delta}\left(\log \f{64m}{\delta}\right)^4(\log t)\left(\f{\adj}{\sqrt{\la}}+1\right)^2\f{\adj^2}{\sqrt{\la}}\\
&&\leq \ww{C}_{13}\f{m}{\delta}\left(\log |D|\right)^4\left(\log\f{64}{\delta}\right)^4(\log |D|)\f{m\la^{-\aaa}}{|D|}\times\f{1}{\sqrt{\la}}\\
&&\leq \ww{C}_{13}\f{1}{\delta}\left(\log\f{64}{\delta}\right)^4\left(\log |D|\right)^{5}m^2\times\la^{-\aaa}\times\la^{2\ttt+\aaa+1}\times\la^{-1/2}\\
&&\leq 	\ww{C}_{13}\f{1}{\delta}\left(\log\f{64}{\delta}\right)^4|D|^{-\f{\ttt+\f{1}{2}}{2\ttt+\aaa+1}}
\ees
\begin{comment}
\left\{
\begin{aligned}
\ww{C}_{13}\f{1}{\delta}\left(\log\f{64}{\delta}\right)^4|D|^{-\f{\f{1}{2}}{\aaa+1}}\left(\log|D|\right)^6,\quad \ttt=0,\\
\ww{C}_{13}\f{1}{\delta}\left(\log\f{64}{\delta}\right)^4|D|^{-\f{\ttt+\f{1}{2}}{2\ttt+\aaa+1}},\quad \ttt>0.\\
\end{aligned}
\right. 
\end{comment}
where the constant $\ww{C}_{13}=8\ww{C}_{12}(2\max\{2\kkk^2\nu,2\kkk\nu\}+1)^4\f{1}{(2\ttt+\aaa+1)(1-\mu)}$.
From inequality \eqref{cdest} and Proposition \ref{ddest}, with probability at least $1-\f{\delta}{2}$, the following inequalities hold simultaneously, 
\bea
&&\cd\leq\left(\log\f{8}{\delta}\right)\ad\leq (2\kkk^2\nu+2\kkk\nu)\left(\log\f{8}{\delta}\right)|D|^{-\f{\ttt+\f{1}{2}}{2\ttt+\aaa+1}},  \label{cdbdd1.7}\\
&&\dd\leq\ww{C}_{11}\f{4}{\delta}\left(\log\f{64}{\delta}\right)\ad\leq \ww{C}_{11}(2\kkk^2\nu+2\kkk\nu)\f{4}{\delta}\left(\log\f{64}{\delta}\right)|D|^{-\f{\ttt+\f{1}{2}}{2\ttt+\aaa+1}}.  \label{ddbdd1.7}
\eea
Then we know, with probability at least $1-\delta$,
\bes
\left\|T_{C,K}^{1/2}L_K^{-1/2}\left(\overline{\beta_{t,D}}-\beta_t\right)\right\|_{L^2(\Omega)}\leq\ww{C}_{14}\f{1}{\delta}\left(\log\f{64}{\delta}\right)^4|D|^{-\f{\ttt+\f{1}{2}}{2\ttt+\aaa+1}}, 
\ees
\begin{comment}
\left\{
\begin{aligned}
\ww{C}_{14}\f{1}{\delta}\left(\log\f{64}{\delta}\right)^4|D|^{-\f{\f{1}{2}}{\aaa+1}}\left(\log|D|\right)^6,\quad \ttt=0,\\
\ww{C}_{14}\f{1}{\delta}\left(\log\f{64}{\delta}\right)^4|D|^{-\f{\ttt+\f{1}{2}}{2\ttt+\aaa+1}},\quad \ttt>0.\\
\end{aligned}
\right. 
\end{comment}
where $\ww{C}_{14}=\ww{C}_{10}\ww{C}_{13}+\ww{C}_{10}(4\ww{C}_{11}+1)(2\kkk^2\nu+2\kkk\nu)$.

If the noise condition \eqref{pmomc} holds, then when $\la=1/t^{1-\mu}$, $t=\left\lfloor|D|^{\f{1}{(2\theta+\aaa+1)(1-\mu)}}\right\rfloor$, and the total number $m$ of the local processors satisfy 
\bes
m\leq\f{\left\lfloor|D|^{\f{1}{2\ttt+\aaa+1}}\right\rfloor^{\ttt}}{(\log|D|)^{5}}.
\ees
Following similar computations to \eqref{adeq1} and \eqref{adeq2}, we have
\bes
\sqrt{\f{m\la^{\aaa-1}}{|D|}}=\la^{-\f{\ttt}{2}}\times\la^{\f{\aaa-1}{2}}\times\la^{\ttt+\f{\aaa+1}{2}}=\la^{\aaa+\f{\ttt}{2}}\leq1
\ees
\begin{comment}
\bes
\sqrt{\f{m\la^{\aaa-1}}{|D|}}=	\left\{
\begin{aligned}
	\la^{\f{\aaa-1}{2}}\times\la^{\f{\aaa+1}{2}}=\la^{\aaa}\leq1,\quad \ttt=0,\\
	\la^{-\f{\ttt}{2}}\times\la^{\f{\aaa-1}{2}}\times\la^{\ttt+\f{\aaa+1}{2}}=\la^{\aaa+\f{\ttt}{2}}\leq1,\quad \ttt>0,\\
\end{aligned}
\right. 
\ees
\end{comment}
and  
$\adj\leq 2\max\{2\kkk^2\nu,2\kkk\nu\}\sqrt{\f{m\la^{-\aaa}}{|D|}}$.
Then we also obtain
$\f{\adj}{\sqrt{\la}}\leq2\max\{2\kkk^2\mu,2\kkk\nu\}\sqrt{\f{m\la^{-\aaa-1}}{|D|}}\leq2\max\{2\kkk^2\nu,2\kkk\nu\}$. Accordingly,
after scaling on $\delta$ and using the condition $m\leq\lfloor|D|^{\f{1}{2\ttt+\aaa+1}}\rfloor^{\ttt}/(\log|D|)^{5}$, we have with probability at least $1-\f{\delta}{2}$, there holds,
\bes
&&(\log t)^{H(\ttt)+1}\f{1}{\sqrt{\la}}\sup_{1\leq j \leq m}\bdj\left(\cdj^2+\cdj\ddj\right)\\
&&\leq\ww{C}_{12}\left(\log \f{64m}{\delta}\right)^4(\log t)\left(\f{\adj}{\sqrt{\la}}+1\right)^2\f{\adj^2}{\sqrt{\la}}\\
&&\leq \ww{C}_{13}\left(\log |D|\right)^4\left(\log\f{64}{\delta}\right)^4(\log |D|)\f{m\la^{-\aaa}}{|D|}\times\f{1}{\sqrt{\la}}\\
&&\leq \ww{C}_{13}\left(\log\f{64}{\delta}\right)^4\left(\log |D|\right)^{5}m\times\la^{-\aaa}\times\la^{2\ttt+\aaa+1}\times\la^{-1/2}\\
&&\leq \ww{C}_{13}\left(\log\f{64}{\delta}\right)^4|D|^{-\f{\ttt+\f{1}{2}}{2\ttt+\aaa+1}}.
\ees
\begin{comment}
\left\{
\begin{aligned}
\ww{C}_{13}\left(\log\f{64}{\delta}\right)^4|D|^{-\f{\f{1}{2}}{\aaa+1}}\left(\log|D|\right)^6,\quad \ttt=0,\\
\ww{C}_{13}\left(\log\f{64}{\delta}\right)^4|D|^{-\f{\ttt+\f{1}{2}}{2\ttt+\aaa+1}},\quad \ttt>0.\\
\end{aligned}
\right.
\end{comment}
From inequality \eqref{cdest} and Proposition \ref{ddest}, we know with probability at least $1-\f{\delta}{2}$, the following holds simultaneously
\bes
&&\cd\leq\left(\log\f{8}{\delta}\right)\ad\leq (2\kkk^2\nu+2\kkk\nu)\left(\log\f{8}{\delta}\right)|D|^{-\f{\ttt+\f{1}{2}}{2\ttt+\aaa+1}}\\
&&\dd\leq\ww{C}_{11}\left(\log\f{64}{\delta}\right)\ad\leq \ww{C}_{11}(2\kkk^2\nu+2\kkk\nu)\left(\log\f{64}{\delta}\right)|D|^{-\f{\ttt+\f{1}{2}}{2\ttt+\aaa+1}}.
\ees
Then we conclude that, with probability at least $1-\delta$
\bes
\left\|T_{C,K}^{1/2}L_K^{-1/2}\left(\overline{\beta_{t,D}}-\beta_t\right)\right\|_{L^2(\Omega)}\leq	\ww{C}_{14}\left(\log\f{64}{\delta}\right)^4|D|^{-\f{\ttt+\f{1}{2}}{2\ttt+\aaa+1}} 
\ees
\begin{comment}
\left\{
\begin{aligned}
\ww{C}_{14}\left(\log\f{64}{\delta}\right)^4|D|^{-\f{\f{1}{2}}{\aaa+1}}\left(\log|D|\right)^6,\quad \ttt=0,\\
\ww{C}_{14}\left(\log\f{64}{\delta}\right)^4|D|^{-\f{\ttt+\f{1}{2}}{2\ttt+\aaa+1}},\quad \ttt>0.\\
\end{aligned}
\right. 
\end{comment}
with $\ww{C}_{14}$ defined as before.

Finally, when the noise condition \eqref{2momc} holds and and the total number $m$ of the local processors satisfy 
\eqref{disfl_m1},
we have
\bes
\begin{aligned}
	\huaE(\overline{\eta_{t,D}})-\huaE(\eta^*)=&\left\|T_{C,K}^{1/2}L_K^{-1/2}(\overline{\beta_{t,D}}-\beta^*)\right\|_{L^2(\Omega)}^2\\
	\leq&2\left\|T_{C,K}^{1/2}L_K^{-1/2}(\overline{\beta_{t,D}}-\beta_t)\right\|_{L^2(\Omega)}^2+2\left\|T_{C,K}^{1/2}L_K^{-1/2}(\beta_{t}-\beta^*)\right\|_{L^2(\Omega)}^2\\
	\leq&2\ww{C}_{14}^2\f{1}{\delta^2}\left(\log\f{64}{\delta}\right)^8|D|^{-\f{2\ttt+1}{2\ttt+\aaa+1}}+2C_{\theta+\f{1}{2},\gamma}^2\left\|g^*\right\|_{L^2(\Omega)}^2 t^{-(2\theta+1)(1-\mu)}\\
	\leq&	C_4^*\f{1}{\delta^2}\left(\log\f{64}{\delta}\right)^8|D|^{-\f{2\ttt+1}{2\ttt+\aaa+1}},
\end{aligned}
\ees
\begin{comment}
\left\{
\begin{aligned}
C_4^*\f{1}{\delta^2}\left(\log\f{64}{\delta}\right)^8|D|^{-\f{2\ttt+1}{2\ttt+\aaa+1}}\left(\log|D|\right)^{12},\quad \ttt=0,\\
C_4^*\f{1}{\delta^2}\left(\log\f{64}{\delta}\right)^8|D|^{-\f{2\ttt+1}{2\ttt+\aaa+1}},\quad \ttt>0,\\
\end{aligned}
\right. 
\end{comment}
which concludes the proof of Theorem \ref{dgdthm1}. Correspondingly, when  the noise condition \eqref{pmomc} holds and the total number m of the local processors satisfy 
\eqref{disfl_m2}, we have
\bes
\begin{aligned}
	\huaE(\overline{\eta_{t,D}})-\huaE(\eta^*)=&\left\|T_{C,K}^{1/2}L_K^{-1/2}(\overline{\beta_{t,D}}-\beta^*)\right\|_{L^2(\Omega)}^2\\
	\leq&2\ww{C}_{14}^2\left(\log\f{64}{\delta}\right)^8|D|^{-\f{2\ttt+1}{2\ttt+\aaa+1}}+2C_{\theta+\f{1}{2},\gamma}^2\left\|g^*\right\|_{L^2(\Omega)}^2 t^{-(2\theta+1)(1-\mu)}\\
	\leq&C_4^*\left(\log\f{64}{\delta}\right)^8|D|^{-\f{2\ttt+1}{2\ttt+\aaa+1}},
\end{aligned}
\ees
\begin{comment}
\left\{
\begin{aligned}
C_4^*\left(\log\f{64}{\delta}\right)^8|D|^{-\f{2\ttt+1}{2\ttt+\aaa+1}}\left(\log|D|\right)^{12},\quad \ttt=0,\\
C_4^*\left(\log\f{64}{\delta}\right)^8|D|^{-\f{2\ttt+1}{2\ttt+\aaa+1}},\quad \ttt>0.\\
\end{aligned}
\right. 
\end{comment}
where $C_4^*=2\ww{C}_{14}^2+C_{\theta+\f{1}{2},\gamma}^22^{(2\ttt+1)(1-\mu)+1}\left\|g^*\right\|_{L^2(\Omega)}^2$. We conclude the proof of Theorem \ref{dgdthm2}.

\subsection{Proofs of Corollary \ref{disflcor1}, Corollary \ref{disflcor2}}
\begin{proof}[Proof of Corrolary \ref{disflcor1}]
	Set
		\bes
	\mbb \huaT(|D|)= \left\{
	\begin{aligned}
		C_1^*|D|^{-\f{1}{\aaa+1}}\left(\log|D|\right)^2,\quad \ttt=0,\\
		C_4^*|D|^{-\f{2\ttt+1}{2\ttt+\aaa+1}},\quad \ttt>0.\\
	\end{aligned}
	\right. 
	\ees
Then we know from Theorem  \ref{dgdthm2} that for any $1<\delta<1$, there holds,
$\text{Prob}\{\huaE(\overline{\eta_{t,D}})-\huaE(\eta^*)>\huaT(|D|)(\log\f{64}{\delta})^8\}\leq\delta$.
	If we set $s=\huaT(|D|)(\log\f{64}{\delta})^8>\huaT(|D|)(\log64)^8$, then  $\delta=64\exp\{-(\f{s}{\huaT(|D|)})^{1/8}\}$. It follows that when $s>\huaT(|D|)(\log64)^8$,
	\bes
	Prob\Big\{\huaE(\overline{\eta_{t,D}})-\huaE(\eta^*)>s\Big\}\leq64\exp\left\{-\left(\f{s}{\huaT(|D|)}\right)^{1/8}\right\}.
	\ees
When 	$s\leq\huaT(|D|)(\log64)^8$, the above inequality also holds since the right hand side of the above inequality is greater than 1.
Hence we have
\bes
&&\mbb E\Big[\huaE(\overline{\eta_{t,D}})-\huaE(\eta^*)\Big]=\int_0^{\infty}Prob\Big\{\huaE(\overline{\eta_{t,D}})-\huaE(\eta^*)>s\Big\}ds\leq \int_0^{\infty}64\exp\left\{-\left(\f{s}{\huaT(|D|)}\right)^{1/8}\right\}ds\\
&&=64\times8\huaT(|D|)\int_0^{\infty}e^{-u}u^{7}du=64\Gamma(9)\huaT(|D|)
\leq\left\{
\begin{aligned}
	C_5^*|D|^{-\f{1}{\aaa+1}}\left(\log|D|\right)^{2},\quad \ttt=0,\\
	C_5^*|D|^{-\f{2\ttt+1}{2\ttt+\aaa+1}},\quad \ttt>0,\\
\end{aligned}
\right. 
\ees
where $C_5^*=64\Gamma(9)\max\{C_1^*,C_4^*\}$.	
	\end{proof}

	To prove Corollary \ref{disflcor2}, we need the following Borel-Cantelli lemma from \cite{Dudley2018}.
	\begin{lem}\label{bclem}
		Let $\{\eta_N\}$ be a sequence of events in some probability space and $\{\tau_N\}$ be a sequence of positive numbers satisfying $\lim_{N\rightarrow\infty}\tau_N=0$. If
		\bes
		\sum_{N=1}^{\infty}Prob\left(|\eta_N-\eta|>\tau_N\right)<\infty,
		\ees
		then $\eta_N$ converges to $\eta$ almost surely.
	\end{lem}
\begin{proof}[Proof of Corollary \ref{disflcor2}]	Denote $|D|=N$ and set $\delta=\delta_N=N^{-2}$ in Theorem \ref{dgdthm2}. Then
\bes
Prob\left\{N^{\f{2\ttt+1}{2\ttt+\aaa+1}(1-\eee)}\Big[\huaE(\overline{\eta_{t,D}})-\huaE(\eta^*)\Big]>\huaR_{\eee}(N)\left(\log\f{64}{\delta_N}\right)^8\right\}\leq\delta_N,  \ttt\geq0,
\ees
where
\bes
\huaR_{\eee}(N)	=\left\{
\begin{aligned}
C_5^*	N^{-\f{\eee}{\aaa+1}}(\log N)^{2},\quad \ttt=0,\\
C_5^*	N^{-\f{(2\ttt+1)\eee}{2\ttt+\aaa+1}},\quad \ttt>0.\\
\end{aligned}
\right. 
\ees
If we denote $\tau_N=\huaR_{\eee}(N)\left(\log\f{64}{\delta_N}\right)^8$, then we know $\tau_N\rightarrow0$,
\bes
\sum_{N=2}^{\infty}Prob\left\{N^{\f{2\ttt+1}{2\ttt+\aaa+1}(1-\eee)}\Big[\huaE(\overline{\eta_{t,D}})-\huaE(\eta^*)\Big]>\tau_N\right\}\leq\sum_{N=2}^\infty\delta_N<\infty.
\ees
Then using Lemma \ref{bclem} yields our desired result.
\end{proof}

\section{Analysis of semi-supervised DGDFL algorithm}

Recall the representation of $\dd$ in \eqref{dddef}. For convenience, we denote $$\ww{f}_{K,D}=\f{1}{|D|}\sum_{(X,Y)\in D}YL_K^{1/2}X$$ and we have the representation $\dd=\|(\la I+T_{C,K})^{-1/2}(\ww{f}_{K,D}-T_{C,K}L_K^{-1/2}\beta^*)\|_{L^2(\Omega)}$. According to the definition of $D^*$, the condition $|D_1|=|D_2|=\cdots=|D_m|=|D|/m$ and $|D_1^*|=|D_2^*|=\cdots=|D_m^*|=|D^*|/m$, we have 
\bes
&&\ww{f}_{K,D^*}=\f{1}{|D^*|}\sum_{(X_i^*,Y_i^*)\in D^*}Y_i^*L_K^{1/2}X_i^*=\f{1}{|D^*|}\sum_{j=1}^m\sum_{(X_i^*,Y_i^*)\in D_j^*}Y_i^*L_K^{1/2}X_i^*\\
&&=\f{1}{|D^*|}\sum_{j=1}^m\sum_{(X_i,Y_i)\in D_j}\f{|D_j^*|}{|D_j|}Y_iL_K^{1/2}X_i=\f{1}{|D^*|}\sum_{(X_i,Y_i)\in D}\f{|D^*|}{|D|}Y_iL_K^{1/2}X_i=\ww{f}_{K,D}.
\ees 
Then we know 
\bea
\dd=\dds.
\eea
For the local data set $D_j^*$, we have
\bes
\ww{f}_{K,D_j^*}=\f{1}{|D_j^*|}\sum_{(X_i^*,Y_i^*)\in D_j^*}Y_i^*L_K^{1/2}X_i^*=\f{1}{|D_j^*|}\sum_{(X_i,Y_i)\in D_j}\f{|D_j^*|}{|D_j|}Y_iL_K^{1/2}X_i=\ww{f}_{K,D_j}
\ees
and then we know for the $j$-th local machine, there holds,
\bea
\ddj=\ddjs, \ j=1,2,...,m.
\eea
According to Proposition \ref{distbeta_est} with the data set $D$ replaced by $D^*$, we have
\bes
\left\|T_{C,K}^{1/2}L_K^{-1/2}\left(\overline{\beta_{t,D^*}}-\beta_t\right)\right\|_{L^2(\Omega)}&\leq& \ww{C}_{10}\Bigg[(\log t)^{H(\ttt)+1}\f{1}{\sqrt{\la}}\sup_{1\leq j \leq m}\bdjs\left(\cdjs^2+\cdjs\ddjs\right)\\
&&+(\log t)^{H(\ttt)}(\cds+\dds)\Bigg].
\ees
Using a similar argument as  \eqref{cd_ad}, \eqref{dd_ad}, \eqref{supbdd1.7} and the fact that $\adjs\leq\adj$
 we know with probability at least $1-\f{\delta}{2}$,
\bea
\nono&&\f{1}{\sqrt{\la}}(\log t)^{H(\ttt)+1}\sup_{1\leq j \leq m}\bdjs\left(\cdjs^2+\cdjs\ddjs\right)	\\
&&\leq%\left\{
\begin{dcases}
	\ww{C}_{12}(\log t)^{H(\ttt)+1}\f{8m}{\delta}\left(\log\f{64m}{\delta}\right)^4\eds ,\quad if \  \ \eqref{2momc} \ holds,\\
	\ww{C}_{12}(\log t)^{H(\ttt)+1}\left(\log\f{64m}{\delta}\right)^4\eds,\quad if \  \ \eqref{pmomc} \ holds, \\
\end{dcases}
 \label{semiedsbdd}
\eea
with
\bea
\eds=\sup_{1\leq j\leq m}\left(\f{\adjs}{\sqrt{\la}}+1\right)^2\f{\adjs}{\sqrt{\la}}\adj 
\eea
and with probability at least $1-\f{\delta}{2}$, the following two inequalities hold simultaneously,
\bea
&&\cds\leq\left(\log\f{8}{\delta}\right)\ads\leq\left(\log\f{8}{\delta}\right)\ad,   \label{semicd}\\
&&\dds=\dd\leq
\begin{dcases}
	\ww{C}_{11}\left(\f{4}{\delta}\log\f{64}{\delta}\right)\ad ,\quad if  \ \eqref{2momc} \ holds,\\
	\ww{C}_{11}\left(\log\f{64}{\delta}\right)\ad ,\quad if  \ \eqref{pmomc} \ holds.  \\
\end{dcases}
 \label{semidd}
\eea
After making these preparation, we are ready to give the proof of Theorem \ref{semithm1}.

\begin{proof}[Proof of Theorem \ref{semithm1}]
 When the noise condition \eqref{2momc} holds,	using the fact that $\la=1/t^{1-\mu}$, $t=\left\lfloor|D|^{\f{1}{(2\theta+\aaa+1)(1-\mu)}}\right\rfloor$, we have
\bes	
	\adjs&\leq&\max\{2\kkk^2\nu,2\kkk\nu\}\left(\f{1}{|D_j^*|\sqrt{\la}}+\f{\sqrt{\huaN(\la)}}{\sqrt{|D_j^*|}}\right)\\
	&\leq&\max\{2\kkk^2\nu,2\kkk\nu\}\left(m|D^*|^{-1}|D|^{\f{1}{2(2\ttt+\aaa+1)}}+\sqrt{m}|D^*|^{-\f{1}{2}}|D|^{\f{\aaa}{2(2\ttt+\aaa+1)}}\right).
\ees	
	According to condition \eqref{semi1m}, we know
	\bea
	m|D^*|^{-1}|D|^{\f{1}{2(2\ttt+\aaa+1)}}\leq \sqrt{m}|D^*|^{-\f{1}{2}}|D|^{\f{\aaa}{2(2\ttt+\aaa+1)}}
	\eea
and
\bea
\adjs\leq2\max\{2\kkk^2\nu,2\kkk\nu\}\sqrt{m}|D^*|^{-\f{1}{2}}|D|^{\f{\aaa}{2(2\ttt+\aaa+1)}},\label{adjs_est}
\eea	
which further imply that
\bea
\f{\adjs}{\sqrt{\la}}\leq2\max\{2\kkk^2\nu,2\kkk\nu\}\sqrt{m}|D^*|^{-\f{1}{2}}|D|^{\f{\aaa+1}{2(2\ttt+\aaa+1)}}.  \label{adjsla_est}
\eea	
It is also easy to see from  $m\leq |D^*|^{\f{1}{4}}|D|^{-\f{\aaa+1}{4(2\ttt+\aaa+1)}}$ in \eqref{semi1m} and the fact $|D|\leq|D^*|$, we know $\sqrt{m}|D^*|^{-\f{1}{2}}|D|^{\f{\aaa+1}{2(2\ttt+\aaa+1)}}\leq1$, and hence 
	\bea
	\left(\f{\adjs}{\sqrt{\la}}+1\right)^2\leq\left(2\max\{2\kkk^2\nu,2\kkk\nu\}+1 \right)^2 \label{adjslasq_est}.
	\eea
Also recall
\bea
\nono\adj&\leq&\max\{2\kkk^2\nu,2\kkk\nu\}\left(\f{1}{|D_j|\sqrt{\la}}+\f{\sqrt{\huaN(\la)}}{\sqrt{|D_j|}}\right)\\
&\leq&\max\{2\kkk^2\nu,2\kkk\nu\}\left(m|D|^{\f{-4\ttt-2\aaa-1}{2(2\ttt+\aaa+1)}}+\sqrt{m}|D|^{\f{-2\ttt-1}{2(2\ttt+\aaa+1)}}\right).
\eea	
	Then we have
	\bea
	\eds\leq \left(2\max\{2\kkk^2\nu,2\kkk\nu\}+1 \right)^3\left(m\sqrt{m}|D^*|^{-\f{1}{2}}|D|^{\f{-4\ttt-\aaa}{2(2\ttt+\aaa+1)}}+m|D^*|^{-\f{1}{2}}|D|^{\f{\aaa-2\ttt}{2(2\ttt+\aaa+1)}}\right) \label{ed_est}.
	\eea
Then we can return to inequality \eqref{semiedsbdd}. After using the size condition \eqref{semi1m} on $m$, we  get with probability at least $1-\f{\delta}{2}$,
\bes
&&\f{1}{\sqrt{\la}}(\log t)^{H(\ttt)+1}\sup_{1\leq j \leq m}\bdjs\left(\cdjs^2+\cdjs\ddjs\right)\\
&&\leq\ww{C}_{15}\f{1}{\delta}\left(\log \f{64}{\delta}\right)^4(\log |D|)^{H(\ttt)+5}\left(m^{\f{5}{2}}|D^*|^{-\f{1}{2}}|D|^{\f{-4\ttt-\aaa}{2(2\ttt+\aaa+1)}}+m^2|D^*|^{-\f{1}{2}}|D|^{\f{\aaa-2\ttt}{2(2\ttt+\aaa+1)}}\right)\\
&&\leq2\ww{C}_{15}\f{1}{\delta}\left(\log \f{64}{\delta}\right)^4|D|^{-\f{2\ttt+1}{2(2\ttt+\aaa+1)}}
\ees
	where $\ww{C}_{15}=8\ww{C}_{12}\left(2\max\{2\kkk^2\nu,2\kkk\nu\}+1 \right)^4(\f{1}{2\ttt+\aaa+1})^{H(\ttt)+1}+(\ww{C}_{11}+1)(\f{1}{(2\ttt+\aaa+1)(1-\mu)})^{H(\ttt)}$.
Finally, combining the above estimates with \eqref{semicd} and \eqref{semidd}, we have with probability at least $1-\delta$,
\bes
\begin{aligned}
	\huaE(\overline{\eta_{t,D^*}})-\huaE(\eta^*)=&\left\|T_{C,K}^{1/2}L_K^{-1/2}(\overline{\beta_{t,D^*}}-\beta^*)\right\|_{L^2(\Omega)}^2\\
	\leq&2\left\|T_{C,K}^{1/2}L_K^{-1/2}(\overline{\beta_{t,D^*}}-\beta_t)\right\|_{L^2(\Omega)}^2+2\left\|T_{C,K}^{1/2}L_K^{-1/2}(\beta_{t}-\beta^*)\right\|_{L^2(\Omega)}^2\\
	\leq&4\ww{C}_{15}^2\f{1}{\delta^2}\left(\log\f{64}{\delta}\right)^8|D|^{-\f{2\ttt+1}{2\ttt+\aaa+1}}(\log|D|)^{2H(\ttt)}+2C_{\theta+\f{1}{2},\gamma}^2\left\|g^*\right\|_{L^2(\Omega)}^2 t^{-(2\theta+1)(1-\mu)}\\
	\leq&C_6^*\f{1}{\delta^2}\left(\log\f{64}{\delta}\right)^8|D|^{-\f{2\ttt+1}{2\ttt+\aaa+1}}(\log|D|)^{2H(\ttt)}
\end{aligned}
\ees	
with $C_6^*=4\ww{C}_{15}^2+2^{(2\ttt+1)(1-\mu)+1}C_{\theta+\f{1}{2},\gamma}^2\left\|g^*\right\|_{L^2(\Omega)}^2$.	

We turn to handle the case of the noise condition \eqref{pmomc}. Following similar procedures with the above calculations of $\adj$ and $\adjs$,   we can derive \eqref{adjs_est}, \eqref{adjsla_est}, \eqref{adjslasq_est}, \eqref{ed_est} under our size condition \eqref{semi2m}. Return to inequality \eqref{semiedsbdd} and use the size condition \eqref{semi2m}. We  obtain that, with probability at least $1-\f{\delta}{2}$,
\bes
&&\f{1}{\sqrt{\la}}(\log t)^{H(\ttt)+1}\sup_{1\leq j \leq m}\bdjs\left(\cdjs^2+\cdjs\ddjs\right)\\
&&\leq\ww{C}_{15}\left(\log \f{64}{\delta}\right)^4(\log |D|)^{H(\ttt)+5}\left(m^{\f{3}{2}}|D^*|^{-\f{1}{2}}|D|^{\f{-4\ttt-\aaa}{2(2\ttt+\aaa+1)}}+m|D^*|^{-\f{1}{2}}|D|^{\f{\aaa-2\ttt}{2(2\ttt+\aaa+1)}}\right)\\
&&\leq2\ww{C}_{15}\f{1}{\delta}\left(\log \f{64}{\delta}\right)^4|D|^{-\f{2\ttt+1}{2(2\ttt+\aaa+1)}}
\ees
where $\ww{C}_{15}=8\ww{C}_{12}\left(2\max\{2\kkk^2\nu,2\kkk\nu\}+1 \right)^3(\f{1}{2\ttt+\aaa+1})^{H(\ttt)+1}$.
We finally conclude that, 
  with probability at least $1-\delta$,
\bes
\begin{aligned}
	\huaE(\overline{\eta_{t,D^*}})-\huaE(\eta^*)=&\left\|T_{C,K}^{1/2}L_K^{-1/2}(\overline{\beta_{t,D^*}}-\beta^*)\right\|_{L^2(\Omega)}^2\\
	\leq&4\ww{C}_{15}^2\left(\log\f{64}{\delta}\right)^8|D|^{-\f{2\ttt+1}{2\ttt+\aaa+1}}(\log|D|)^{2H(\ttt)}+2C_{\theta+\f{1}{2},\gamma}^2\left\|g^*\right\|_{L^2(\Omega)}^2 t^{-(2\theta+1)(1-\mu)}\\
	\leq&C_6^*\left(\log\f{64}{\delta}\right)^8|D|^{-\f{2\ttt+1}{2\ttt+\aaa+1}}.
\end{aligned}
\ees	
The proof of Theorem \ref{semithm1} is complete.
\end{proof}

{\section{Numerical experiments}
{In this section, we conduct some  numerical experiments with simulated data to verify the effectiveness of our proposed algorithms, and compare the results with the previous methodologies for the functional linear model \cite{yc2010,ls2022}.}

{\subsection{A simulation of the DGDFL algorithm}
In this subsection, we conduct a numerical simulation to verify the effectiveness of our proposed algorithms and the corresponding theoretical results, with the assumptions described in the paper being satisfied. We use the similar numerical settings as the previous papers \cite{cy2012},\cite{ggs2022} for the functional linear regression. }

{We consider the domain $\Omega=[0,1]$, the functional predictors $X$ are generated through the process
	\begin{equation}
		X(x)= \sum_{k=1}^N \frac{(-1)^{k+1} Z_k}{k^\alpha} \sqrt{2}\cos(k\pi x), \quad x \in[0,1],
	\end{equation}
	where $N=50, \alpha=0.5$ are utilized in our experiments, and $Z_k \sim U(-\sqrt{3},\sqrt{3})$ are independent uniform random variables. Then, the covariance function is
	\begin{equation}
		C(x,y) = \sum_{k=1}^N \frac{2}{k^{2\alpha}} \cos(k\pi x) \cos(k\pi y) , \quad x,y \in [0,1].
	\end{equation}
	Moreover, we consider the RKHS $\mathcal{H}_K$ induced by the Mercer kernel $K$ as
	\begin{equation}
		\begin{aligned}
			K(x,y) &= \sum_{k=1}^\infty \frac{2}{(k\pi)^4} \cos(k\pi x)   \cos(k\pi y)  \\
			&= -\frac{1}{3} B_4\left( \frac{|x-y|}{2} \right)  -\frac{1}{3} B_4\left( \frac{x+y}{2} \right) , \quad x,y \in [0,1],
		\end{aligned}
	\end{equation}
	where $B_k$ is the $k$-th Bernoulli polynomial, with the fact that
	\begin{equation*}
		B_{2m}(x)= (-1)^{m-1} 2(2m)! \sum_{k\geq 1} \frac{\cos(2\pi kx)}{(2\pi k)^{2m}}, \quad x\in [0,1].
	\end{equation*}
	Furthermore, we set the slope function $\beta^*= L_K^{1/2} T_{C,K}^\theta g^*$ to make the regularity assumption \eqref{regularity} being satisfied, where we choose $\theta=0.1$, and
	\begin{equation*}
		g^*(x)=  4\pi^2 \sum_{k=1}^N \sqrt{2} \cos(k\pi x), \quad x \in[0,1].
	\end{equation*}
	The random noise $\epsilon \sim \mathcal{N}(0,\sigma^2)$ is assumed to be independent of $X$ and follows the normal distribution. This makes the noise assumptions \eqref{2momc} and \eqref{pmomc} being satisfied. Moreover, since $\{Z_k\}_{k=1}^N$ are bounded random variables, the assumption \eqref{lk1/2} is also satisfied with some absolute constant $\kkk$.}

\begin{figure} [t]
	\centering
	\includegraphics[width=1\textwidth]{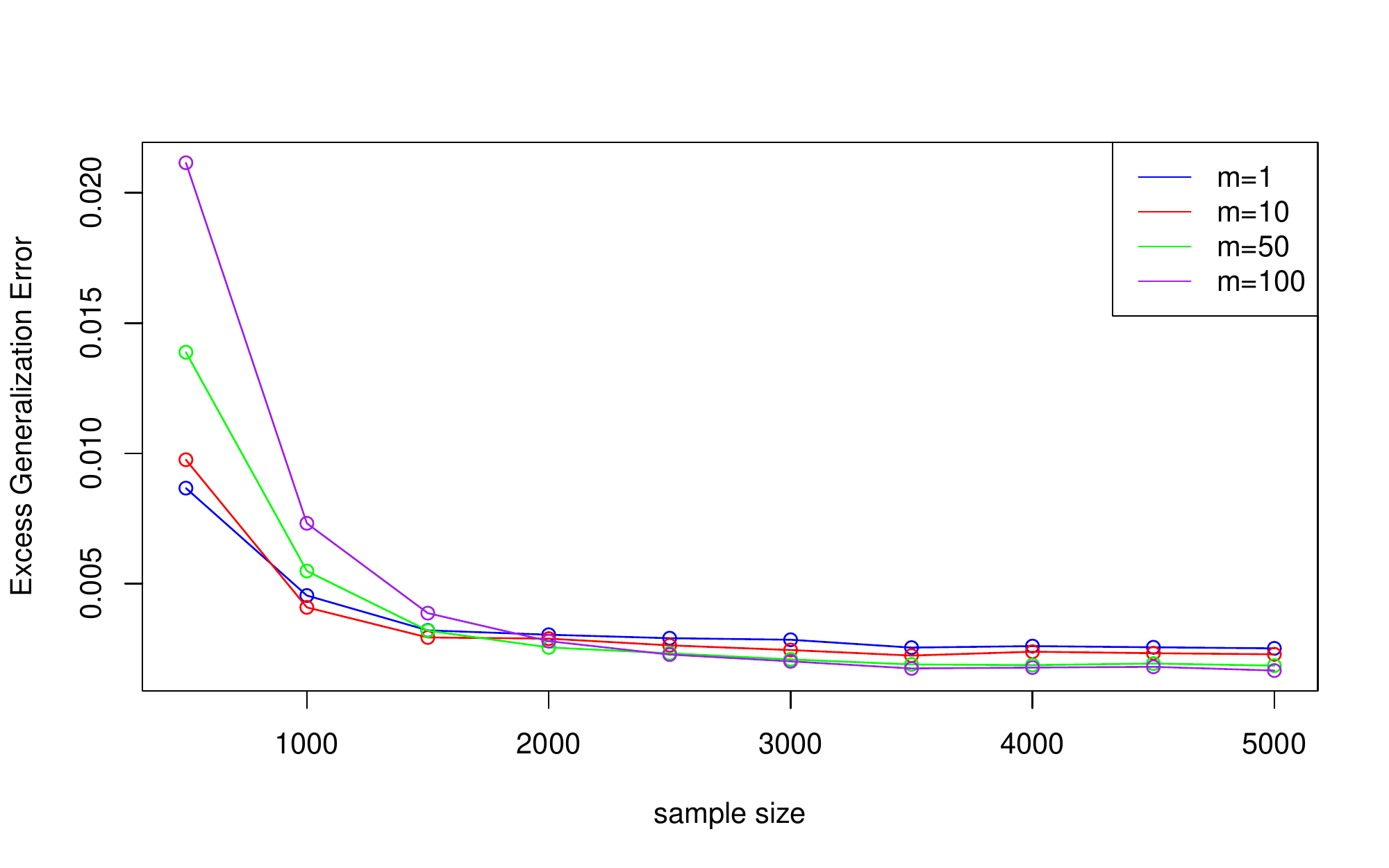}
	\caption{{The excess risk w.r.t. the sample size for the DGDFL algorithm, with the number of local machines being $m=1, 10, 50,100$ respectively, and $\sigma=1$. The experiments are repeated for 20 times.}}
	\label{figure1}
\end{figure}

{We then conduct the numerical experiments to examine the empirical performance of our proposed algorithms. For all the experiments, the stepsizes are selected as $\gamma_k= O(k^{-\mu})$ with  $\mu=0.1$. For each local machine, the iteration stops when $\left\| \beta_{t+1,D_j}- \beta_{t,D_j} \right\|_{L^2(\Omega)} \leq 10^{-4}$. The excess generalization error of the final estimator $\overline{\beta_{t,D}}$ is calculated using a testing sample with size 5000.}

%\begin{figure}
%	\centering
%	\subfigure[{Excess risk w.r.t. the iteration steps for GDFL algorithm}]{
%		\centering\includegraphics[width=0.46\textwidth]{1_iteration_m1.pdf}
%	}
%	\quad
%	\subfigure[{Excess risk w.r.t. the iteration steps for DGDFL algorithm with m=50}]{
%		\includegraphics[width=0.46\textwidth]{1_iteration_m50.pdf}
%	}
%	\caption{{The excess risk w.r.t. the iteration steps for GDFL and DGDFL algorithms.}}
%	\label{figure1}
%\end{figure}

%{Figure \ref{figure1} (a) demonstrates the excess risk w.r.t. the iteration steps for the GDFL algorithm in subsection \ref{section21}, where the sample size is chosen as $|D|= 5000$. It can be observed that the excess risk decreases quickly to the minimal value at some iteration step $T$, and it doesn't become smaller with the further iterations. Figure \ref{figure1} (b) demonstrates the excess risk w.r.t. the iteration steps for the DGDFL algorithm in subsection \ref{section22}, with the sample size $|D|= 5000$, and the number of local machines $m=50$. For the both algorithms, the stepsizes are chosen the same as is described before. It can be observed that the two algorithms can achieve almost the same excess risk, and the number of iterations needed for the convergence of the two algorithms to the minimal excess risk are almost the same. This coincides with our theoretical results in subsections \ref{section21} and \ref{section22} that for the two algorithms, the excess risk bounds are of the same rate of convergence, and the number of iterations to achieve such excess risk bounds are also the same.}

{Figure \ref{figure1} and Figure \ref{figure2} exhibit the excess risk w.r.t. the sample size for our proposed DGDFL algorithm with $\sigma=1$ and $\sigma=1.5$ respectively. We conduct several experiments with the choice of different {numbers} of local machines. When $m=1$, this is in fact the GDFL algorithm. Firstly, we can observe that for both algorithms, the excess risk decreases quite fast with the increase of the sample size. This corresponds to our theoretical results that both algorithms can achieve the almost optimal learning rates $|D|^{-\beta}$ for some $\beta>0$. Secondly, when the sample size is small (e.g., $|D|=500$), the DGDFL algorithm performs worse when the number of local machines increases, this corresponds to our theoretical result that the restriction on the maximal number of local machines is strict when $|D|$ is small.  Finally, when the sample size is large (e.g., $|D|=5000$), the restriction on the maximal number of local machines is lenient, and the performances of the DGDFL algorithm are similar with {the} usage of whatever number of local machines satisfying such restriction. Therefore, for a large sample size, we might use more local machines to achieve unimpaired performance with even less computational cost on each local machine. This embodies the effectiveness of our proposed DGDFL algorithm.  }

\begin{figure} [t]
	\centering
	\includegraphics[width=1\textwidth]{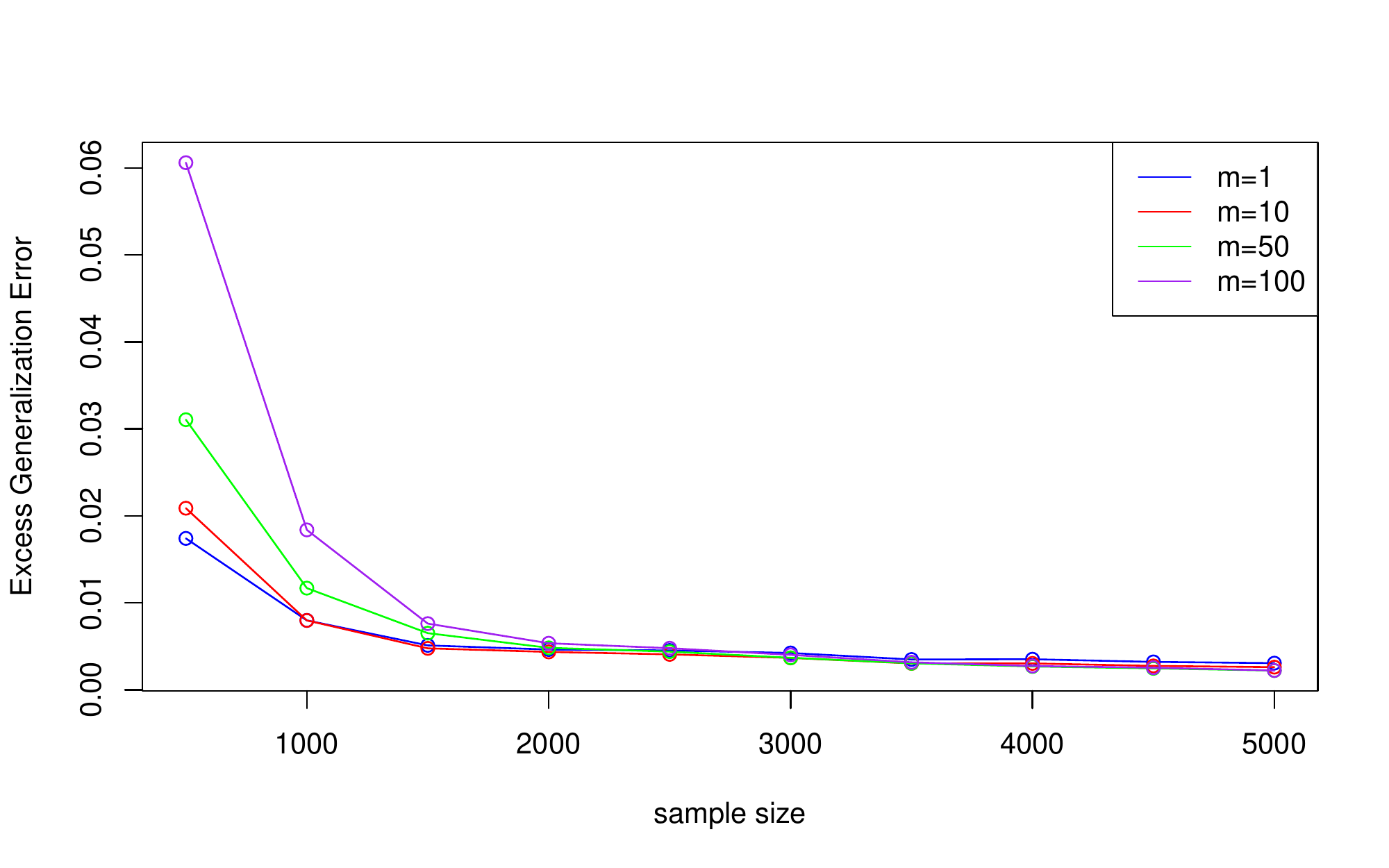}
	\caption{{The excess risk w.r.t. the sample size for the DGDFL algorithm, with the number of local machines being $m=1, 10, 50,100$ respectively, and $\sigma=1.5$. The experiments are repeated for 20 times.}}
	\label{figure2}
\end{figure}

%{In Figure \ref{figure3}, we present the excess risk w.r.t. the number of local machines for the semi-supervised DGDFL algorithm in subsection \ref{semi-supervised_learning_section}, and compare it with the original DGDFL algorithm. We set the slope function $\beta^*= 100 L_K^{1/2} T_{C,K}^\theta g^*$, and choose the number of unlabeled data to be $|\widetilde{D}|= |D|=5000$. The other hyperparameters are chosen the same as before. We can notice that when the number of local machines is small, the semi-supervised DGDFL algorithm performs worse than the DGDFL algorithm. But with the increase of the local machines, the performance of the semi-supervised DGDFL algorithm decreases much slower than the DGDFL algorithm, and it works better than the DGDFL algorithm from $m=60$. This indicates the advantage of the semi-supervised DGDFL algorithm compared with the original DGDFL algorithm when the number of local machines is large, and empirically verifies the effectiveness of the semi-supervised DGDFL algorithm on relaxing the restriction on the maximal number of local machines.} 

%\begin{figure} 
%	\centering
%	\includegraphics[width=0.95\textwidth]{Figure3_final.pdf}
%	\caption{{The excess risk w.r.t. the number of local machines for the DGDFL and semi-supervised DGDFL algorithm. The experiment is repeated for 10 times.}}
%	\label{figure3}
%\end{figure}

{\subsection{Comparison with previous methods}}
{
	In this subsection, we compare our proposed GDFL and DGDFL algorithms with the previous methods for functional linear model, i.e., the classical reproducing kernel (RK) approach \cite{yc2010}, and a subsequently proposed  divide-and-conquer version of it called the divide-and-conquer reproducing kernel (DRK) approach \cite{ls2022}, to further verify the effectiveness of our proposed algorithms. }

{We consider the simulation setting of \cite{hh2007,yc2010} where the domain $\Omega=[0,1]$. The true slope function $\beta^*$ is given by
	\begin{equation}
		\beta^*(x)= \sum_{k=1}^{50} 4(-1)^{k+1}k^{-2} \phi_k(x),
	\end{equation}
	where $\phi_1(x)=1$ and $\phi_{k+1}(x)= \sqrt{2} cos(k\pi x)$ for $k\geq 1$. The functional predictors $X$ are generated through the process
	\begin{equation}
		X(x)= \sum_{k=1}^{50} (-1)^{k+1} k^{-\frac{\nu}{2}} Z_k \phi_k(x), \quad x \in[0,1],
	\end{equation}
	where $\nu=1.1$, and $Z_k \sim U(-\sqrt{3},\sqrt{3})$ are independently sampled from the uniform distribution on $[-\sqrt{3},\sqrt{3}]$.
	The random noise of the functional linear model is $\epsilon \sim \mathcal{N}(0,\sigma^2)$ with $\sigma=1.5$.}

\begin{table}[t] 
	{
		\centering
		\begin{tabular}{|c|c|c|c|c|} 
			\hline
			Data & Method & Estimation Error & Prediction Error & Computation Time [s] \\
			\hline
			\multirow{6}*{$|D|=100$} & GDFL & 0.1936 (0.1628) & 0.0788 (0.0703) & 1.0650 (0.2991) \\
			& DGDFL [2] & 0.1948 (0.1581) & 0.0774 (0.0678) & 0.5016 (0.1531) \\
			& DGDFL [5] & 0.2252 (0.1467) & 0.0888 (0.0698) & 0.1509 (0.0442) \\
			& RK & 0.2179 (0.1656)  & 0.0851 (0.0694) & 3.3786 (0.1084) \\
			& DRK [2] & 0.2086 (0.1587)  & 0.0820 (0.0681) & 0.8256 (0.0288) \\
			& DRK [5] & 0.2205 (0.1761)  & 0.0914 (0.0762) & 0.1996 (0.0087) \\
			\hline
			\multirow{6}*{$|D|=200$} & GDFL  & 0.1140 (0.0630) & 0.0403 (0.0256) & 3.749 (0.4958) \\
			& DGDFL [2] & 0.1180 (0.0681) & 0.0407 (0.0261) & 1.1659 (0.2815) \\
			& DGDFL [5] & 0.1204 (0.0703) & 0.0431 (0.0274) & 0.4562 (0.1049) \\
			& RK & 0.1189 (0.0905)  & 0.0420 (0.0305) & 17.0356 (0.4466) \\
			& DRK [2] & 0.1162 (0.0760)  & 0.0416 (0.0278) & 3.3752 (0.0767) \\
			& DRK [5] & 0.1294 (0.0840)  & 0.0458 (0.0303) & 0.5570 (0.0135) \\
			\hline
			\multirow{6}*{$|D|=500$} & GDFL & 0.0678 (0.0325) & 0.0214 (0.0117) & 16.1347 (1.0137) \\
			& DGDFL [2] & 0.0732 (0.0317) & 0.0220 (0.0113) & 5.1245 (0.5934) \\
			& DGDFL [5] & 0.0792 (0.0342) & 0.0238 (0.0124) & 1.1748 (0.1719) \\
			& RK & 0.0625 (0.0452)  & 0.0206 (0.0150) & 315.816 (11.457) \\
			& DRK [2] & 0.0724 (0.0492)  & 0.0230 (0.0152) & 33.5947 (0.8937) \\
			& DRK [5] & 0.0795 (0.0465)  & 0.0250 (0.0148) & 3.3646 (0.0436) \\
			\hline
		\end{tabular}	
		\caption{{Comparison of the estimation error, prediction error, and computation time of different algorithms for the simulation, the number inside the square brackets indicates the number of local machines. We repeat the experiments for 100 times, with the mean and standard deviation displayed.}} \label{table1}
	}
\end{table}

{We then conduct the numerical experiments to examine the empirical performance of our proposed algorithms and compare with the previous methods. %The values of the functional predictors are observed at evenly spaced points on the domain with size 501.
	For all the experiments, we use the Gaussian kernel with bandwidth $0.33$, and the stepsizes are selected as $\gamma_k= O(k^{-\mu})$ with  $\mu=0$. For each local machine, the iteration stops when $\left\| \beta_{t+1,D_j}- \beta_{t,D_j} \right\|_{L^2(\Omega)} \leq \tau$, with the tolerance $\tau$ chosen based on the training sample size in the local machines. The estimation error of the final estimator $\overline{\beta_{t,D}}$ is calculated based on the true slope function $\beta^*$, and the prediction error (excess risk) is calculated by a testing sample with size 1000. The computation time represents the average running time of the local machines.}

{We present the performance of different algorithms in Table \ref{table1}, with the training sample size $|D|$ chosen as 100, 200, and 500 respectively. The GDFL algorithm and the RK algorithm utilize the whole training sample for the training in one machine, while the DGDFL algorithm and the DRK algorithm are the divide-and-conquer algorithms based on them, and the number inside the square brackets indicates the number of local machines. }

{It can be observed that compared with the classical RK algorithm, which requires $O(|D|^3)$ computational cost due to the computation of the inverse of the kernel matrix, the GDFL algorithm can achieve comparable estimation error and prediction error, with much less computational cost due to the avoidance of the calculation of the inverse matrix, especially when the sample size $|D|$ is quite large. For example, the RK algorithm would be very slow when $|D|=500$, while the GDFL algorithm only needs one fifteenth of the running time of the RK algorithm.} %This is because the GDFL algorithm only requires $O(|D|T)$ computational cost, where $T$ is the number of iterations, which is typically of the order $O(|D|)$, resulting the $O(|D|^2)$ computational cost for the GDFL algorithm.}

{The DGDFL algorithm we proposed is a divide-and-conquer approach of the GDFL algorithm, when $m$ local machines are utilized, each contains $|D|/m$ training samples, thus making the computational cost of the DGDFL algorithm much smaller than that of the original GDFL algorithm due to a smaller training sample size in each local machine. The DRK algorithm is a divide-and-conquer approach of the classical RK approach, and it can approximately diminish the computational cost to $1/m^2$ of the initial requirements \cite{ls2022}. These are also verified by our numerical simulation: while increasing the number of local machines, the estimation error and the prediction error are almost unchanged or only getting slightly worse, but the mean and variance of the computation time are largely improved.}

{Moreover, it can be noticed in Table \ref{table1} that our proposed DGDFL algorithm can achieve similar accuracy as the classical RK approach, while largely reducing the computational cost, especially when the sample size $|D|$ is quite large or more local machines are utilized. Furthermore, compared with the DRK algorithm with the same local machines, the DGDFL algorithm can also achieve comparable accuracy with a smaller computational cost, especially when the sample size is larger.}

\iffalse
\begin{table}[t] 
{
	\centering
	\begin{tabular}{|c|c|c|c|} 
		\hline
		Data & Method & Test-set Mean Squared Error  & Computation Time [s] \\
		\hline
		\multirow{4}*{Tecator Fat} & RK & 5.1802 (1.7534)  & 4.1799 (0.1560) \\
		& GDFL & 4.9616 (1.2195) & 2.1312 (0.0948)  \\
		& DGDFL [2] & 5.4078 (1.4210) &  1.3757 (0.0461)  \\
		& DGDFL [4] & 5.6851 (1.5072) & 0.8501 (0.0258)  \\
		\hline
		\multirow{6}*{$|D|=200$} & GDFL  & 0.1140 (0.0630)  &  \\
		& DGDFL [2] & 0.1180 (0.0681) & 0.0407 (0.0261)  \\
		& DGDFL [5] & 0.1204 (0.0703) & 0.0431 (0.0274)  \\
		& RK & 0.1189 (0.0905)  & 0.0420 (0.0305) \\
		& DRK [2] & 0.1162 (0.0760)  & 0.0416 (0.0278) \\
		& DRK [5] & 0.1294 (0.0840)  & 0.0458 (0.0303)  \\
		\hline
		\multirow{6}*{$|D|=500$} & GDFL & 0.0678 (0.0325) & 0.0214 (0.0117)  \\
		& DGDFL [2] & 0.0732 (0.0317) & 0.0220 (0.0113)  \\
		& DGDFL [5] & 0.0792 (0.0342) & 0.0238 (0.0124)  \\
		& RK & 0.0625 (0.0452)  & 0.0206 (0.0150)  \\
		& DRK [2] & 0.0724 (0.0492)  & 0.0230 (0.0152)  \\
		& DRK [5] & 0.0795 (0.0465)  & 0.0250 (0.0148)  \\
		\hline
	\end{tabular}	
	\caption{{Comparison of estimation error, prediction error, and computation time of different algorithms for functional linear model, the number inside the square brackets after the algorithm's name means the number of local machines. We repeat the experiments for 100 times, with the mean and standard deviation displayed.}} \label{table2}
}
\end{table}
\fi

\begin{figure} [t]
\centering
\includegraphics[width=1.02\textwidth]{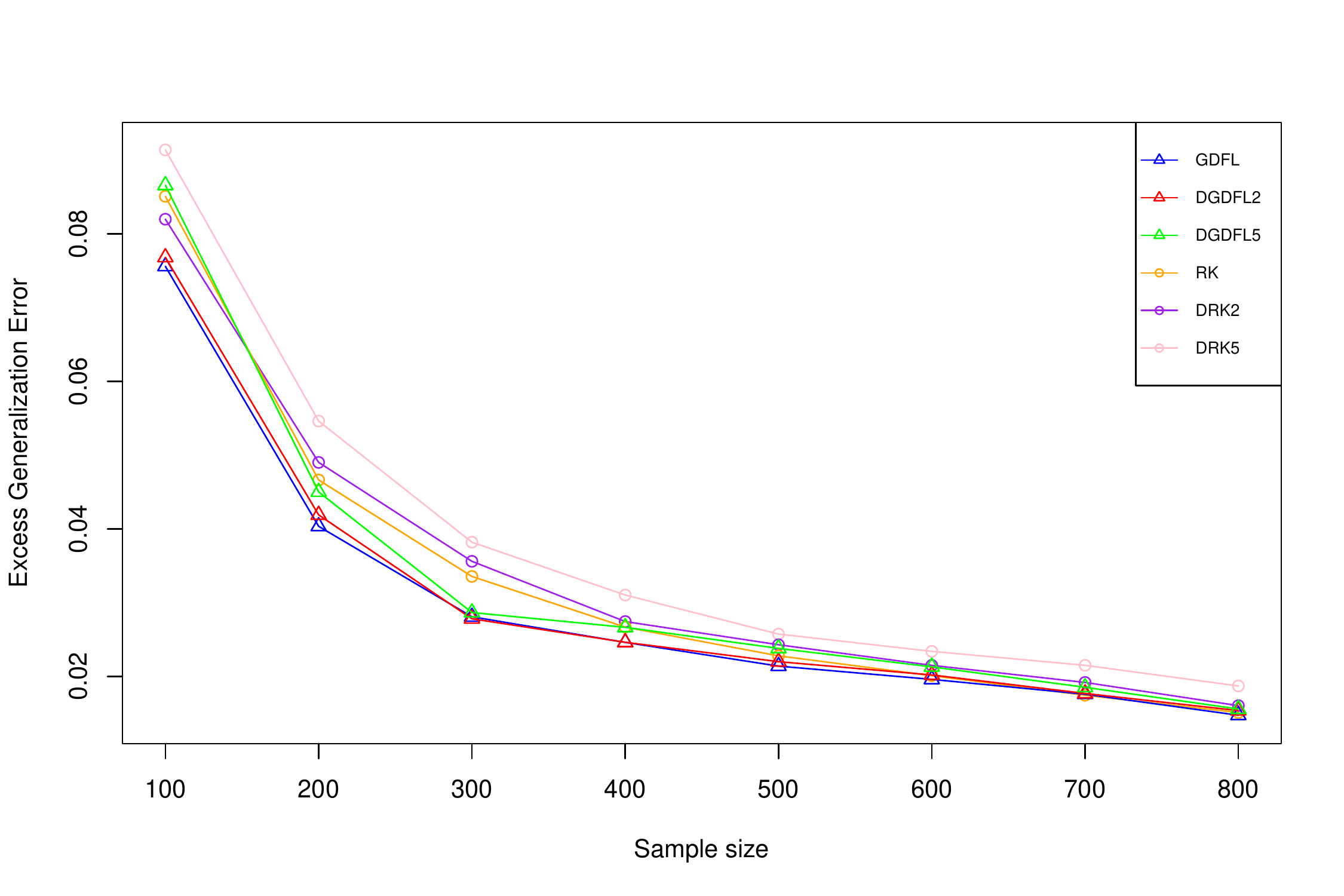}
\caption{{The excess risk w.r.t. the sample size for the GDFL, DGDFL, RK, and DRK algorithms, with the the number in the label indicating the number of local machines. The experiments are repeated for 50 times.}}
\label{figure3}
\end{figure}

\begin{figure} [t]
\centering
\includegraphics[width=1.02\textwidth]{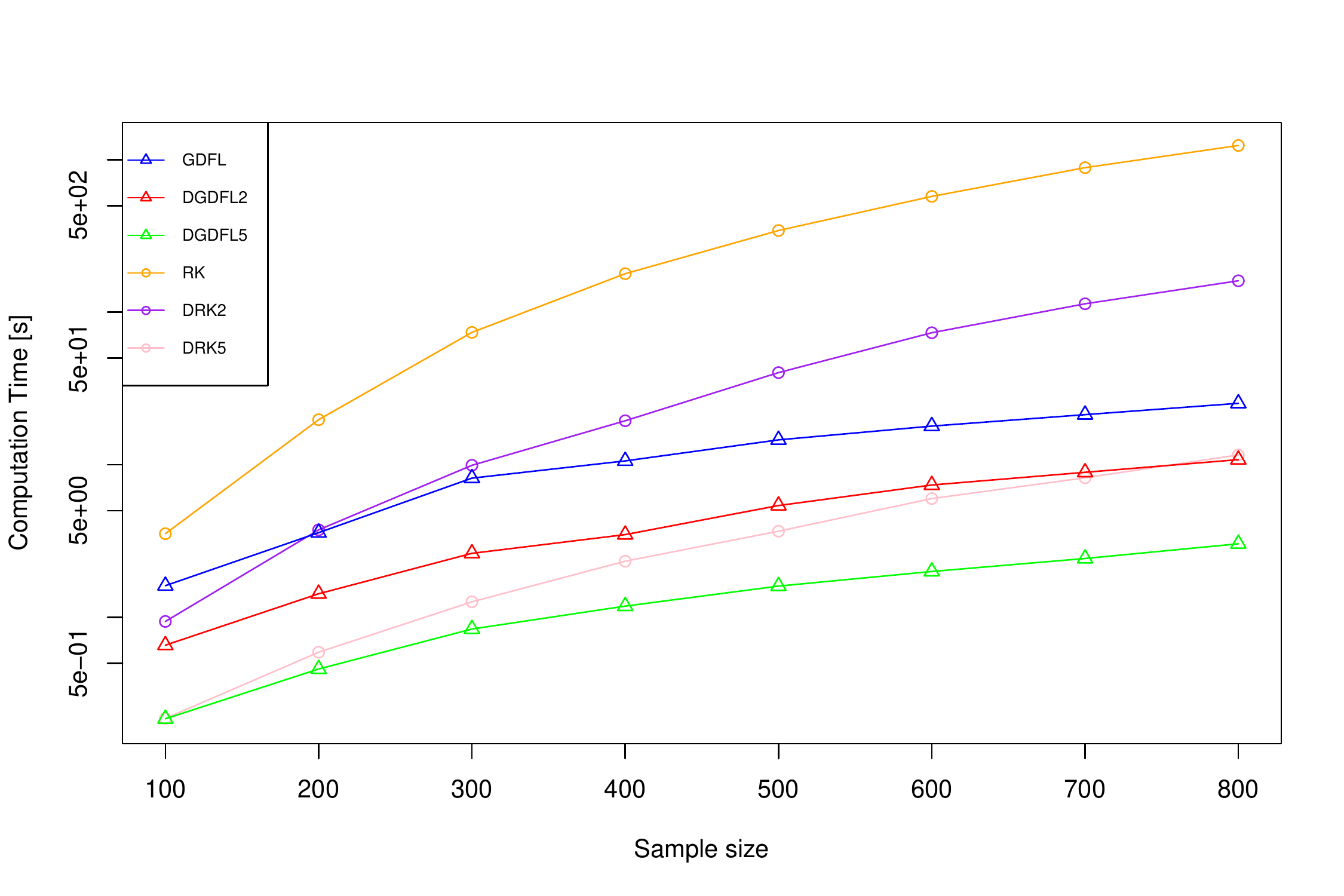}
\caption{{The computation time w.r.t. the sample size for the GDFL, DGDFL, RK, and DRK algorithms, with the the number in the label indicating the number of local machines. The y-axis is in the log scale. The experiments are repeated for 50 times.}}
\label{figure4}
\end{figure}

{We further plot the excess risk and the computation time w.r.t. the sample size for our proposed GDFL and DGDFL algorithms and the previously proposed RK and DRK algorithms in Figure \ref{figure3} and Figure \ref{figure4} respectively. It can be observed that when the sample size becomes larger, the DGDFL algorithm can achieve almost the same excess risk as the GDFL algorithm, while the computation time can be  largely improved. However, for the DRK algorithm, even though it can also largely improve the computation time compared with the RK algorithm, with the number of local machines increasing such as $m=5$, the excess risk of the DRK algorithm would become slightly worse than that of the RK algorithm.}

{As for the computation time of different algorithms shown in Figure \ref{figure4}, comparing the GDFL algorithm with the RK algorithm, or comparing the DGDFL algorithm with the DRK algorithm that utilize the same number of local machines, the running time of our proposed GDFL and DGDFL algorithms is always better. This advantage of the computation time is more remarkable when the sample size becomes larger, since our proposed algorithms have a lower order of the computational cost. Moreover, we can also notice that, when the sample size becomes quite large such as $|D|=800$, the DGDFL algorithm with only two local machines can even be slightly faster than the DRK algorithm with five local machines, and obtains a slightly better excess risk in the meantime. These numerical simulations further demonstrate the effectiveness and advantage of our proposed GDFL and DGDFL algorithms.}

\section*{Acknowledgements}
{The authors would like to thank AE, three anonymous referees for their nice and constructive comments that help improve the quality of the paper}. The work described in this paper was partially supported by InnoHK initiative, The Government of the HKSAR, Laboratory for AI-Powered Financial Technologies, the Research Grants Council
of Hong Kong [Project No. \# CityU 11308121 and No. N\_CityU102/20], and National Natural Science Foundation of China under [Project No. 11461161006]. The work by Jun Fan is partially supported by the Research Grants Council of Hong Kong [Project No. HKBU 12302819] and [Project No. HKBU 12301619].

%\section*{Appendix}

%\subsection*{xx}

\bibliography{ref}

\begin{thebibliography}{999}
\bibitem{bach2023}
{Francis Bach. Information theory with kernel methods. IEEE Transactions on Information Theory, 69(2): 752-775, 2023.	
}
	\bibitem{bpr2007}
	 Frank Bauer, Sergei Pereverzev, and Lorenzo Rosasco. On regularization algorithms in learning theory, Journal of complexity 23.1: 52-72, 2007.
	
	\bibitem{ch2006}
	  T. Tony Cai,  Peter Hall. Prediction in functional linear regression, Annals of Statistics 34(5): 2159-2179, 2006.	
	 
	 \bibitem{cy2012}
	  T. Tony Cai,  Ming Yuan. Minimax and adaptive prediction for functional linear regression, Journal of the American Statistical Association, 107(499): 1201-1216, 2012.
	 
	 \bibitem{clz2017}
	 Xiangyu Chang, Shao-Bo Lin, Ding-Xuan Zhou. Distributed semi-supervised learning with kernel ridge regression. The Journal of Machine Learning Research, 18(1): 1493-1514, 2017.
	 
	 \bibitem{ctfg2022}
	 Xiamin Chen, Bohao Tang, Jun Fan, Xin Guo (2022). Online gradient descent algorithms for functional data learning, Journal of Complexity, 70: 101635, 2022.
	 {
	 	\bibitem{db2016}
	 	 Aymeric Dieuleveut, Francis Bach. Nonparametric stochastic approximation with large step-sizes, Annals of Statistics, 1363-1399, 2016.
	}
	 \bibitem{Dudley2018}
	  Richard M. Dudley. Real Analysis and Probability. CRC Press, 2018.
	
	 
	 \bibitem{fls2019}
	 Jun Fan, Fusheng Lv, Lei Shi. An RKHS approach to estimate individualized treatment rules based on functional predictors, Mathematical Foundations of Computing, 2(2): 169-181, 2019.
	
	 \bibitem{ggs2022}
	Xin Guo, Zheng-Chu Guo, Lei Shi. Capacity dependent analysis for functional online learning algorithms, Applied and Computational Harmonic Analysis, 67: 101567, 2023.
	 
	 \bibitem{gls2019}
	 Zheng-Chu Guo, Shao-Bo Lin, and Lei Shi. Distributed learning with multi-penalty regularization, Applied and Computational Harmonic Analysis, 46(3): 478-499, 2019.
	 
	 \bibitem{gs2013}
	 Zheng-Chu Guo,  Lei Shi. Learning with coefficient-based regularization and $l1$-penalty. Advances in Computational Mathematics 39(3): 493-510, 2013.
	 {
	 \bibitem{gs2019}
	 Zheng-Chu Guo,  Lei Shi. Fast and strong convergence of online learning algorithms. Advances in Computational Mathematics, 45, 2745-2770, 2019.
}
	\bibitem{hh2007}
	 Peter Hall, Joel L. Horowitz. Methodology and convergence rates for functional linear regression, Annals of Statistics 35(1): 70-91, 2007.
	
	\bibitem{hwz2020}
     Ting Hu, Qiang Wu, and Ding-Xuan Zhou. Distributed kernel gradient descent algorithm for minimum error entropy principle, Applied and Computational Harmonic Analysis, 49(1):  229-256, 2020.
     {
     	\bibitem{jssh2023}
     	Yao Ji, Gesualdo Scutari, Ying Sun, Harsha Honnappa. Distributed (ATC) gradient descent for high dimension sparse regression. IEEE Transactions on Information Theory, 2023.
     }	
     {
     	\bibitem{jymlz2013}
     Rong Jin, Tianbao Yang, Mehrdad Mahdavi, Yu-Feng Li, Zhi-Hua Zhou. Improved bounds for the Nystr\"om method with application to kernel classification,  IEEE Transactions on Information Theory, 59(10): 6939-6949, 2013.
    }
	{
	\bibitem{kprr2018}
	Alec Koppel, Santiago Paternain, Cedric Richard, Alejandro Ribeiro. Decentralized online learning with kernels, IEEE Transactions on Signal Processing, 66(12): 3240-3255, 2018.
}
{
\bibitem{kwsr2019}
Alec Koppel, Garrett Warnell, Ethan Stump, Alejandro Ribeiro, Parsimonious online learning with kernels via sparse projections in function space. The Journal of Machine Learning Research, 20(1): 83-126, 2019.
}
	%{
	%\bibitem{ll2023}
	%Heng Lian, Jiamin Liu. Decentralized Learning over a Network with Nyström Approximation Using SGD. Applied and Computational Harmonic Analysis. 2023.
%}
	
	
	\bibitem{lgz2017}
	Shao-Bo Lin, Xin Guo, and Ding-Xuan Zhou. Distributed learning with regularized least squares. The Journal of Machine Learning Research, 18.1: 3202-3232, 2017.
	
	\bibitem{lz2018}
	 Shao-Bo Lin, and Ding-Xuan Zhou. Distributed kernel-based gradient descent algorithms. Constructive Approximation, 47(2):  249-276, 2018.
	
	%\bibitem{lzc2017}
	%Shao-Bo Lin, Jinshan Zeng, Xiangyu Chang. Learning rates for classification with Gaussian kernels. Neural computation 29.12 (2017): 3353-3380.	
	
	\bibitem{lwz2021}
	 Shao-Bo Lin, Yu Guang Wang, and Ding-Xuan Zhou. Distributed filtered hyperinterpolation for noisy data on the sphere, SIAM Journal on Numerical Analysis, 59(2): 634-659, 2021.
	{
	\bibitem{lc2018}
	Junhong Lin, Volkan Cevher. Optimal distributed learning with multi-pass stochastic gradient methods. In International Conference on Machine Learning, 3092-3101, PMLR, 2018.
}
	\bibitem{ls2022}
	 Jiading Liu, Lei Shi. Statistical optimality of divide and conquer kernel-based functional linear regression, arXiv preprint arXiv:2211.10968, 2022.
	{
		\bibitem{msz2023}
		Tong Mao, Zhongjie Shi, and Ding-Xuan Zhou, Approximating functions with multi-features by deep convolutional neural networks, Analysis and Applications, 21: 93--125, 2023.
	}
	\bibitem{mg2022}
	Yuan Mao, Zheng-Chu Guo. Online regularized learning algorithm for functional data, arXiv:2211.13549, 2022.
	{\bibitem{mm2017}
		Brendan McMahan, Eider Moore, Daniel Ramage, Seth Hampson, and Blaise Aguera y Arcas, Communication-efficient learning of deep networks from decentralized data, Artificial Intelligence and Statistics. PMLR, 1273–1282, 2017.
		}
	{
	\bibitem{no2009}
	Angelia Nedic, Asuman Ozdaglar. Distributed subgradient methods for multi-agent optimization, IEEE Transactions on Automatic Control, 54(1): 48-61, 2009.
}
	\bibitem{pinelis1994}
	 Iosif Pinelis. Optimum bounds for the distributions of martingales in Banach spaces, Annals of Probability, 1679–1706, 1994.
	{
	\bibitem{rr2007}
	 Ali Rahimi, Benjamin Recht. Random features for large-scale kernel machines, Advances in Neural Information Processing Systems, 20, 2007.
}	
	\bibitem{ramsay1982}
	 James O. Ramsay. When the data are functions. Psychometrika 47(4): 379-396, 1982.
	
	
	\bibitem{rd1991}
	 James O. Ramsay, and C.  J. Dalzell. Some tools for functional data analysis, Journal of the Royal Statistical Society: Series B (Methodological), 53(3), 539-561, 1991.
	
	\bibitem{rs2005}
	 James O. Ramsay, Bernard W. Silverman. Fitting Differential Equations to Functional Data: Principal Differential Analysis. Springer New York, 2005.
	 {
	 \bibitem{rrr2020}
	 Dominic Richards, Patrick Rebeschini, Lorenzo Rosasco.  Decentralised learning with distributed gradient descent and random features, Proceedings of 37th International Conference on Machine Learning, PMLR, 119, 2020.
	}
	 {
	 	\bibitem{rr2017}
	 	Alessandro Rudi, Lorenzo  Rosasco. Generalization properties of learning with random features. Advances in Neural Information Processing Systems, 30, 2017.
	 }
	{
		\bibitem{sz2007}
	Steve Smale, Ding-Xuan Zhou. Learning theory estimates via integral of operators and their approximations, Constructive Approximation 26(2): 153-172, 2007.
}
	\bibitem{sw2021}
	 Hongwei Sun,  Qiang Wu. Optimal rates of distributed regression with imperfect kernels, Journal of Machine Learning Research. 22: 7732-7765, 2021.
	 
	 \bibitem{sl2022}
	 Zirui Sun,  Shao-Bo Lin. Distributed learning with dependent samples, IEEE Transactions on Information Theory, 68(9): 6003-6020, 2022.
	 
	 \bibitem{sspg2016}
	 Z. Szabó, B. K. Sriperumbudur, B. Póczos, A. Gretton. Learning theory for distribution regression. The Journal of Machine Learning Research, 17(1), 5272-5311, 2016.
	 
	
	\bibitem{tn2018}
	Hongzhi Tong, Michael Ng. Analysis of regularized least squares for functional linear regression model, Journal of Complexity, 49: 85-94, 2018.
	
	\bibitem{tong2021}
	Hongzhi Tong.  Distributed least squares prediction for functional linear regression. Inverse Problems, 38(2): 025002, 2021.
	
	\bibitem{tong2023}
	Hongzhi Tong. Functional linear regression with Huber loss. Journal of Complexity, 74: 101696, 2023.
	
	\bibitem{wh2019a}
	Cheng Wang,  Ting Hu. Online minimum error entropy algorithm with unbounded sampling, Analysis and Applications, 17(2): 293-322, 2019.
	

	
	
	\bibitem{wcm2016}
Jane-Ling Wang, Jeng-Min Chiou,  Hans-Georg M\"uller. Functional data analysis,
	Annual Review of Statistics and Its Application, 3:257-295, 2016.
	{
	\bibitem{xwct2021}
	Ping Xu, Yue Wang, Xiang Chen, Zhi  Tian.  COKE: Communication-censored decentralized kernel learning, Journal of Machine Learning Research, 22(1): 8813-8847, 2021.
}
	{
	\bibitem{wyz2023}
Le-Yin Wei, Zhan Yu, and Ding-Xuan Zhou. Federated learning for minimizing nonsmooth convex loss functions, Mathematical Foundations of Computing, 6(4): 753-770, 2023.	
	}
	\bibitem{Wellner2013}
	Jon Wellner. Weak Convergence and Empirical Processes: with Applications to Statistics. Springer Science \& Business Media, 2013.
	
	\bibitem{yra2007}
	Yuan Yao, Lorenzo Rosasco, and Andrea Caponnetto. On early stopping in gradient descent learning. Constructive Approximation, 26(2): 289-315, 2007.
	{
	\bibitem{yz2006}
Yiming Ying, Ding-Xuan Zhou. Online regularized classification algorithms. IEEE Transactions on Information Theory, 52(11): 4775-4788, 2006.	
}
{
	\bibitem{yp2008}
	 Yiming Ying, Massimiliano Pontil. Online gradient descent learning algorithms. Foundations of Computational Mathematics, 8: 561-596, 2008.
	}
	\bibitem{yhsz2021}
Zhan Yu, Daniel Ho, Zhongjie Shi, Ding-Xuan Zhou. Robust kernel-based distribution regression, Inverse Problems, 37(10): 105014, 2021.
{
	\bibitem{yhy2022}
	Zhan Yu, Daniel  Ho, Deming Yuan. Distributed randomized gradient-free mirror descent algorithm for constrained optimization. IEEE Transactions on Automatic Control, 67(2): 957-964, 2022.
}
{
	\bibitem{ypg2020}
	Deming Yuan, Alexandre Proutiere, Guodong Shi. Distributed online linear regressions. IEEE Transactions on Information Theory, 67(1): 616-639, 2020.
}
	
	
	\bibitem{yc2010}
	 Ming Yuan, T. Tony Cai. A reproducing kernel Hilbert space approach to functional linear regression, Annals of Statistics 38.6: 3412-3444, 2010.
	
	\bibitem{zdw2015}
	Yuchen Zhang, John Duchi,  Martin Wainwright. Divide and conquer kernel ridge regres-
	sion: A distributed algorithm with minimax optimal rates,  Journal of Machine Learning
	Research, 16(1): 3299–3340, 2015.
	
	\bibitem{zhou2003}
	Ding-Xuan Zhou. Capacity of reproducing kernel spaces in learning theory, IEEE Transactions on Information Theory, 49(7): 1743-1752, 2003.
	
	\bibitem{zhou2018}
 Ding-Xuan Zhou. Deep distributed convolutional neural networks: Universality, Analysis and Applications, 16: 895-919, 2018.
	
	
	
	
	
	
	
	
\end{thebibliography}

\end{document}